\documentclass[acmtog]{acmart}

\usepackage{algorithm}
\usepackage{algpseudocode}
\usepackage{float}
\usepackage{amsmath}   
\usepackage{amsthm}    
\usepackage{mathtools}
\usepackage{subcaption} 
\usepackage{tikz}       
\usetikzlibrary{arrows.meta} 
\usepackage[export]{adjustbox} 
\usepackage[dvipsnames, table]{xcolor}
\usepackage{pifont}
\usepackage{stfloats}
\usepackage{caption}

\AtBeginDocument{%
  }

\setcopyright{acmlicensed}
\copyrightyear{2026}
\acmYear{2026}
\setcopyright{cc}
\setcctype{by}
\acmConference[SIGGRAPH Conference Papers '26]{Special Interest Group on Computer Graphics and Interactive Techniques Conference Conference Papers}{July 19--23, 2026}{Los Angeles, CA, USA}
\acmBooktitle{Special Interest Group on Computer Graphics and Interactive Techniques Conference Conference Papers (SIGGRAPH Conference Papers '26), July 19--23, 2026, Los Angeles, CA, USA}
\acmDOI{10.1145/3799902.3811080}
\acmISBN{979-8-4007-2554-8/2026/07}

\acmSubmissionID{papers\_416s1}

\begin{document}

\definecolor{ProjectSiteColor}{HTML}{2563EB}
\title{InfiniteDiffusion: Bridging Learned Fidelity and Procedural Utility for Open-World Terrain Generation - \href{https://xandergos.github.io/terrain-diffusion/}{\textcolor{ProjectSiteColor}{\underline{\textbf{Project Website}}}}}


\author{Alexander Goslin}
\email{alexander.goslin@gmail.com}
\affiliation{
    \institution{Independent Researcher}
    \city{Sunnyvale}
    \country{USA
}}

\renewcommand{\shortauthors}{Goslin}

\begin{abstract}
 For decades, procedural worlds have been built on procedural noise functions such as Perlin noise, which are fast and infinite, yet fundamentally limited in realism and large-scale coherence. Conversely, diffusion models offer unprecedented fidelity but remain generally confined to bounded canvases. We introduce InfiniteDiffusion, a training-free algorithm that reformulates diffusion sampling for lazy and unbounded generation, bridging the fidelity of diffusion models with the properties that made procedural noise indispensable: seamless infinite extent, seed-consistency, and constant-time random access. To demonstrate the utility of this approach, we present Terrain Diffusion, a framework for learned procedural terrain generation with a procedural noise-like interface. Our framework outpaces orbital velocity by 9 times on a consumer GPU, enabling realistic terrain generation at interactive rates. We integrate a hierarchical stack of diffusion models to couple planetary context with local detail, a compact Laplacian encoding to stabilize outputs across Earth-scale dynamic ranges, and an open-source infinite-tensor framework for constant-memory manipulation of unbounded tensors. Together, these components position diffusion models as a practical foundation for the next generation of infinite virtual worlds.
\end{abstract}

\begin{teaserfigure}
  \Description{A four panel composite shows progressively closer views of synthetic terrain generated by Terrain Diffusion. The leftmost panel displays a broad archipelago of small continents with a red square marking a region of interest. The second panel zooms into that region and reveals many mountain ranges and coastal plains, again marked by a red square highlighting a smaller area. The third panel zooms further into the selected mountain region, showing an intricate mountain range with various valleys and complex topography. The rightmost panel presents the closest view, focusing on fine scale topographic structure with sharply defined ridgelines, valleys, and erosion patterns.}
  \includegraphics[width=\textwidth]{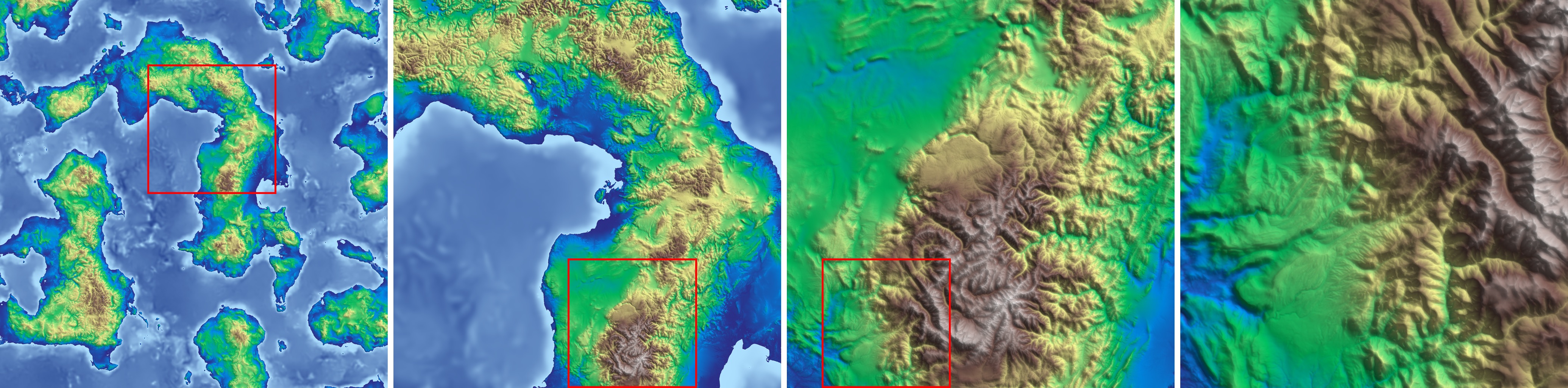}
  \caption{A region of a world generated with Terrain Diffusion. The leftmost panel spans roughly five million square kilometers, with about 2.2 million square kilometers of land area, comparable to the size of the Congo. Red boxes denote the region shown in the next panel, illustrating coherent terrain generation across four orders of magnitude in scale. Zoom for details.}
  \label{fig:teaser}
\end{teaserfigure}

\settopmatter{printacmref=false}
\setcopyright{none}
\acmDOI{}
\acmISBN{}
\renewcommand\footnotetextcopyrightpermission[1]{} 

\maketitle

\section{Introduction}

Procedural terrain generation underpins the creation of virtual worlds, from open-world games to planetary simulations. For nearly four decades, procedural noise functions such as Perlin noise \cite{perlin_image_1985, perlin_improving_2002} have defined this field. They offer three properties that make them indispensable for procedural worlds: seamless infinite extensibility, seed-consistency, and constant-time random access. A single random seed can deterministically produce a boundless landscape without storing vast datasets, providing an elegant foundation for procedural worlds.

Yet these procedural methods are inherently limited. Their patterns are smooth and lack the hierarchical organization of real geography. Continents, mountain ranges, and river basins emerge in nature from structured, multi-scale processes that simple noise cannot capture. As a result, worlds built from procedural noise often appear plausible but not real.

Recent advances in generative modeling, particularly diffusion models \cite{pmlr-v37-sohl-dickstein15, ho_denoising_2020}, have transformed image synthesis by learning to reproduce natural structure with remarkable realism and control, but these methods are typically confined to smaller bounded domains. Recent work has explored infinite or large-scale generation capabilities, but these approaches generally lose one or more of the core properties that make procedural noise valuable in interactive applications.

We address these limitations through InfiniteDiffusion, a generalization of MultiDiffusion \cite{bar-tal_multidiffusion_2023} for infinite inference. Our framework retains the functional utility of noise: seamless infinite extensibility, seed-consistency, and constant-time random access, while leveraging diffusion models for realism far beyond the reach of procedural noise. We first validate InfiniteDiffusion on standard text-to-image models before applying it as the foundation for Terrain Diffusion. We further provide an open-source infinite-tensor framework for constant-memory streaming and composable manipulation of infinite tensors with arbitrary functions.

Finally, we introduce Terrain Diffusion, the first practical learned procedural terrain generator, capable of streaming an entire planet in real time on consumer GPUs. A hierarchical diffusion stack unifies global and local structure through a coarse planetary model that establishes continental structure, refined by higher-resolution models that introduce mountain ranges, valleys, and local relief. A novel elevation encoding further stabilizes training and inference across the full dynamic range of Earth's terrain, and few-step consistency distillation \cite{song_consistency_2023, lu_simplifying_2025} enables rapid inference.

Together, these components demonstrate that diffusion models can serve as a practical foundation for robust infinite worlds that can be explored interactively and without restrictions.

\section{Related Works}

\noindent\textbf{Procedural noise.}
Procedural terrain traditionally relies on procedural noise such as Perlin or Simplex noise \cite{perlin_image_1985, perlin_improving_2002}, often combined with fractal Brownian motion (fBm) \cite{fournier_computer_1982}. These methods remain popular for their controllability, speed, seed-consistency, and infinite extent, but they lack the large-scale structure of real landscapes. Procedural noise with fBm can produce fractal-like structures that mimic general mountainous terrain, but requires heavy post-processing to approach realism. Even then, complex features like erosion, rivers, canyons, and volcanoes remain notoriously difficult to generate convincingly, particularly while retaining procedural utility. One Noise to Rule Them All \cite{maesumi2024noise} expands on the capabilities of procedural noise by learning to unify existing algorithms, but remains bounded and generally restricted to the distribution of predefined procedural noises. In contrast, InfiniteDiffusion supports arbitrary pixel or voxel-based diffusion models, enabling arbitrary distributions of tensors to be generated and queried like procedural noise.

\medskip
\noindent\textbf{Diffusion and consistency models.}
Denoising diffusion models \cite{ho_denoising_2020, pmlr-v37-sohl-dickstein15} generate samples by iteratively refining noise and are widely used for high-fidelity synthesis. Consistency models \cite{song_consistency_2023} approximate denoised diffusion outputs in one or a few steps, and continuous variants \cite{lu_simplifying_2025} achieve throughput competitive with GANs while retaining most of the quality of full diffusion sampling, enabling interactive use cases.

\medskip
\noindent\textbf{Learned terrain generation.}
GAN-based terrain models \cite{10.1145/3422622, voulgaris_procedural_2021, spick_realistic_2019, beckham_step_2017, https://doi.org/10.1111/cgf.13345} can generate convincing local relief, but they operate on fixed crops and do not tile, limiting them to bounded worlds. Diffusion-based synthesis \cite{10.1609/aaai.v38i11.29150, guerin_interactive_2017, bornepons2025mesatextdriventerraingeneration} further improves fidelity and control, but also assumes finite canvases and requires relatively significant compute. Jain et al. \cite{10.1145/3571600.3571657} offers infinite terrain generation by sampling diffusion-based tiles and blending them with a Perlin-based kernel. Tiles are generated independently and the kernel has no awareness of broader context, so structure remains tied to perlin noise rather than the learned model. In contrast, Terrain Diffusion couples all tiles through a shared global context and fuses tiles through a fully learned, context-aware mechanism. Procedural noise is used only for defining continental layouts, where data is sparse and simple enough that more complex alternatives would provide little benefit while reducing user control.

\medskip
\noindent\textbf{MultiDiffusion and unbounded generation.}
Several works extend generative models beyond image bounds seen in training. InfinityGAN \cite{lin_infinitygan_2022} and TileGAN \cite{Fruehstueck2019TileGAN} produce infinite images with GANs but do not extend to diffusion models, limiting scalability. MultiDiffusion \cite{bar-tal_multidiffusion_2023} and Mixture of Diffusers \cite{jimenez_mixture_2023} generate images larger than the training canvas but still assume a bounded final extent. Auto-regressive methods such as BlockFusion \cite{wu_blockfusion_2024} and WorldGrow \cite{li_worldgrow_2025} generate worlds by conditioning each tile on its neighbors, producing continuous worlds but without seed consistency, since outputs depend on sampling order. Additionally, tiles must be produced sequentially, preventing efficient random access. In contrast, we define a seed-consistent InfiniteDiffusion algorithm whose outputs are order invariant and allow constant-time random-access generation over an infinite domain.

\section{InfiniteDiffusion: Unbounded Generation Across Planetary Scales}

MultiDiffusion \cite{bar-tal_multidiffusion_2023} provides a simple and effective way to extend diffusion sampling beyond a model’s native resolution by averaging predictions from overlapping windows. This enables synthesis across larger images, as local predictions fuse into a seamless image. However, in its standard form, MultiDiffusion remains confined to bounded domains: all windows must lie within a fixed finite canvas, limiting its applicability to unbounded worlds or continuously streamed environments.

We introduce InfiniteDiffusion, an extension of MultiDiffusion that lifts this constraint. By reformulating the sampling process to operate over an infinite domain, InfiniteDiffusion supports seamless, consistent generation at scale. The remainder of this section reviews the principles of MultiDiffusion, and formalizes its extension to unbounded domains. We present the definitions in $\mathbb{Z}^2$ for clarity, but all results extend to $\mathbb{Z}^d$ with minimal modification.

\subsection{A Review of MultiDiffusion}

MultiDiffusion extends standard diffusion sampling by averaging overlapping windows, producing consistent outputs from local predictions. Each denoising step aggregates the predictions of overlapping patches, enforcing continuity across window boundaries and allowing generation of regions much larger than a model’s input size. The process begins with a pretrained diffusion model $\Phi$, operating on images in $\mathcal{I} = \mathbb{R}^{H \times W \times C}$. The diffusion process generates a sequence of images
\[
I_T, I_{T-1}, \dots, I_{0} \quad \text{s.t.} \quad I_{t-1} = \Phi(I_t \mid y)
\]
that refines the original noisy image $I_T$ into a fully denoised version $I_0$, under conditioning vector $y$. MultiDiffusion defines a new model $\Psi$ that generates in a different image space $\mathcal{J} = \mathbb{R}^{H' \times W' \times C}$, producing a new sequence of images
\[
J_T, J_{T-1}, \dots, J_{0} \quad \text{s.t.} \quad J_{t-1} = \Psi(J_t \mid z).
\]

To accomplish this, MultiDiffusion defines $n$ windows indexed by $i \in [n]$. In the finite setting, a region $R$ is any rectangular subset of the $H' \times W'$ coordinate grid, while a window region has a fixed size $H \times W$. In MultiDiffusion, each window $i$ is assigned a window region $R_i$. For an image $J \in \mathcal{J}$, we write $J[R]$ for the values of $J$ on the coordinates in $R$.

Each window also has a weight matrix $W_i \in \mathbb{R}^{H \times W}$ that specifies the relative contribution of each pixel in the pretrained diffusion model's output. Let $U_i(x)$ denote the $H' \times W'$ image that places an $H \times W$ tensor $x$ in the region $R_i$ with zeros elsewhere. With these definitions, the closed form MultiDiffusion update for direct pixel or latent-space samples is
\begin{equation}
\Psi(J_t \mid z)
=
\frac{
\sum_{i=1}^{n}
    U_i(W_i
    \otimes
    \Phi(J_t[R_i] \mid y_i))
}{
    \sum_{j=1}^{n} U_j(W_j)
}
\label{eq:multidiff}
\end{equation}
where $\otimes$ denotes the Hadamard product. This expression represents a weighted average of all local denoising predictions, where each window contributes according to its weight map. The result is a global update that reconciles all overlapping diffusion paths into a single image.

Although MultiDiffusion elegantly unifies local diffusion paths, it remains constrained to bounded domains: the process assumes a finite number of windows and requires the pretrained diffusion model to be evaluated at all windows to complete one step. Extending the same principle to infinite domains therefore requires reformulating $\Psi$ so that it operates locally and independently of global window layouts, a key step towards the InfiniteDiffusion algorithm introduced next.

\subsection{From MultiDiffusion to InfiniteDiffusion}

We now seek to extend MultiDiffusion beyond finite image domains. We first redefine the MultiDiffusion image space as an unbounded image, $\mathcal{J} = \mathbb{R}^{\mathbb{Z} \times \mathbb{Z} \times C}$, so that generation produces an infinite output. Consequently, we now define a region to be any rectangular subset of $\mathbb{Z}^2$. Since generation is now over an infinite domain, window indices must now range over a countably infinite set $S$.

For all applications shown in this work, we take $S = \mathbb{Z}^2$, so each window is
indexed by $(i, j)$, and each window region $R_{ij}$ is defined as a square sliding
window with side length $H = W$ and stride $s$ on both axes. $s$ is typically proportional to $H$ and $W$. Concretely, $R_{ij} = [is,\, is + H) \times [js,\, js + W)$. This particular layout is not essential to the InfiniteDiffusion formulation and serves only as an implementation choice for the experiments.

With an infinite number of windows, the MultiDiffusion update becomes intractable, and computation requires an infinite sum to produce the final image. Instead, we seek to generate the image lazily, by only evaluating particular regions $R$. To achieve this, we define $\kappa$ to be the function mapping a region to the set of window indices that overlap it. We assume $|\kappa(R)|$ is always finite. We provide $\kappa$ for the sliding-window case in Appendix ~\ref{appendix:infinitediffusion-properties}. This enables the InfiniteDiffusion update

\begin{equation}
\Psi(J_t \mid z)[R]
=
\left(
\frac{
\sum_{i \in \kappa(R)}
    U_i(W_i \otimes \Phi(J_t[R_i] \mid y_i))
}{
    \sum_{j \in \kappa(R)} U_j(W_j)
}
\right)[R],
\label{eq:infinidiff}
\end{equation}

which is the MultiDiffusion update with only the windows intersecting $R$ evaluated. In the finite setting, the full image $J_t$ can be generated in advance, making each $J_t[R_i]$ effectively free. In the infinite setting, precomputing $J_t$ is impossible, so evaluating $J_t[R_i]$ requires recursively invoking the same update. A naive implementation would therefore incur exponentially growing compute.

\begin{figure}[t]
\centering
\Description{A conceptual visualization of InfiniteDiffusion with sliding windows.
$J_2[R_1]$ (bottom) is computed first as pure gaussian noise. Its output is used to compute $J_1[R_0]$ (middle), which in turn produces the final result $J_0[R]$ (top). Each region is larger than the one above it, providing surrounding context for the next computation.}
\includegraphics[width=\columnwidth]{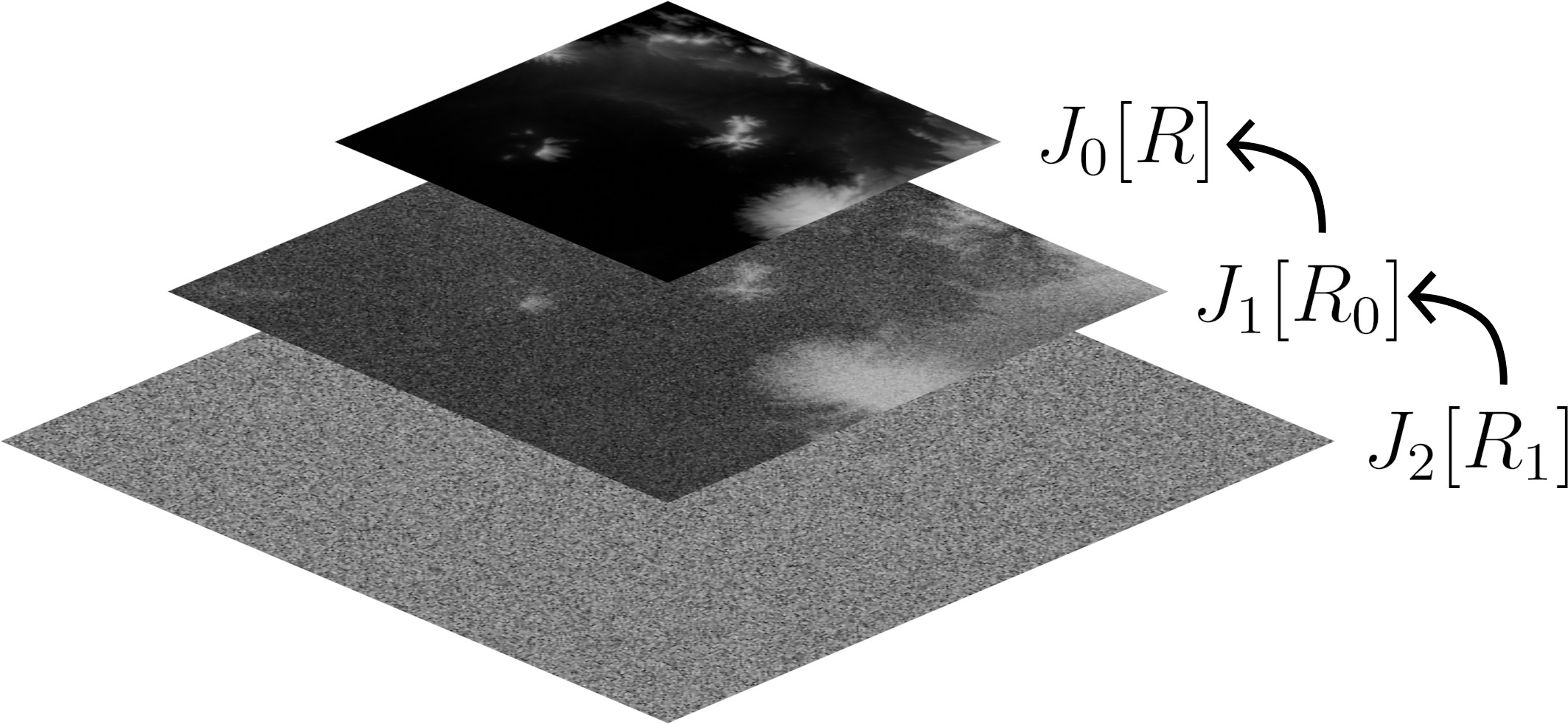}
\caption{A conceptual visualization of InfiniteDiffusion with sliding windows.
$J_2[R_1]$ (bottom) is computed first as pure gaussian noise. Its output is used to compute $J_1[R_0]$ (middle), which in turn produces the final result $J_0[R]$ (top). Each region is larger than the one above it, providing surrounding context for the next computation.}
\label{fig:infinite_diffusion_viz}
\end{figure}

\subsection{Practical Querying of InfiniteDiffusion}
\label{sec:practical_querying}

To make queries practical, we avoid recomputing the same window updates across recursive calls. Instead, for each image $J_t$ we maintain two corresponding sparse infinite tensors: $A_t$, which stores the numerator in Eq. ~\ref{eq:infinidiff}, and $B_t$, which stores the denominator. When a query $J_t[R]$ occurs, we identify all the windows required to generate the region, and process all previously unprocessed windows, populating the desired regions of $A_t$ and $B_t$. Then $J_t[R] = A_t[R] / B_t[R]$. Final generation proceeds as a recursive process that begins by sampling $J_T$ as Gaussian noise. $J_0[R]$ is obtained by recursively applying the query routine at all earlier steps. In summary, Algorithm ~\ref{alg:infinite_diffusion_step} below computes $J_t[R]$ by evaluating only the windows that overlap $R$, and caching each window's contribution in $A_t$ and $B_t$. The tensors $A_t$, $B_t$, and the set of processed windows $P_t$ are mutated in-place. See Figure ~\ref{fig:infinite_diffusion_viz} for a visualization.

\begin{algorithm}[H]
\caption{Querying $J_t[R]$ ($t$ < $T$) with InfiniteDiffusion.}
\label{alg:infinite_diffusion_step}

\begin{algorithmic}

\State \textbf{Inputs:}

\begin{tabbing}
\hspace{0em} \= \hspace{2.5em} \= \hspace{1em} \= \kill
\> $\Phi$ \> $\vartriangleright$ \> pretrained diffusion model\\
\> $J_{t+1}$               \> $\vartriangleright$ \> infinite noisy input image \\
\> $A_t$           \> $\vartriangleright$ \> infinite accumulated output image \\
\> $B_t$           \> $\vartriangleright$ \> infinite accumulated weights for $A_t$ \\
\> $R$                 \> $\vartriangleright$ \> region to query \\
\> $P_t$                 \> $\vartriangleright$ \> set of processed windows \\
\end{tabbing}

\For{each window $i$ in $\kappa(R) \setminus P$}
    \State $A_t[R_i] \gets A_t[R_i] + W_i \otimes \Phi(J_{t+1}[R_i] \mid y_i)$
    \State $B_t[R_i] \gets B_t[R_i] + W_i$
\EndFor
\State $P_t \gets P_t \cup \kappa(R)$

\\
\State \textbf{Output:} $J_t[R] = A_t[R] / B_t[R]$

\end{algorithmic}
\end{algorithm}

To keep storage bounded, we store each infinite tensor as a set of per-window contributions $(i, x_i)$,  where $x_i = W_i \otimes \Phi(J_{t+1}[R_i] \mid y_i)$ for $A_t$, and $x_i=W_i$ for $B_t$. Evaluating $\mathcal{T}_t[R]$ accumulates each $x_i$ with $i \in \kappa(R)$ into an output buffer. Algorithm~\ref{alg:infinite_diffusion_step} guarantees that all required pairs are present before evaluation occurs. This representation supports a lightweight in-memory LRU cache: evicting a pair $(i, x_i)$ is safe as long as $i$ is simultaneously removed from $P_t$, causing that window to be treated as unprocessed and recomputed on the next query.

\subsection{Tractability via Truncated $T$}
\label{sec:truncated_t}

There is one final but critical barrier in making InfiniteDiffusion practical. Each query $J_t[R]$ typically requires a region of $J_{t+1}$ larger than $R$ itself, as pixels near the edge of $R$ are generated by windows that extend beyond $R$, making large $T$ prohibitively expensive, particularly for the first query. For traditional MultiDiffusion, which relies on continuously fusing diffusion paths over dozens of steps, this complexity renders infinite generation effectively impossible. 

To overcome this, we redefine $\Phi$ as an arbitrary denoising function, such as a few-step consistency model, or a sequence of standard diffusion steps, rather than a single atomic step. This insight decouples $T$ from the internal diffusion schedule, which may be much longer. We find that our framework retains coherence even when $T$ is aggressively truncated. To quantify this, we apply InfiniteDiffusion to a standard text-to-image model: Stable Diffusion.

Table ~\ref{tab:infinidiff-fid} compares the FID of InfiniteDiffusion with $T=1$ and $T=2$ to bounded MultiDiffusion (effectively $T=50$) on a reference dataset generated with Stable Diffusion across 8 distinct prompts. Fig. ~\ref{fig:infinite_diffusion_comparison} shows side-by-side comparisons. $T=1$ retains overall structure but introduces artifacts in many cases, while $T=2$ effectively saturates FID and eliminates most artifacts. In Appendix ~\ref{app:panorama_details}, we provide additional samples and runtime comparisons, and show that $T=5$ eliminates practically all artifacts across every scenario tested, with diminishing overhead for larger regions.

\begin{table}[h!]
\caption{FID for InfiniteDiffusion and MultiDiffusion on a reference Stable Diffusion dataset across 8 prompts. See Appendix ~\ref{app:table1methodology} for methodology.}
\label{tab:infinidiff-fid}
\centering
\begin{tabular}{lcc}
\toprule
\textbf{Method} & \textbf{FID $\downarrow$} \\
\midrule
InfiniteDiffusion ($T=1$) & 7.20 \\
InfiniteDiffusion ($T=2$) & 5.83 \\
MultiDiffusion & 5.75 \\
\midrule
Stable Diffusion & 1.94 \\
\bottomrule
\end{tabular}
\label{tab:fid}
\end{table}

\subsection{Properties of InfiniteDiffusion}
\emph{Formal proofs for all properties are provided in Appendix ~\ref{appendix:infinitediffusion-properties}.}
\paragraph{Seed consistency.}
A central property of InfiniteDiffusion is that the entire infinite output is completely determined by the input seed. Once a seed is fixed, the initial noise $J_T$ is fully determined. Since every subsequent step is a deterministic function of the previous one, every region $J_t[R]$ becomes a deterministic function of the seed alone, independent of query order or the status of the cache.

\paragraph{Constant-time random access.}
Assuming $|\kappa(R)| \leq M$, with $M$ depending only on the size of the region, not its location, InfiniteDiffusion guarantees constant-time random access. In particular, for any fixed-size region $R$, the query $J_0[R]$ is $O(1)$, independent of query location or previous evaluations. Combined with seed-consistency, this makes InfiniteDiffusion functionally stateless: any region can be efficiently queried independently, at any time, in any order, and the result is guaranteed to be identical regardless of what has been queried before.

\paragraph{Parallelization.}
InfiniteDiffusion also admits parallel evaluation of window updates. For any fixed timestep $t$, each evaluation of $\Phi$ is independent, so they can be batched and executed in parallel.

\subsection{An Open Source Infinite Tensor Framework}

To support the creation, storage, and manipulation of infinite tensors without exceeding memory limits, we introduce the Infinite Tensor framework, \textcolor{blue}{\href{https://github.com/xandergos/infinite-tensor}{a Python library}} that enables sliding window computation over tensors with infinite dimensions. The framework allows querying of implicitly infinite tensors as if they were standard PyTorch tensors while keeping only the visible region and a bounded transient cache in memory. Each operation is performed through a fixed-sized sliding window that dynamically loads, evicts, or deletes data as sampling progresses. Tiled disk usage is also supported for space-efficient persistent storage. Additional details on this framework are in Appendix ~\ref{sec:infinite_tensor_appendix}.

This abstraction lets diffusion and consistency models operate directly on infinite images without manual data management. Windows can overlap to provide context and blend results, and multiple infinite tensors can depend on one another to form hierarchical pipelines. The framework serves as the runtime layer that links local model inference with practical infinite synthesis. 

Next, we introduce Terrain Diffusion, which builds on this foundation by combining these capabilities with large-scale real-world training data, hierarchical modeling, and a task-specific architecture.

\section{A Global Terrain Dataset}

To enable truly global generation, we construct a seamless global elevation dataset merging land topography from the 90m MERIT DEM \cite{yamazaki_highaccuracy_2017} and ocean bathymetry from ETOPO1 \cite{noaa_national_geophysical_data_center_etopo1_2009}, supplemented by climatic data from WorldClim \cite{fick_worldclim_2017}. We process this data into equal-area tiles by dynamically stretching longitude based on latitude; this ensures that pixel sizes represent a consistent physical area, allowing the diffusion model to learn features with minimal polar distortion. The dataset is split into 2048$\times$2048 tiles, with 80\% randomly assigned to the training set, and the reamining 20\% left for validation. Specific details on coastline smoothing, bathymetric merging, and sampling heuristics are provided in Appendix ~\ref{sec:dataset_appendix}.

\section{Hierarchical Modeling \& Stabilization}

This section outlines the hierarchical architecture and data representation underlying our pipeline, which together enable coherent, high-fidelity, and multi-scale terrain generation.

\subsection{Signed Square-Root Transform}

Terrain tiles vary significantly in elevation range. Under a fixed noise distribution, this leads to uneven effective SNR: low-relief regions behave as though exposed to stronger noise, while high-relief terrain is affected much less. In the raw elevation space, tiles with higher absolute elevations exhibit larger variance. To reduce this variation, we apply a signed square-root transform $z \mapsto \mathrm{sign}(z)\sqrt{|z|}$. The transform compresses high-relief values and distributes variance across tiles more uniformly, decreasing the correlation between the mean and log standard deviation of tiles from $0.66$ to $0.31$. In practice, this allows us to train with a more focused noise distribution and enhances the visibility of small features, especially coastlines. Additional details in Appendix ~\ref{sec:signed-sqrt-appendix}.

\begin{figure*}[t]
\includegraphics[width=\textwidth]{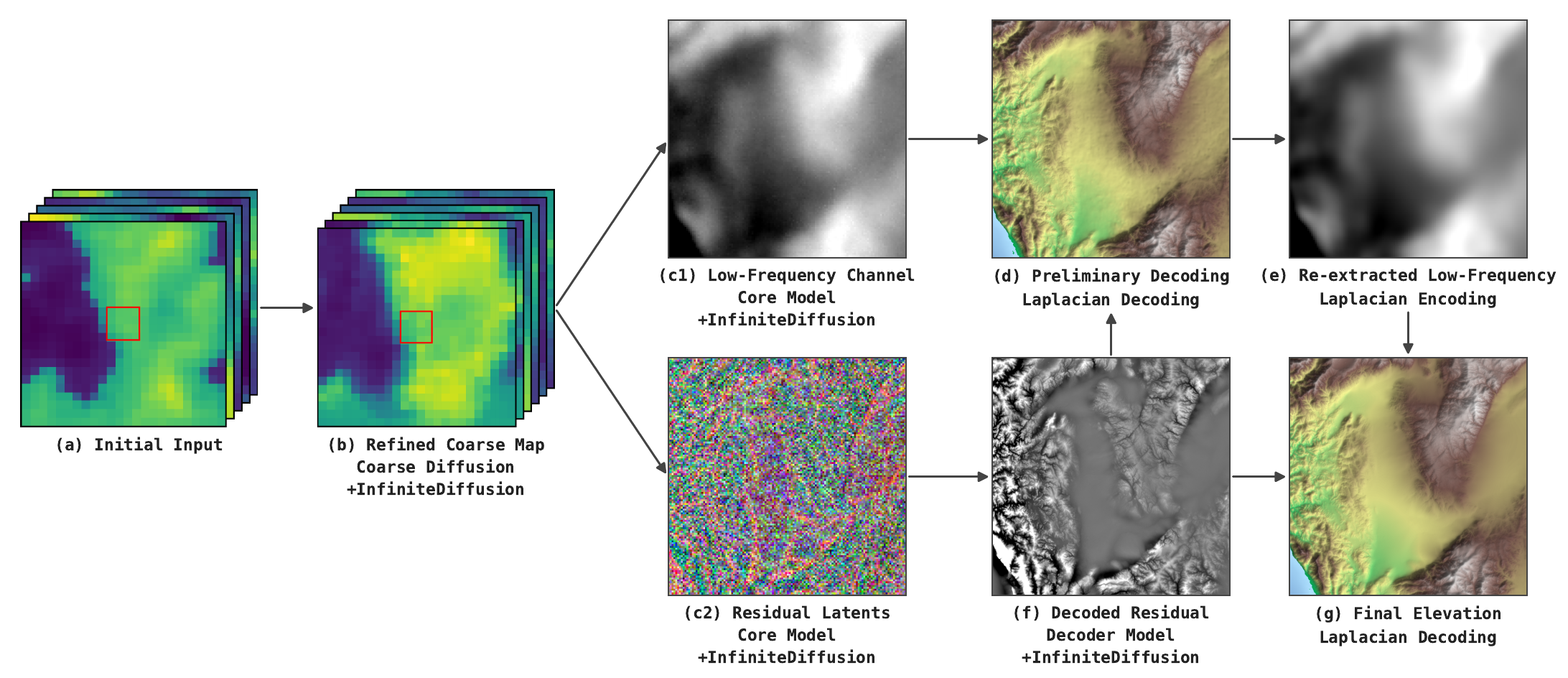}
\caption{The Terrain Diffusion Pipeline. (a) An initial input guides continental layouts and climate. The red box indicates the queried region. (b) The initial input is refined into a realistic coarse map. (c1, c2) The latent model generates the low-frequency channel and the residual latents simultaneously. (f) The residual latents are decoded into a full-resolution residual. (d, e) Laplacian denoising is used to eliminate errors in the low-frequency channel. (g) The denoised low-frequency channel and residual are merged into the final output. "+InfiniteDiffusion" indicates that an image represents only a region of an infinite image generated with InfiniteDiffusion. (d, e, f) are computed per-query.}
\label{fig:td_pipeline_viz}
\end{figure*}

\subsection{Stabilization Via Laplacian Encodings}

Due to normalization, the large dynamic range of Earth elevations make model errors deceptively large in absolute units. Even relatively small errors of $\sigma=0.01$ can correspond to $\pm25$m noise after denormalization. To mitigate this, we predict a Laplacian-based representation comprising a low-frequency component, obtained by downsampling and blurring the original image, and a residual/high-frequency component given by subtracting the upsampled low frequency component from the original image. 

Residual errors are over $30\times$ smaller in magnitude due to their lower variance. To clean the low-frequency channel after generation, we decode the noisy low- and high-frequency components $(L + H)$ into a provisional heightmap, then downsample and blur it to re-extract a denoised low-frequency $\hat{L}$, with any high-frequency noise redirected to the residual $\hat{H}$. Final synthesis uses $\hat{L} + H$, so low-frequency errors vanish while high-frequency detail is preserved. In practice, $L \approx \hat{L}$ even under strong synthetic noise, confirming that re-extraction cleanly isolates low-frequency structure. In this work, we downsample $8\times$ followed by a $\sigma=5$ blur. Our Laplacian denoising step reduces the FID \cite{heusel_gans_2018} of the untiled core diffusion model (introduced next) from $21.51$ to $8.11$, and the corresponding consistency model, which powers our final pipeline, from $75.15$ to $12.72$.

\subsection{A Hierarchical Model}

Planetary terrain spans several orders of magnitude in scale, from continental structure to meter-level detail, making one-pass generation infeasible. Several previous works \cite{du2024demofusion, sharma_earthgen_2024, 11209478, Lee_Kim_Kim_Sung_2023, Zhou_Tang_2024} have shown that MultiDiffusion produces repetitive results when poorly conditioned, a problem that InfiniteDiffusion inherits. We therefore organize generation into a small hierarchy of models operating at progressively finer resolutions. Each stage refines and conditions on the one above, maintaining large-scale coherence while producing realistic local detail. All models share a common EDM2 \cite{Karras2024edm2} backbone with the modifications proposed in sCM \cite{lu_simplifying_2025}. The hierarchy begins with a coarse planetary model, which generates the basic structure (23km/pixel) of the world from a rough, procedural or user-provided layout. The next stage is the core latent diffusion model \cite{rombach2021highresolution}, which transforms that rough structure into realistic 46km tiles in latent space. Finally, a consistency decoder expands these latents into a high-fidelity elevation map. All models are trained on random crops to learn a nearly translation-invariant representation, reflecting the fact that generation should be independent of absolute location. We visualize the coarse-to-fine pipeline in Figure ~\ref{fig:td_pipeline_viz}.

The core latent diffusion model synthesizes $512 \times 512$ patches at a 90m resolution in signed-sqrt space, corresponding to 46km tiles. It predicts a 64x64 low-frequency elevation channel and latent map that compactly represents the corresponding $512 \times 512$ image. To supply the latent codes, we train a separate VAE-style autoencoder that shares the same U-Net backbone but omits diffusion-specific components and skip connections. The model is optimized using L1 and LPIPS losses with a weak KL term to prevent overfitting. After training, the encoder processes each 2048×2048 tile in the dataset as a whole, and the resulting latent image is precomputed and stored alongside the tile. To maximize local quality, the imprecise VAE decoder is discarded, and a diffusion decoder learns to expand these latents into realistic and high-resolution residuals. Conditioning in the diffusion decoder is implemented by concatenating nearest-neighbor interpolated latents to the noisy input image at each diffusion step.

To facilitate long-range global coherence, the core model is conditioned on $4 \times 4$ patches of elevation data. Each pixel of the patch is about 23km, with the model prediction corresponding to the $2 \times 2$ interior. Each patch contains 3 channels: the mean elevation of the pixel, the 5th percentile elevation of the pixel, and a binary mask indicating which pixels have data available. We also provide the model with the tile's mean temperature, temperature variation, annual precipitation, and precipitation seasonality for additional coherence and user control. Since climatic data is not available in the ocean, we replace missing values with a standard gaussian, ensuring the model accepts any combination of climatic variables in ocean regions. These values are injected into the model as a flattened vector alongside the noise embedding.

\subsection{Real-Time Planetary Scale Synthesis}

The coarse diffusion model facilitates this hierarchy by producing the conditioning variables required by the core latent model. The user provides initial maps for this model using hand-drawn sketches or procedural noise, which we found works as well as purely learned methods while providing additional control. During inference, these inputs are corrupted with gaussian noise according to the user's preference on a per-channel basis, and concatenated against the usual diffusion inputs. The model follows the EDM2 design but with no downsampling or upsampling operations. This limits the receptive field by design, preventing the model from drifting toward the massive continental structures present in Earth data and avoiding conflicts where global priors override user guidance while still allowing strong local corrections.

To enable real-time streaming, all diffusion models, except the coarse model, are distilled into continuous-time consistency models \cite{lu_simplifying_2025}. To improve fidelity further, we apply the guidance scheme proposed in AutoGuidance \cite{Karras2024autoguidance}. Combined, these stages form a complete generation pipeline, from planetary context to local detail, capable of on-demand, real-time synthesis.
\section{Results}

We evaluate Terrain Diffusion's visual fidelity and latency using a single NVIDIA RTX 3090 Ti. The core consistency model uses two-step generation (stride 32, batch size 16). The coarse and decoder stages use one-step generation ($T=1$, batch size 1), with the decoder employing size $512$ windows with stride $384$. We use a separable linear weight window (decaying from 1 at the center to $\epsilon$ at the boundary) which improves FID from $19.32$ to $14.78$ compared to a constant map.

\subsection{Visual Fidelity}

We perform internal ablations to isolate the effects of our proposed architecture. We calculate FID \cite{heusel_gans_2018} for (1) non-tiled diffusion samples, (2) non-tiled consistency samples, and (3) tiled samples generated with InfiniteDiffusion\footnote{FID is computed on central 984×984 crops of the validation tiles to ensure adequate surrounding context. InfiniteDiffusion is applied only to the latent space; tiling the decoder would require impractically large context, and the decoder already produces near seamless outputs. (Appendix ~\ref{appendix:decoder_details})}. This isolates base fidelity, the effect of consistency distillation, and the effect of tiling. 

We also compare to Perlin blending \cite{10.1145/3571600.3571657}, where tiles are blended with perlin noise, and naive tiling, where tiles are simply concatenated without blending. To investigate the effect of latent blending, we evaluate InfiniteDiffusion with $T = 1$ and $T = 2$ (i.e. blending once or twice in the latent space). We also evaluate 'naive InfiniteDiffusion' ($T = 0$), which combines the typically separate latent model and decoder into a single stage. Table ~\ref{tab:fid} shows our results.

\begin{table}[h!]
\caption{FID-50k for generations with InfiniteDiffusion vs. other methods. Lower is better. The distilled theoretical lower bound is underlined. The best tiled result is bolded.}
\centering
\begin{tabular}{lcc}
\toprule
\textbf{Tiling} & Distilled & \textbf{FID $\downarrow$} \\
\midrule
Perlin Blending \cite{10.1145/3571600.3571657} & \checkmark & 186.70 \\
Naive Tiling & \checkmark & 74.44 \\
Naive InfiniteDiffusion ($T=0$) & \checkmark & 27.61\\
InfiniteDiffusion ($T=1$) & \checkmark & 19.78 \\
InfiniteDiffusion ($T=2$) & \checkmark & \textbf{14.78} \\
\midrule
None & \checkmark & \underline{12.72} \\
None & $\times$ & 8.11 \\
\bottomrule
\end{tabular}
\label{tab:fid}
\end{table}

Perlin blending yields a high FID of 186.70, confirming that procedural blending techniques fail to approximate the statistical distribution of natural topography. Naive InfiniteDiffusion, without blending in the latent space, is able to achieve good quality, which we attribute primarily to guidance from the hierarchy preserving structure. InfiniteDiffusion with $T=1$ closes the fidelity gap with the base model further. Most notably, InfiniteDiffusion with $T=2$ preserves the fidelity of the base consistency model with remarkable accuracy (14.78 vs. 12.72), successfully maintaining perceptual and spatial continuity, and scaling from finite training crops to infinite worlds with little degradation in quality. 

\subsection{Qualitative Analysis}

Figure~\ref{fig:terrain_grid} shows 20 1024$\times$1024 tiles from Terrain Diffusion, all from the same seed used for Fig. ~\ref{fig:teaser}. The model produces sharp ridges, visually coherent river basins, smooth transitions, and varied landscapes. No visible tiling artifacts confirm the effectiveness of InfiniteDiffusion. See Appendix ~\ref{appendix:fid-details} (Fig. ~\ref{fig:perlin_blending_samples}) for examples of Perlin Blending.

To demonstrate the practical utility of this approach, we integrate Terrain Diffusion into the Minecraft engine by replacing the native world generator. Elevation and biome queries are routed through our model, and climatic outputs are mapped to Minecraft biomes using a lightweight rule set. The system streams terrain in real time and handles arbitrary traversal, with runtime dominated by Minecraft’s own generation logic rather than our model. Figure~\ref{fig:minecraft_grid} shows representative in-game terrain. For these interactive visualizations, we apply bilinear interpolation to upsample the heightmaps $4 \times$.

\subsection{Latency: Time to First and Second Tile}
\label{sec:latency}

Because generation of any fixed-size region has bounded cost, generating a contiguous n-length strip is $O(n)$. But neighboring tiles reuse cached context, so the cost for querying the first region is larger than for subsequent regions.

Motivated by these facts, we measure end-to-end latency as the \textit{time to first tile} (TTFT) and \textit{time to second tile} (TTST) across resolutions. TTFT denotes the delay from model initialization to the first $512\times512$ tile becoming available, reflecting initial setup cost or a full cache miss, such as from teleportation. TTST measures the time to generate a neighboring tile thereafter, reflecting a lower bound on interactive exploration performance: performance improves further with more nearby cached tiles. While both metrics are bounded, they vary with the specific region location because the number of intersecting windows differs across positions. To account for this, we perform 1000 runs at random locations and report the average and standard deviation. Our results are reported in Table ~\ref{tab:latency}.

\begin{table}[h!]
\caption{Generation latency for the first and second tile. The best seamless result is bolded.}
\label{tab:latency}
\centering
\begin{tabular}{lccc}
\toprule
\textbf{Method} & \textbf{TTFT (s) $\downarrow$} & \textbf{TTST (s) $\downarrow$} \\
\midrule
InfiniteDiffusion ($T=2$) & $1.72 \pm 0.19$ & $0.66 \pm 0.18$ \\
InfiniteDiffusion ($T=1$) & $\mathbf{1.39 \pm 0.20}$ & $\mathbf{0.63 \pm 0.16}$ \\
Naive InfiniteDiffusion  & $2.05 \pm 0.22$ & $1.60 \pm 0.17$ \\
\midrule
Independent Tiles         & $0.51 \pm 0.05$ & $0.22 \pm 0.02$ \\
\bottomrule
\end{tabular}
\end{table}

With our highest quality configuration, an F-35, one of the fastest conventional aircraft in service at roughly 550 m/s, would traverse a 512×512 tile at 90 m resolution in about 84 seconds. In that time, Terrain Diffusion can produce 130 additional tiles. In a 1:15 miniature world, where a vehicle at highway speeds (60mph) effectively encounters terrain at 405 m/s, the system maintains a 170$\times$ performance buffer. Even at the theoretical extreme of orbital velocity ($\approx 7,700$ m/s), generation remains 9$\times$ faster than traversal. 

InfiniteDiffusion with $T = 1$ is moderately faster than $T = 2$ in TTFT, but performs similarly in TTST, demonstrating that the effect of larger $T$ is most impactful for the first query. Naive InfiniteDiffusion performs poorly despite using fewer blending steps, since the decoder must use a stride of 256 instead of 384 to align with the latent model. Independent tiling is about 3x faster than InfiniteDiffusion, roughly aligning with our stride choices of $50\%$ in the latent model (4x overhead) and $75\%$ in the decoder (1.78x overhead). A minimal cache ($\approx$ 10MB) achieves optimal TTFT and TTST performance by only memoizing the windows needed for the previous query, though we typically use a larger cache to achieve superior performance for each additional tile.

\section{Discussion}
\label{sec:discussion}

\paragraph{Limitations.} For more ambiguous prompts, InfiniteDiffusion with $T = 2$ can introduce some artifacts. While a few additional blending steps ($T=5$) typically resolves this, this can limit the use of two-step models and adds some overhead, particularly for TTFT. A second limitation is that, unlike point-wise procedural noise, InfiniteDiffusion generates in windows, making queries for significantly smaller regions inefficient if uncached. This primarily impacts scattered point queries, leaving Minecraft features reliant on long-range biome searches (namely \texttt{/locate biome} and explorer maps) unsupported. Finally, while InfiniteDiffusion effectively supports deep cascades of diffusion models, enabling immense scale, acquiring training data at the coarsest layer remains a challenge. Synthetic data offers a promising avenue for many applications, effectively distilling any other model or simulation into an approximate procedural counterpart.

\paragraph{As a procedural noise. } Although Terrain Diffusion retains the formal guarantees of procedural noise, it fundamentally diverges in fidelity and computational cost. By leveraging deep generative models, our method captures complex structure that noise functions only approximate as stationary textures. While this comes at the cost of raw throughput compared to the microsecond-level latency of noise, our warmup time of 1.72 seconds and steady-state generation time of 0.66 seconds remain highly practical. Future applications may distill complex but established physical simulations into practical procedural approximations that inherit both the realism of simulation and the functional utility of noise.

\section{Future Work}
For Terrain Diffusion, adding features to the hierarchy is a natural next step. Either as outputs, conditioning, or both. The coarse model, the base model, or both could incorporate additional variables such as soil properties,  other climatic variables, or satellite imagery, enhancing control and enabling additional downstream applications. The InfiniteDiffusion algorithm could also be applied to additional dimensions, venturing into voxel-based methods for 3D synthesis. 

InfiniteDiffusion also scales favorably with cascade depth. Existing approaches to ultra-high-resolution generation face a fundamental trade-off: MultiDiffusion scales exponentially in compute, while cropping-based methods scale linearly but discard context at every step, limiting exploration. InfiniteDiffusion resolves this through lazy evaluation: a user can zoom into any region of an implicitly ultra-high-resolution image while computing only the queried region.

After the first query, further navigation becomes even cheaper. Because panning by $S$ pixels at the finest layer requires moving only $S/a$ pixels at the next coarser layer, $S/a^2$ at the one above, and so on, cached computations at coarser layers remain almost entirely valid between queries. The amortized cost per $S \times S$ region is therefore
\[
    O(S^2)\left(1 + \frac{1}{a} + \frac{1}{a^2} + \cdots\right) = O(S^2) \cdot \frac{a}{a-1},
\]
independent of cascade depth, and just $\frac{4}{3} \approx 1.33\times$ the cost of a single layer for $4\times$ super-resolution at infinite depth. In other words, adding unlimited resolution to an already infinite world is nearly free.
\section{Conclusion}

We have presented InfiniteDiffusion, an algorithm for lazily sampling diffusion models across unbounded domains with seed-consistency and constant-time random-access. We further introduce Terrain Diffusion, a framework for practical learned procedural terrain generation with unprecedented realism. Together, these components position diffusion models as a practical foundation for learned procedural worldbuilding and infinite generation in general.

\section{Compute Statement}

Almost all training was performed on an RTX 3090 Ti, and all experiments fit within 24\,GB of VRAM. 
Training the entire pipeline end-to-end takes approximately two weeks on an RTX 3090 Ti, or one week with an RTX 5090.

\bibliographystyle{ACM-Reference-Format}
\bibliography{references}


\begin{thebibliography}{36}


\ifx \showCODEN    \undefined \def \showCODEN     #1{\unskip}     \fi
\ifx \showISBNx    \undefined \def \showISBNx     #1{\unskip}     \fi
\ifx \showISBNxiii \undefined \def \showISBNxiii  #1{\unskip}     \fi
\ifx \showISSN     \undefined \def \showISSN      #1{\unskip}     \fi
\ifx \showLCCN     \undefined \def \showLCCN      #1{\unskip}     \fi
\ifx \shownote     \undefined \def \shownote      #1{#1}          \fi
\ifx \showarticletitle \undefined \def \showarticletitle #1{#1}   \fi
\ifx \showURL      \undefined \def \showURL       {\relax}        \fi
\providecommand\bibfield[2]{#2}
\providecommand\bibinfo[2]{#2}
\providecommand\natexlab[1]{#1}
\providecommand\showeprint[2][]{arXiv:#2}

\bibitem[Argudo et~al\mbox{.}(2018)]%
        {https://doi.org/10.1111/cgf.13345}
\bibfield{author}{\bibinfo{person}{O. Argudo}, \bibinfo{person}{A. Chica}, {and} \bibinfo{person}{C. Andujar}.} \bibinfo{year}{2018}\natexlab{}.
\newblock \showarticletitle{Terrain Super-resolution through Aerial Imagery and Fully Convolutional Networks}.
\newblock \bibinfo{journal}{\emph{Computer Graphics Forum}} \bibinfo{volume}{37}, \bibinfo{number}{2} (\bibinfo{year}{2018}), \bibinfo{pages}{101--110}.
\newblock
\showeprint{https://onlinelibrary.wiley.com/doi/pdf/10.1111/cgf.13345}
\href{https://doi.org/10.1111/cgf.13345}{doi:\nolinkurl{10.1111/cgf.13345}}


\bibitem[Bar-Tal et~al\mbox{.}(2023)]%
        {bar-tal_multidiffusion_2023}
\bibfield{author}{\bibinfo{person}{Omer Bar-Tal}, \bibinfo{person}{Lior Yariv}, \bibinfo{person}{Yaron Lipman}, {and} \bibinfo{person}{Tali Dekel}.} \bibinfo{year}{2023}\natexlab{}.
\newblock \showarticletitle{{MultiDiffusion}: {Fusing} {Diffusion} {Paths} for {Controlled} {Image} {Generation}}.
\newblock \bibinfo{journal}{\emph{arXiv preprint arXiv:2302.08113}} (\bibinfo{year}{2023}).
\newblock


\bibitem[Beckham and Pal(2017)]%
        {beckham_step_2017}
\bibfield{author}{\bibinfo{person}{Christopher Beckham} {and} \bibinfo{person}{Christopher Pal}.} \bibinfo{year}{2017}\natexlab{}.
\newblock \bibinfo{title}{A step towards procedural terrain generation with {GANs}}.
\newblock
\href{https://doi.org/10.48550/ARXIV.1707.03383}{doi:\nolinkurl{10.48550/ARXIV.1707.03383}}
\newblock
\shownote{Version Number: 1}.


\bibitem[Bińkowski et~al\mbox{.}(2018)]%
        {binkowski_demystifying_2018}
\bibfield{author}{\bibinfo{person}{Mikołaj Bińkowski}, \bibinfo{person}{Dougal~J. Sutherland}, \bibinfo{person}{Michael Arbel}, {and} \bibinfo{person}{Arthur Gretton}.} \bibinfo{year}{2018}\natexlab{}.
\newblock \showarticletitle{Demystifying {MMD} {GANs}}. In \bibinfo{booktitle}{\emph{International {Conference} on {Learning} {Representations}}}.
\newblock
\urldef\tempurl%
\url{https://openreview.net/forum?id=r1lUOzWCW}
\showURL{%
\tempurl}


\bibitem[Borne-Pons et~al\mbox{.}(2025)]%
        {bornepons2025mesatextdriventerraingeneration}
\bibfield{author}{\bibinfo{person}{Paul Borne-Pons}, \bibinfo{person}{Mikolaj Czerkawski}, \bibinfo{person}{Rosalie Martin}, {and} \bibinfo{person}{Romain Rouffet}.} \bibinfo{year}{2025}\natexlab{}.
\newblock \bibinfo{title}{MESA: Text-Driven Terrain Generation Using Latent Diffusion and Global Copernicus Data}.
\newblock
\showeprint[arxiv]{2504.07210}~[cs.GR]
\urldef\tempurl%
\url{https://arxiv.org/abs/2504.07210}
\showURL{%
\tempurl}


\bibitem[Du et~al\mbox{.}(2024)]%
        {du2024demofusion}
\bibfield{author}{\bibinfo{person}{Ruoyi Du}, \bibinfo{person}{Dongliang Chang}, \bibinfo{person}{Timothy Hospedales}, \bibinfo{person}{Yi-Zhe Song}, {and} \bibinfo{person}{Zhanyu Ma}.} \bibinfo{year}{2024}\natexlab{}.
\newblock \showarticletitle{DemoFusion: Democratising High-Resolution Image Generation With No \$\$\$}. In \bibinfo{booktitle}{\emph{CVPR}}.
\newblock


\bibitem[Fick and Hijmans(2017)]%
        {fick_worldclim_2017}
\bibfield{author}{\bibinfo{person}{Stephen~E. Fick} {and} \bibinfo{person}{Robert~J. Hijmans}.} \bibinfo{year}{2017}\natexlab{}.
\newblock \showarticletitle{{WorldClim} 2: new 1‐km spatial resolution climate surfaces for global land areas}.
\newblock \bibinfo{journal}{\emph{International Journal of Climatology}} \bibinfo{volume}{37}, \bibinfo{number}{12} (\bibinfo{date}{Oct.} \bibinfo{year}{2017}), \bibinfo{pages}{4302--4315}.
\newblock
\showISSN{0899-8418, 1097-0088}
\href{https://doi.org/10.1002/joc.5086}{doi:\nolinkurl{10.1002/joc.5086}}


\bibitem[Fournier et~al\mbox{.}(1982)]%
        {fournier_computer_1982}
\bibfield{author}{\bibinfo{person}{Alain Fournier}, \bibinfo{person}{Don Fussell}, {and} \bibinfo{person}{Loren Carpenter}.} \bibinfo{year}{1982}\natexlab{}.
\newblock \showarticletitle{Computer rendering of stochastic models}.
\newblock \bibinfo{journal}{\emph{Commun. ACM}} \bibinfo{volume}{25}, \bibinfo{number}{6} (\bibinfo{date}{June} \bibinfo{year}{1982}), \bibinfo{pages}{371--384}.
\newblock
\showISSN{0001-0782, 1557-7317}
\href{https://doi.org/10.1145/358523.358553}{doi:\nolinkurl{10.1145/358523.358553}}


\bibitem[Fr\"{u}hst\"{u}ck et~al\mbox{.}(2019)]%
        {Fruehstueck2019TileGAN}
\bibfield{author}{\bibinfo{person}{Anna Fr\"{u}hst\"{u}ck}, \bibinfo{person}{Ibraheem Alhashim}, {and} \bibinfo{person}{Peter Wonka}.} \bibinfo{year}{2019}\natexlab{}.
\newblock \showarticletitle{{TileGAN}: Synthesis of Large-Scale Non-Homogeneous Textures}.
\newblock \bibinfo{journal}{\emph{ACM Transactions on Graphics (Proc. SIGGRAPH)}} \bibinfo{volume}{38}, \bibinfo{number}{4} (\bibinfo{year}{2019}), \bibinfo{pages}{58:1--58:11}.
\newblock


\bibitem[Goodfellow et~al\mbox{.}(2020)]%
        {10.1145/3422622}
\bibfield{author}{\bibinfo{person}{Ian Goodfellow}, \bibinfo{person}{Jean Pouget-Abadie}, \bibinfo{person}{Mehdi Mirza}, \bibinfo{person}{Bing Xu}, \bibinfo{person}{David Warde-Farley}, \bibinfo{person}{Sherjil Ozair}, \bibinfo{person}{Aaron Courville}, {and} \bibinfo{person}{Yoshua Bengio}.} \bibinfo{year}{2020}\natexlab{}.
\newblock \showarticletitle{Generative adversarial networks}.
\newblock \bibinfo{journal}{\emph{Commun. ACM}} \bibinfo{volume}{63}, \bibinfo{number}{11} (\bibinfo{date}{Oct.} \bibinfo{year}{2020}), \bibinfo{pages}{139–144}.
\newblock
\showISSN{0001-0782}
\href{https://doi.org/10.1145/3422622}{doi:\nolinkurl{10.1145/3422622}}


\bibitem[Guérin et~al\mbox{.}(2017)]%
        {guerin_interactive_2017}
\bibfield{author}{\bibinfo{person}{Éric Guérin}, \bibinfo{person}{Julie Digne}, \bibinfo{person}{Éric Galin}, \bibinfo{person}{Adrien Peytavie}, \bibinfo{person}{Christian Wolf}, \bibinfo{person}{Bedrich Benes}, {and} \bibinfo{person}{Benoît Martinez}.} \bibinfo{year}{2017}\natexlab{}.
\newblock \showarticletitle{Interactive example-based terrain authoring with conditional generative adversarial networks}.
\newblock \bibinfo{journal}{\emph{ACM Transactions on Graphics}} \bibinfo{volume}{36}, \bibinfo{number}{6} (\bibinfo{date}{Dec.} \bibinfo{year}{2017}), \bibinfo{pages}{1--13}.
\newblock
\showISSN{0730-0301, 1557-7368}
\href{https://doi.org/10.1145/3130800.3130804}{doi:\nolinkurl{10.1145/3130800.3130804}}


\bibitem[Heusel et~al\mbox{.}(2018)]%
        {heusel_gans_2018}
\bibfield{author}{\bibinfo{person}{Martin Heusel}, \bibinfo{person}{Hubert Ramsauer}, \bibinfo{person}{Thomas Unterthiner}, \bibinfo{person}{Bernhard Nessler}, {and} \bibinfo{person}{Sepp Hochreiter}.} \bibinfo{year}{2018}\natexlab{}.
\newblock \bibinfo{title}{{GANs} {Trained} by a {Two} {Time}-{Scale} {Update} {Rule} {Converge} to a {Local} {Nash} {Equilibrium}}.
\newblock
\href{https://doi.org/10.48550/arXiv.1706.08500}{doi:\nolinkurl{10.48550/arXiv.1706.08500}}
\newblock
\shownote{arXiv:1706.08500 [cs]}.


\bibitem[Ho et~al\mbox{.}(2020)]%
        {ho_denoising_2020}
\bibfield{author}{\bibinfo{person}{Jonathan Ho}, \bibinfo{person}{Ajay Jain}, {and} \bibinfo{person}{Pieter Abbeel}.} \bibinfo{year}{2020}\natexlab{}.
\newblock \showarticletitle{Denoising {Diffusion} {Probabilistic} {Models}}. In \bibinfo{booktitle}{\emph{Advances in {Neural} {Information} {Processing} {Systems}}}, \bibfield{editor}{\bibinfo{person}{H.~Larochelle}, \bibinfo{person}{M.~Ranzato}, \bibinfo{person}{R.~Hadsell}, \bibinfo{person}{M.~F. Balcan}, {and} \bibinfo{person}{H.~Lin}} (Eds.), Vol.~\bibinfo{volume}{33}. \bibinfo{publisher}{Curran Associates, Inc.}, \bibinfo{pages}{6840--6851}.
\newblock
\urldef\tempurl%
\url{https://proceedings.neurips.cc/paper_files/paper/2020/file/4c5bcfec8584af0d967f1ab10179ca4b-Paper.pdf}
\showURL{%
\tempurl}


\bibitem[Hu et~al\mbox{.}(2024)]%
        {10.1609/aaai.v38i11.29150}
\bibfield{author}{\bibinfo{person}{Zexin Hu}, \bibinfo{person}{Kun Hu}, \bibinfo{person}{Clinton Mo}, \bibinfo{person}{Lei Pan}, {and} \bibinfo{person}{Zhiyong Wang}.} \bibinfo{year}{2024}\natexlab{}.
\newblock \showarticletitle{Terrain diffusion network: climatic-aware terrain generation with geological sketch guidance}. In \bibinfo{booktitle}{\emph{Proceedings of the Thirty-Eighth AAAI Conference on Artificial Intelligence and Thirty-Sixth Conference on Innovative Applications of Artificial Intelligence and Fourteenth Symposium on Educational Advances in Artificial Intelligence}} \emph{(\bibinfo{series}{AAAI'24/IAAI'24/EAAI'24})}. \bibinfo{publisher}{AAAI Press}, Article \bibinfo{articleno}{1402}, \bibinfo{numpages}{9}~pages.
\newblock
\showISBNx{978-1-57735-887-9}
\href{https://doi.org/10.1609/aaai.v38i11.29150}{doi:\nolinkurl{10.1609/aaai.v38i11.29150}}


\bibitem[Jain et~al\mbox{.}(2023)]%
        {10.1145/3571600.3571657}
\bibfield{author}{\bibinfo{person}{Aryamaan Jain}, \bibinfo{person}{Avinash Sharma}, {and} \bibinfo{person}{Rajan}.} \bibinfo{year}{2023}\natexlab{}.
\newblock \showarticletitle{Adaptive \& Multi-Resolution Procedural Infinite Terrain Generation with Diffusion Models and Perlin Noise}. In \bibinfo{booktitle}{\emph{Proceedings of the Thirteenth Indian Conference on Computer Vision, Graphics and Image Processing}} (Gandhinagar, India) \emph{(\bibinfo{series}{ICVGIP '22})}. \bibinfo{publisher}{Association for Computing Machinery}, \bibinfo{address}{New York, NY, USA}, Article \bibinfo{articleno}{55}, \bibinfo{numpages}{9}~pages.
\newblock
\showISBNx{9781450398220}
\href{https://doi.org/10.1145/3571600.3571657}{doi:\nolinkurl{10.1145/3571600.3571657}}


\bibitem[Jiménez(2023)]%
        {jimenez_mixture_2023}
\bibfield{author}{\bibinfo{person}{Álvaro~Barbero Jiménez}.} \bibinfo{year}{2023}\natexlab{}.
\newblock \bibinfo{title}{Mixture of {Diffusers} for scene composition and high resolution image generation}.
\newblock
\href{https://doi.org/10.48550/arXiv.2302.02412}{doi:\nolinkurl{10.48550/arXiv.2302.02412}}
\newblock
\shownote{arXiv:2302.02412 [cs]}.


\bibitem[Karras et~al\mbox{.}(2024a)]%
        {Karras2024autoguidance}
\bibfield{author}{\bibinfo{person}{Tero Karras}, \bibinfo{person}{Miika Aittala}, \bibinfo{person}{Tuomas Kynk\"a\"anniemi}, \bibinfo{person}{Jaakko Lehtinen}, \bibinfo{person}{Timo Aila}, {and} \bibinfo{person}{Samuli Laine}.} \bibinfo{year}{2024}\natexlab{a}.
\newblock \showarticletitle{Guiding a Diffusion Model with a Bad Version of Itself}. In \bibinfo{booktitle}{\emph{Proc. NeurIPS}}.
\newblock


\bibitem[Karras et~al\mbox{.}(2024b)]%
        {Karras2024edm2}
\bibfield{author}{\bibinfo{person}{Tero Karras}, \bibinfo{person}{Miika Aittala}, \bibinfo{person}{Jaakko Lehtinen}, \bibinfo{person}{Janne Hellsten}, \bibinfo{person}{Timo Aila}, {and} \bibinfo{person}{Samuli Laine}.} \bibinfo{year}{2024}\natexlab{b}.
\newblock \showarticletitle{Analyzing and Improving the Training Dynamics of Diffusion Models}. In \bibinfo{booktitle}{\emph{Proc. CVPR}}.
\newblock


\bibitem[Lee et~al\mbox{.}(2023)]%
        {Lee_Kim_Kim_Sung_2023}
\bibfield{author}{\bibinfo{person}{Yuseung Lee}, \bibinfo{person}{Kunho Kim}, \bibinfo{person}{Hyunjin Kim}, {and} \bibinfo{person}{Minhyuk Sung}.} \bibinfo{year}{2023}\natexlab{}.
\newblock \showarticletitle{SyncDiffusion: Coherent Montage via Synchronized Joint Diffusions}. In \bibinfo{booktitle}{\emph{Advances in Neural Information Processing Systems}}, \bibfield{editor}{\bibinfo{person}{A.~Oh}, \bibinfo{person}{T.~Naumann}, \bibinfo{person}{A.~Globerson}, \bibinfo{person}{K.~Saenko}, \bibinfo{person}{M.~Hardt}, {and} \bibinfo{person}{S.~Levine}} (Eds.), Vol.~\bibinfo{volume}{36}. \bibinfo{publisher}{Curran Associates, Inc.}, \bibinfo{pages}{50648–50660}.
\newblock
\urldef\tempurl%
\url{https://proceedings.neurips.cc/paper_files/paper/2023/file/9ee3a664ccfeabc0da16ac6f1f1cfe59-Paper-Conference.pdf}
\showURL{%
\tempurl}


\bibitem[Li et~al\mbox{.}(2025)]%
        {li_worldgrow_2025}
\bibfield{author}{\bibinfo{person}{Sikuang Li}, \bibinfo{person}{Chen Yang}, \bibinfo{person}{Jiemin Fang}, \bibinfo{person}{Taoran Yi}, \bibinfo{person}{Jia Lu}, \bibinfo{person}{Jiazhong Cen}, \bibinfo{person}{Lingxi Xie}, \bibinfo{person}{Wei Shen}, {and} \bibinfo{person}{Qi Tian}.} \bibinfo{year}{2025}\natexlab{}.
\newblock \bibinfo{title}{{WorldGrow}: {Generating} {Infinite} {3D} {World}}.
\newblock
\href{https://doi.org/10.48550/arXiv.2510.21682}{doi:\nolinkurl{10.48550/arXiv.2510.21682}}
\newblock
\shownote{arXiv:2510.21682 [cs]}.


\bibitem[Lin et~al\mbox{.}(2022)]%
        {lin_infinitygan_2022}
\bibfield{author}{\bibinfo{person}{Chieh~Hubert Lin}, \bibinfo{person}{Hsin-Ying Lee}, \bibinfo{person}{Yen-Chi Cheng}, \bibinfo{person}{Sergey Tulyakov}, {and} \bibinfo{person}{Ming-Hsuan Yang}.} \bibinfo{year}{2022}\natexlab{}.
\newblock \showarticletitle{{InfinityGAN}: {Towards} {Infinite}-{Pixel} {Image} {Synthesis}}. In \bibinfo{booktitle}{\emph{International {Conference} on {Learning} {Representations}}}.
\newblock
\urldef\tempurl%
\url{https://openreview.net/forum?id=ufGMqIM0a4b}
\showURL{%
\tempurl}


\bibitem[Lu and Song(2025)]%
        {lu_simplifying_2025}
\bibfield{author}{\bibinfo{person}{Cheng Lu} {and} \bibinfo{person}{Yang Song}.} \bibinfo{year}{2025}\natexlab{}.
\newblock \showarticletitle{Simplifying, {Stabilizing} and {Scaling} {Continuous}-time {Consistency} {Models}}. In \bibinfo{booktitle}{\emph{The {Thirteenth} {International} {Conference} on {Learning} {Representations}}}.
\newblock
\urldef\tempurl%
\url{https://openreview.net/forum?id=LyJi5ugyJx}
\showURL{%
\tempurl}


\bibitem[Maesumi et~al\mbox{.}(2024)]%
        {maesumi2024noise}
\bibfield{author}{\bibinfo{person}{Arman Maesumi}, \bibinfo{person}{Dylan Hu}, \bibinfo{person}{Krishi Saripalli}, \bibinfo{person}{Vladimir~G. Kim}, \bibinfo{person}{Matthew Fisher}, \bibinfo{person}{Sören Pirk}, {and} \bibinfo{person}{Daniel Ritchie}.} \bibinfo{year}{2024}\natexlab{}.
\newblock \showarticletitle{One Noise to Rule Them All: Learning a Unified Model of Spatially-Varying Noise Patterns}.
\newblock  (\bibinfo{year}{2024}).
\newblock


\bibitem[{NOAA National Geophysical Data Center}(2009)]%
        {noaa_national_geophysical_data_center_etopo1_2009}
\bibfield{author}{\bibinfo{person}{{NOAA National Geophysical Data Center}}.} \bibinfo{year}{2009}\natexlab{}.
\newblock \bibinfo{title}{{ETOPO1} 1 {Arc}-{Minute} {Global} {Relief} {Model}}.
\newblock
\href{https://doi.org/10.7289/V5C8276M}{doi:\nolinkurl{10.7289/V5C8276M}}


\bibitem[Perlin(1985)]%
        {perlin_image_1985}
\bibfield{author}{\bibinfo{person}{Ken Perlin}.} \bibinfo{year}{1985}\natexlab{}.
\newblock \showarticletitle{An image synthesizer}.
\newblock \bibinfo{journal}{\emph{SIGGRAPH Comput. Graph.}} \bibinfo{volume}{19}, \bibinfo{number}{3} (\bibinfo{date}{July} \bibinfo{year}{1985}), \bibinfo{pages}{287–296}.
\newblock
\showISSN{0097-8930}
\href{https://doi.org/10.1145/325165.325247}{doi:\nolinkurl{10.1145/325165.325247}}


\bibitem[Perlin(2002)]%
        {perlin_improving_2002}
\bibfield{author}{\bibinfo{person}{Ken Perlin}.} \bibinfo{year}{2002}\natexlab{}.
\newblock \showarticletitle{Improving noise}. In \bibinfo{booktitle}{\emph{Proceedings of the 29th annual conference on {Computer} graphics and interactive techniques}}. \bibinfo{publisher}{ACM}, \bibinfo{address}{San Antonio Texas}, \bibinfo{pages}{681--682}.
\newblock
\showISBNx{978-1-58113-521-3}
\href{https://doi.org/10.1145/566570.566636}{doi:\nolinkurl{10.1145/566570.566636}}


\bibitem[Rombach et~al\mbox{.}(2021)]%
        {rombach2021highresolution}
\bibfield{author}{\bibinfo{person}{Robin Rombach}, \bibinfo{person}{Andreas Blattmann}, \bibinfo{person}{Dominik Lorenz}, \bibinfo{person}{Patrick Esser}, {and} \bibinfo{person}{Björn Ommer}.} \bibinfo{year}{2021}\natexlab{}.
\newblock \bibinfo{title}{High-Resolution Image Synthesis with Latent Diffusion Models}.
\newblock
\showeprint[arxiv]{2112.10752}~[cs.CV]


\bibitem[Sharma et~al\mbox{.}(2024)]%
        {sharma_earthgen_2024}
\bibfield{author}{\bibinfo{person}{Ansh Sharma}, \bibinfo{person}{Albert Xiao}, \bibinfo{person}{Praneet Rathi}, \bibinfo{person}{Rohit Kundu}, \bibinfo{person}{Albert Zhai}, \bibinfo{person}{Yuan Shen}, {and} \bibinfo{person}{Shenlong Wang}.} \bibinfo{year}{2024}\natexlab{}.
\newblock \bibinfo{title}{{EarthGen}: {Generating} the {World} from {Top}-{Down} {Views}}.
\newblock
\href{https://doi.org/10.48550/arXiv.2409.01491}{doi:\nolinkurl{10.48550/arXiv.2409.01491}}
\newblock
\shownote{arXiv:2409.01491 [cs]}.


\bibitem[Sohl-Dickstein et~al\mbox{.}(2015)]%
        {pmlr-v37-sohl-dickstein15}
\bibfield{author}{\bibinfo{person}{Jascha Sohl-Dickstein}, \bibinfo{person}{Eric Weiss}, \bibinfo{person}{Niru Maheswaranathan}, {and} \bibinfo{person}{Surya Ganguli}.} \bibinfo{year}{2015}\natexlab{}.
\newblock \showarticletitle{Deep Unsupervised Learning using Nonequilibrium Thermodynamics}. In \bibinfo{booktitle}{\emph{Proceedings of the 32nd International Conference on Machine Learning}} \emph{(\bibinfo{series}{Proceedings of Machine Learning Research}, Vol.~\bibinfo{volume}{37})}, \bibfield{editor}{\bibinfo{person}{Francis Bach} {and} \bibinfo{person}{David Blei}} (Eds.). \bibinfo{publisher}{PMLR}, \bibinfo{address}{Lille, France}, \bibinfo{pages}{2256--2265}.
\newblock
\urldef\tempurl%
\url{https://proceedings.mlr.press/v37/sohl-dickstein15.html}
\showURL{%
\tempurl}


\bibitem[Song et~al\mbox{.}(2023)]%
        {song_consistency_2023}
\bibfield{author}{\bibinfo{person}{Yang Song}, \bibinfo{person}{Prafulla Dhariwal}, \bibinfo{person}{Mark Chen}, {and} \bibinfo{person}{Ilya Sutskever}.} \bibinfo{year}{2023}\natexlab{}.
\newblock \showarticletitle{Consistency models}. In \bibinfo{booktitle}{\emph{Proceedings of the 40th {International} {Conference} on {Machine} {Learning}}} \emph{(\bibinfo{series}{{ICML}'23})}. \bibinfo{publisher}{JMLR.org}, \bibinfo{address}{Honolulu, Hawaii, USA}.
\newblock


\bibitem[Spick and Walker(2019)]%
        {spick_realistic_2019}
\bibfield{author}{\bibinfo{person}{Ryan~Rs Spick} {and} \bibinfo{person}{James Walker}.} \bibinfo{year}{2019}\natexlab{}.
\newblock \showarticletitle{Realistic and {Textured} {Terrain} {Generation} using {GANs}}. In \bibinfo{booktitle}{\emph{European {Conference} on {Visual} {Media} {Production}}}. \bibinfo{publisher}{ACM}, \bibinfo{address}{London United Kingdom}, \bibinfo{pages}{1--10}.
\newblock
\showISBNx{978-1-4503-7003-5}
\href{https://doi.org/10.1145/3359998.3369407}{doi:\nolinkurl{10.1145/3359998.3369407}}


\bibitem[Voulgaris et~al\mbox{.}(2021)]%
        {voulgaris_procedural_2021}
\bibfield{author}{\bibinfo{person}{Georgios Voulgaris}, \bibinfo{person}{Ioannis Mademlis}, {and} \bibinfo{person}{Ioannis Pitas}.} \bibinfo{year}{2021}\natexlab{}.
\newblock \showarticletitle{Procedural {Terrain} {Generation} {Using} {Generative} {Adversarial} {Networks}}. In \bibinfo{booktitle}{\emph{2021 29th {European} {Signal} {Processing} {Conference} ({EUSIPCO})}}. \bibinfo{publisher}{IEEE}, \bibinfo{address}{Dublin, Ireland}, \bibinfo{pages}{686--690}.
\newblock
\showISBNx{978-90-827970-6-0}
\href{https://doi.org/10.23919/EUSIPCO54536.2021.9616151}{doi:\nolinkurl{10.23919/EUSIPCO54536.2021.9616151}}


\bibitem[Wu et~al\mbox{.}(2024)]%
        {wu_blockfusion_2024}
\bibfield{author}{\bibinfo{person}{Zhennan Wu}, \bibinfo{person}{Yang Li}, \bibinfo{person}{Han Yan}, \bibinfo{person}{Taizhang Shang}, \bibinfo{person}{Weixuan Sun}, \bibinfo{person}{Senbo Wang}, \bibinfo{person}{Ruikai Cui}, \bibinfo{person}{Weizhe Liu}, \bibinfo{person}{Hiroyuki Sato}, \bibinfo{person}{Hongdong Li}, {and} \bibinfo{person}{Pan Ji}.} \bibinfo{year}{2024}\natexlab{}.
\newblock \showarticletitle{{BlockFusion}: {Expandable} {3D} {Scene} {Generation} using {Latent} {Tri}-plane {Extrapolation}}.
\newblock \bibinfo{journal}{\emph{ACM Transactions on Graphics}} \bibinfo{volume}{43}, \bibinfo{number}{4} (\bibinfo{year}{2024}).
\newblock
\href{https://doi.org/10.1145/3658188}{doi:\nolinkurl{10.1145/3658188}}


\bibitem[Yamazaki et~al\mbox{.}(2017)]%
        {yamazaki_highaccuracy_2017}
\bibfield{author}{\bibinfo{person}{Dai Yamazaki}, \bibinfo{person}{Daiki Ikeshima}, \bibinfo{person}{Ryunosuke Tawatari}, \bibinfo{person}{Tomohiro Yamaguchi}, \bibinfo{person}{Fiachra O'Loughlin}, \bibinfo{person}{Jeffery~C. Neal}, \bibinfo{person}{Christopher~C. Sampson}, \bibinfo{person}{Shinjiro Kanae}, {and} \bibinfo{person}{Paul~D. Bates}.} \bibinfo{year}{2017}\natexlab{}.
\newblock \showarticletitle{A high‐accuracy map of global terrain elevations}.
\newblock \bibinfo{journal}{\emph{Geophysical Research Letters}} \bibinfo{volume}{44}, \bibinfo{number}{11} (\bibinfo{date}{June} \bibinfo{year}{2017}), \bibinfo{pages}{5844--5853}.
\newblock
\showISSN{0094-8276, 1944-8007}
\href{https://doi.org/10.1002/2017GL072874}{doi:\nolinkurl{10.1002/2017GL072874}}


\bibitem[Zhang et~al\mbox{.}(2025)]%
        {11209478}
\bibfield{author}{\bibinfo{person}{Xiaoyu Zhang}, \bibinfo{person}{Teng Zhou}, \bibinfo{person}{Xinlong Zhang}, \bibinfo{person}{Jia Wei}, {and} \bibinfo{person}{Yongchuan Tang}.} \bibinfo{year}{2025}\natexlab{}.
\newblock \showarticletitle{{ Multi-Scale Diffusion: Enhancing Spatial Layout in High-Resolution Panoramic Image Generation }}. In \bibinfo{booktitle}{\emph{2025 IEEE International Conference on Multimedia and Expo (ICME)}}. \bibinfo{publisher}{IEEE Computer Society}, \bibinfo{address}{Los Alamitos, CA, USA}, \bibinfo{pages}{1--6}.
\newblock
\href{https://doi.org/10.1109/ICME59968.2025.11209478}{doi:\nolinkurl{10.1109/ICME59968.2025.11209478}}


\bibitem[Zhou and Tang(2024)]%
        {Zhou_Tang_2024}
\bibfield{author}{\bibinfo{person}{Teng Zhou} {and} \bibinfo{person}{Yongchuan Tang}.} \bibinfo{year}{2024}\natexlab{}.
\newblock \showarticletitle{TwinDiffusion: Enhancing Coherence and Efficiency in Panoramic Image Generation with Diffusion Models}.
\newblock  \bibinfo{number}{arXiv:2404.19475} (\bibinfo{year}{2024}).
\newblock
\href{https://doi.org/10.48550/arXiv.2404.19475}{doi:\nolinkurl{10.48550/arXiv.2404.19475}}
\newblock
\shownote{arXiv:2404.19475 [cs]}.


\end{thebibliography}

\newpage
\begin{figure*}[t]
\centering
\includegraphics[height=0.93\textheight]{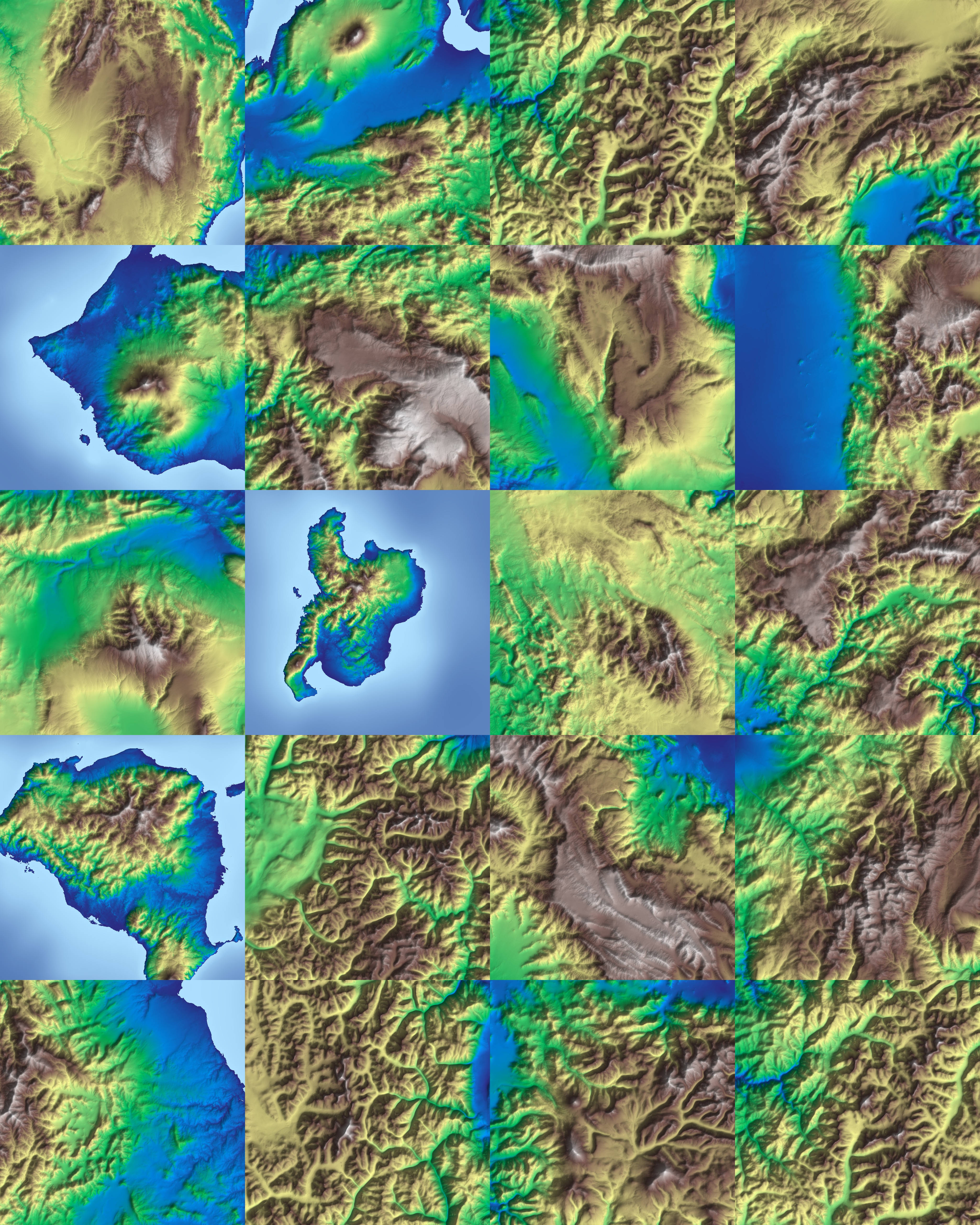}
\caption{Twenty generated 1024 by 1024 regions from Terrain Diffusion. Samples cover volcanic islands, high relief mountain systems, and dissected plateaus, illustrating the model’s ability to reproduce diverse landscapes with coherent multi-scale structure. All emerge from one world generated with the same seed as in Figure ~\ref{fig:teaser}. Zoom for details.}
\label{fig:terrain_grid}
\end{figure*}
\clearpage

\begin{figure*}[t]
\centering
\includegraphics[width=0.9\textwidth]{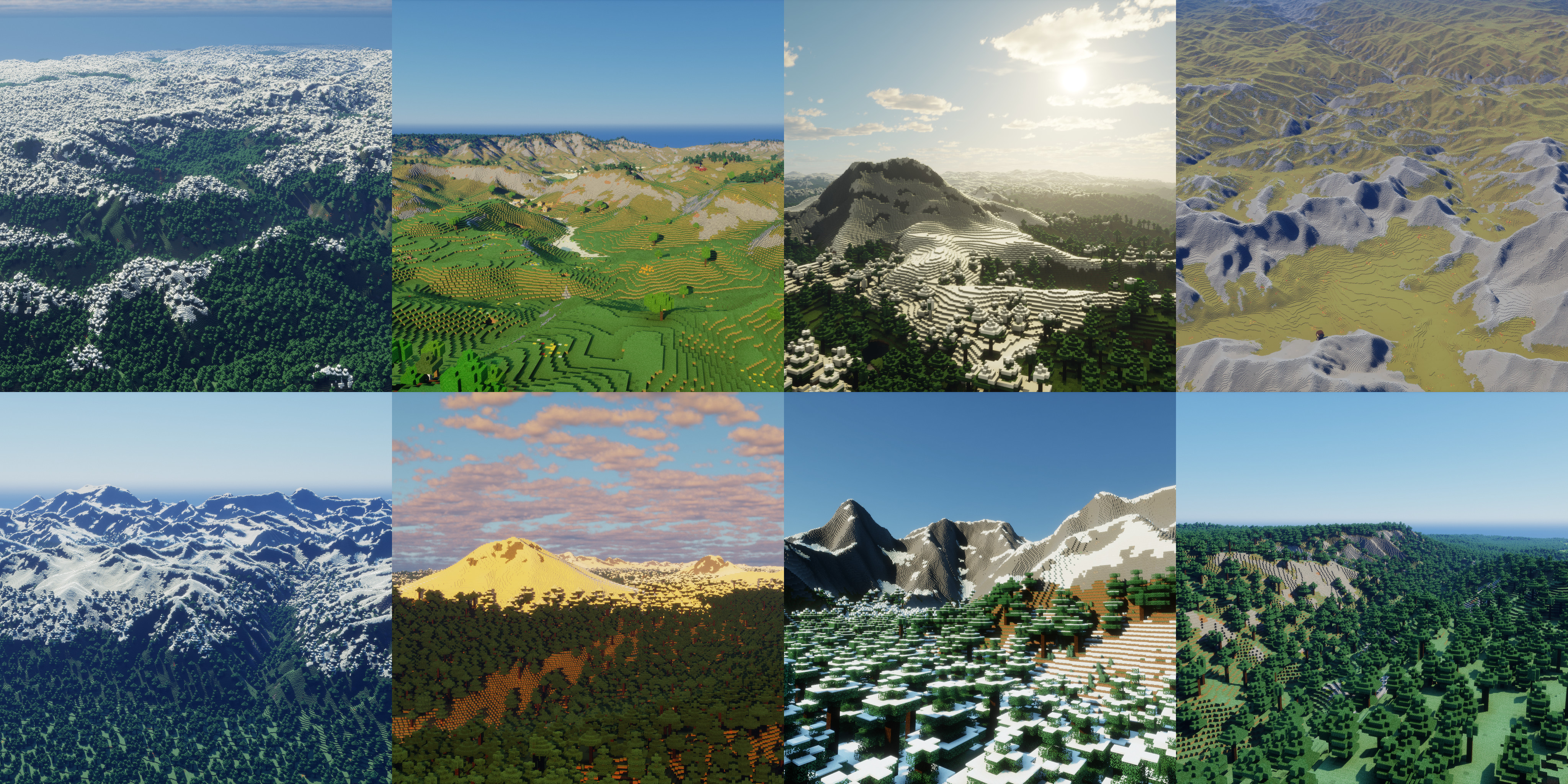}
\caption{Eight Minecraft scenes generated from Terrain Diffusion using a single fixed biome mapping derived from the model’s climatic outputs. The Distant Horizons mod is used to increase render distance, and Bliss shaders are used to enhance visuals.}
\label{fig:minecraft_grid}

\vspace{10pt}

\begin{minipage}[t]{\columnwidth}
\centering
\small{"A photo of mountain range at twilight"}
\includegraphics[width=0.95\textwidth]{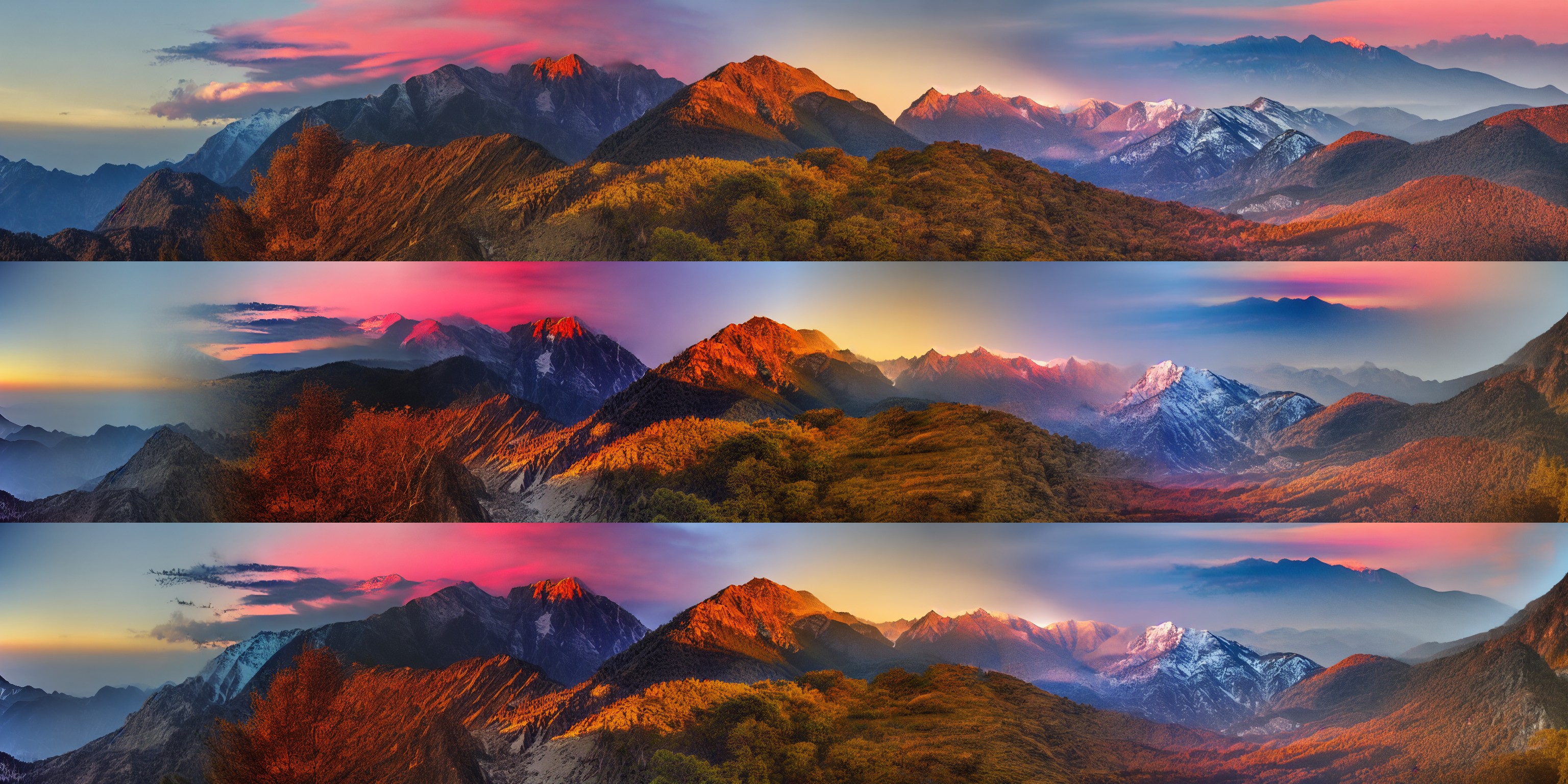}
\end{minipage}
\hfill
\begin{minipage}[t]{\columnwidth}
\centering
\small{"A photo of the dolomites"}
\includegraphics[width=0.95\textwidth]{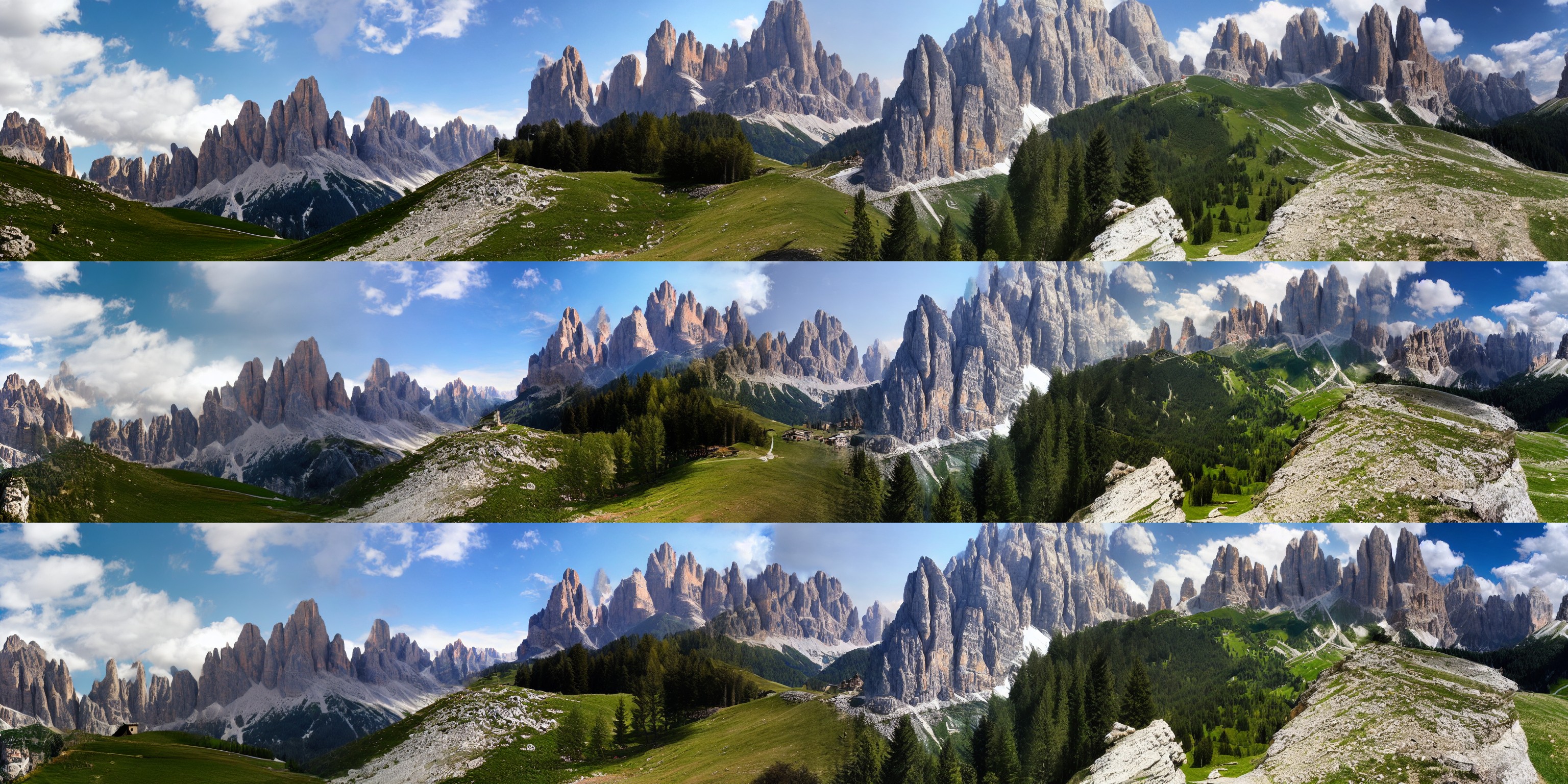}
\end{minipage}

\vspace{5pt}

\begin{minipage}[t]{\columnwidth}
\centering
\small{"Natural landscape in anime style illustration"}
\includegraphics[width=0.95\textwidth]{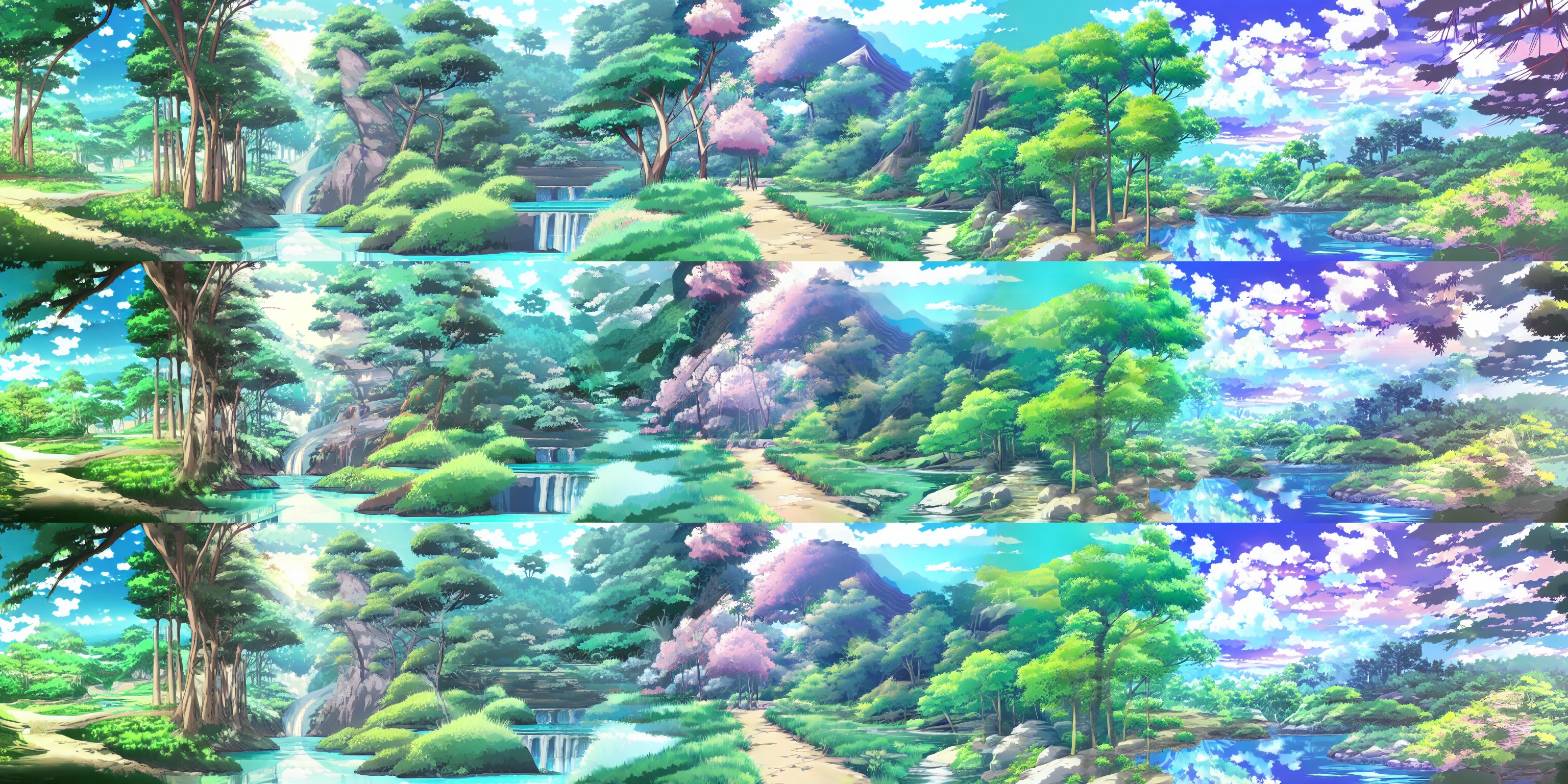}
\end{minipage}
\hfill
\begin{minipage}[t]{\columnwidth}
\centering
\small{"A photo of a snowy mountain peak with skiers"}
\includegraphics[width=0.95\textwidth]{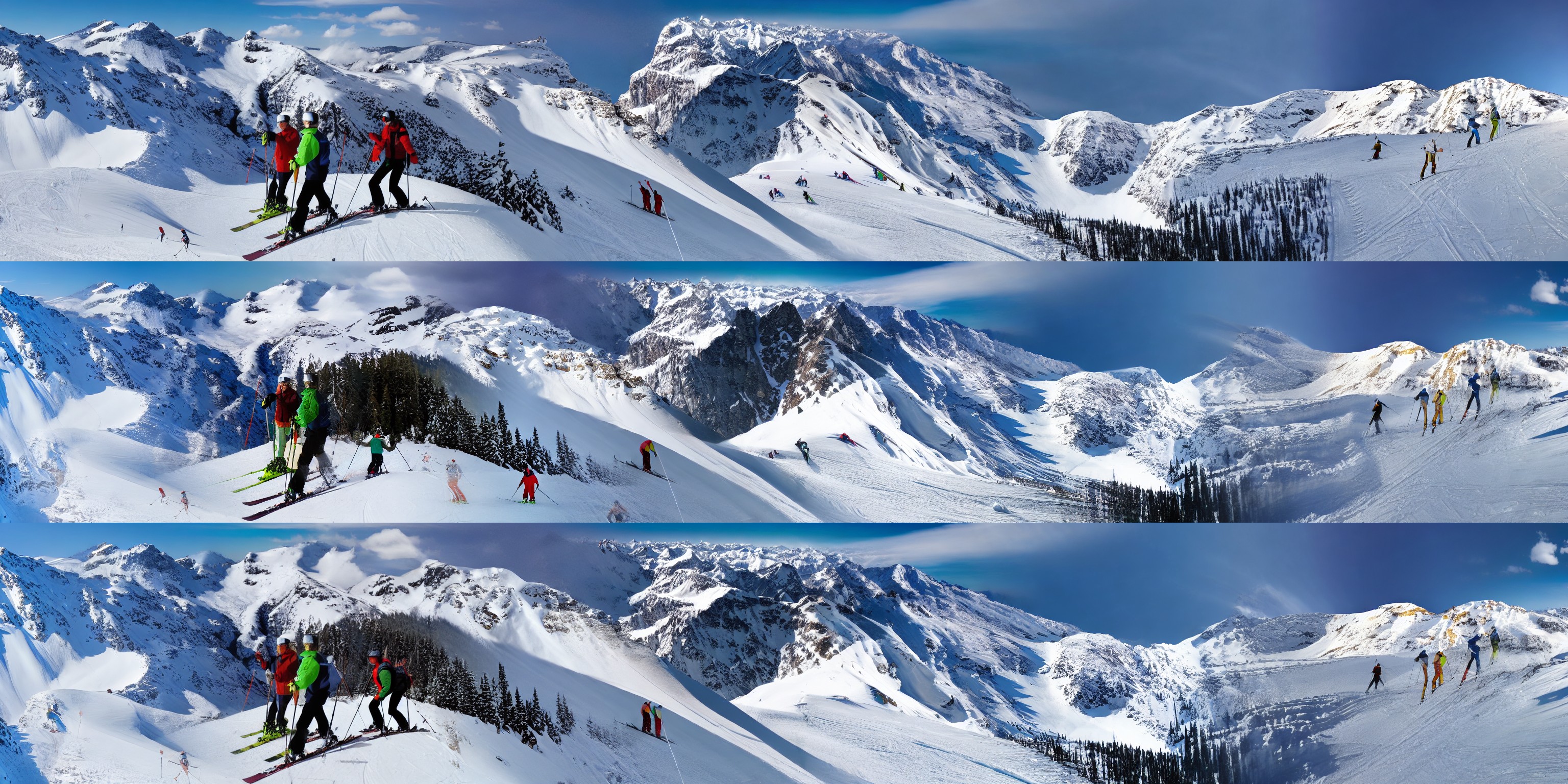}
\end{minipage}

\caption{A comparison of MultiDiffusion (top), InfiniteDiffusion ($T=1$, middle), and InfiniteDiffusion ($T = 2$, bottom). InfiniteDiffusion with $T=1$ retains the overall structure of MultiDiffusion, but introduces considerable artifacts. InfiniteDiffusion with $T=2$ eliminates many of these artifacts partially or fully, though some still remain. A few additional blending steps generally resolves this (Appendix G). Note that differences at the boundary are also caused by the infinite image implicitly continuing beyond the boundary, whereas MultiDiffusion stops abruptly.}

\label{fig:infinite_diffusion_comparison}
\end{figure*}
\clearpage

\appendix

\begin{figure*}[h]
\centering
\includegraphics[height=0.93\textheight]{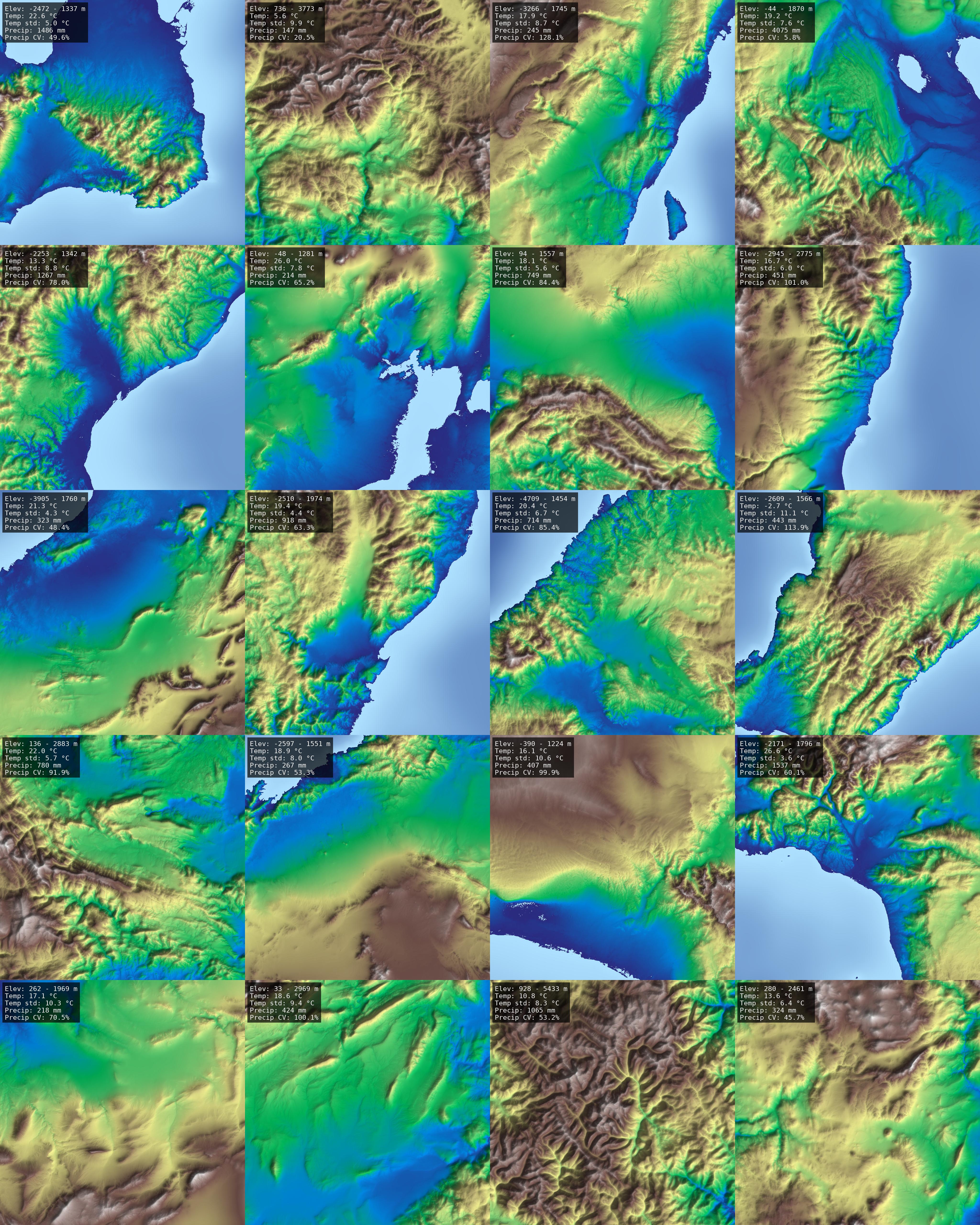}
\caption{Twenty additional 1024 by 1024 regions from Terrain Diffusion. Samples are uncurated, except for filtering regions with more than 50\% ocean. We include additional details on overall elevation range, and climate variables (average over region), in the top left of each sample.}
\label{fig:additional_samples}
\end{figure*}
\clearpage

\begin{figure*}[t]
    \centering
    \begin{subfigure}{0.37\linewidth}
        \centering
        \includegraphics[width=\linewidth]{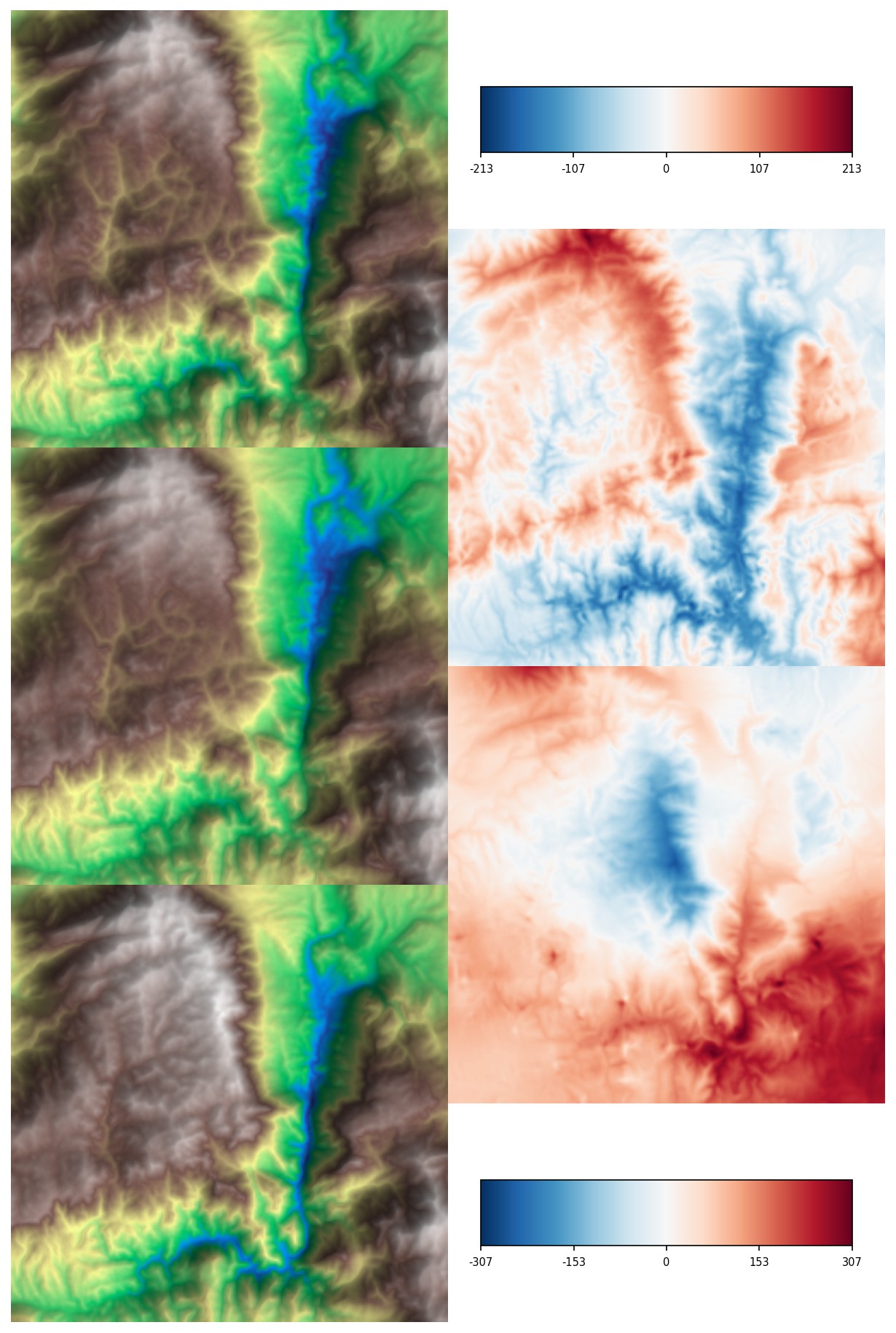}
        \label{fig:comp-a}
    \end{subfigure}
    \begin{subfigure}{0.37\linewidth}
        \centering
        \includegraphics[width=\linewidth]{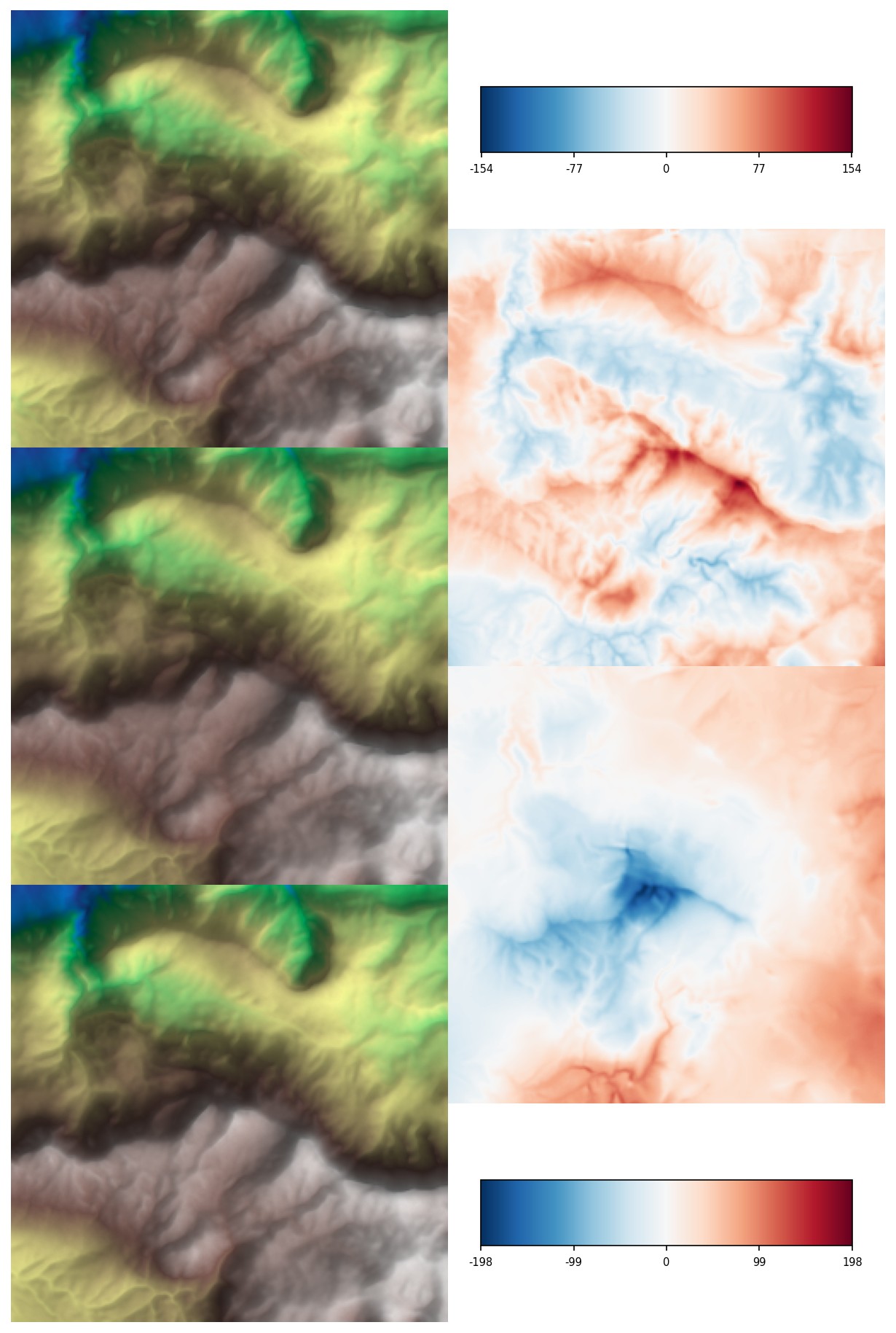}
        \label{fig:comp-b}
    \end{subfigure}
    \begin{subfigure}{0.37\linewidth}
        \centering
        \includegraphics[width=\linewidth]{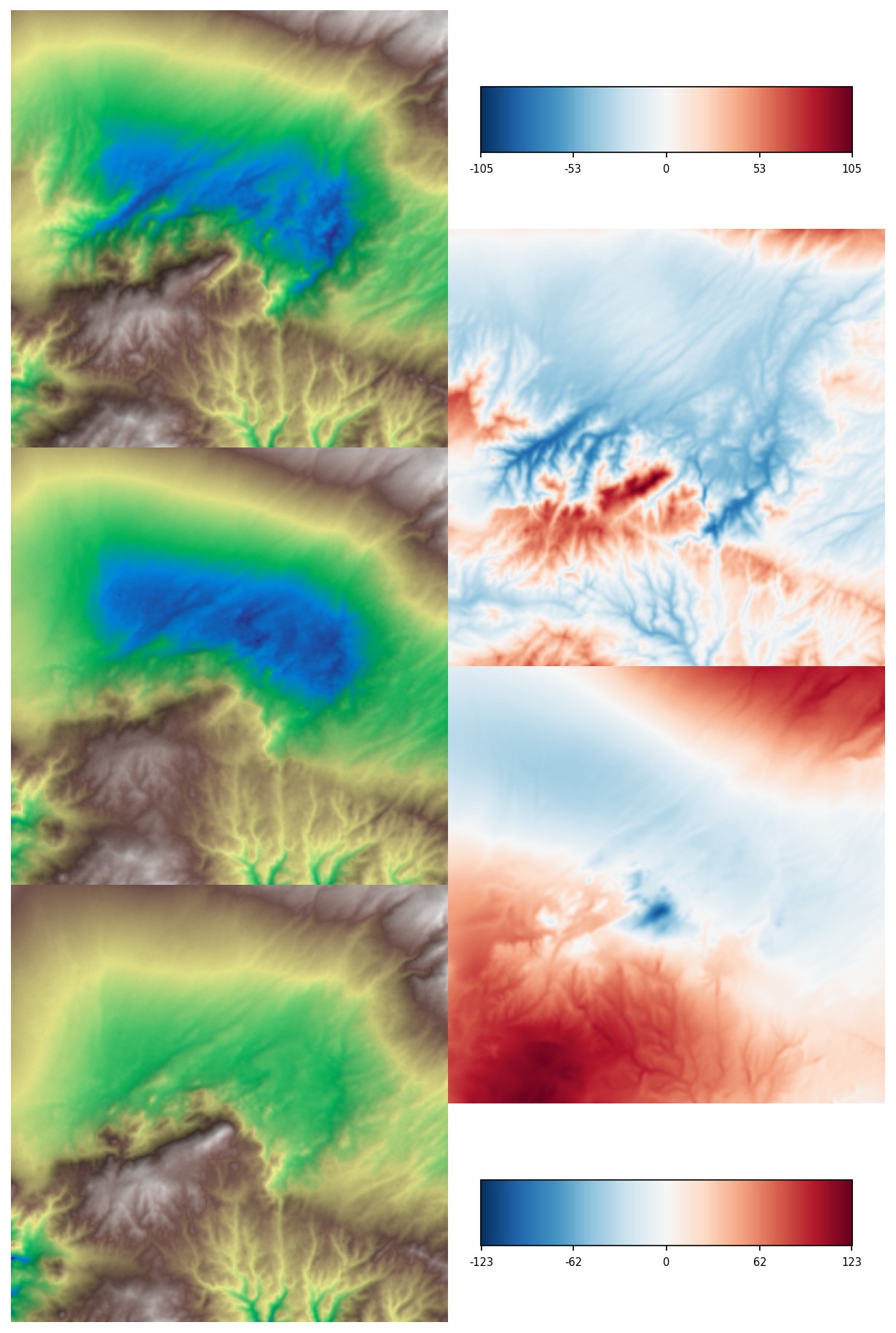}
        \label{fig:comp-c}
    \end{subfigure}
    \begin{subfigure}{0.37\linewidth}
        \centering
        \includegraphics[width=\linewidth]{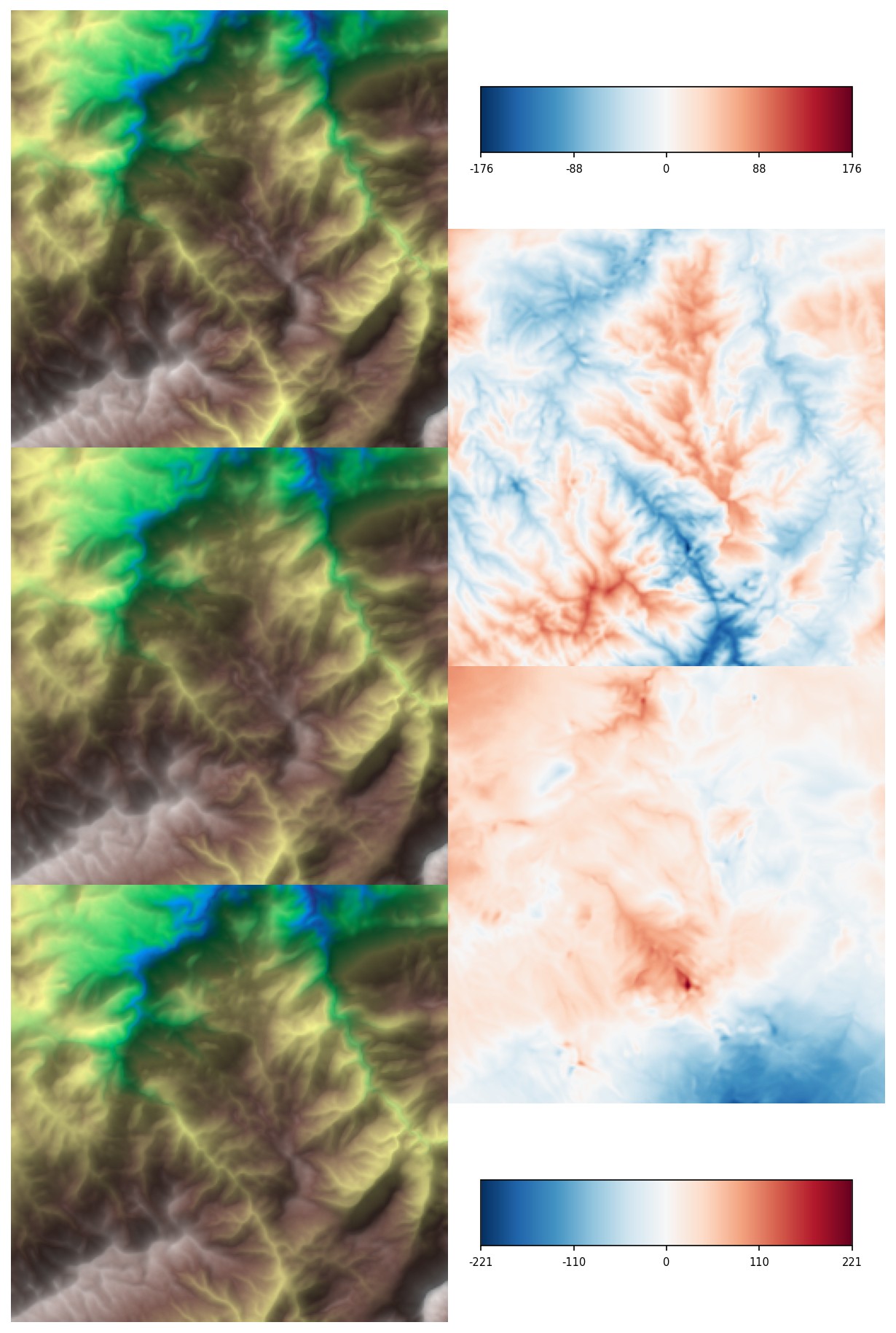}
        \label{fig:comp-d}
    \end{subfigure}
    \caption{A comparison of how different numbers of blending steps ($T$) affect output quality. Top: T=0. Middle: T=1. Bottom: T=2. The heatmaps indicate the change in heights moving from one image below to the one above. Moving from T=2 to T=1 we lose some sharpness, witness hydrology fail more often, and generally lose some finer details. Moving from T=1 to T=0 (Naive InfiniteDiffusion), we see the terrain surface generally exaggerating all features slightly.}
    \label{fig:extended-comparison}
\end{figure*}
\clearpage

\raggedbottom
\setlength{\parindent}{0pt}
\setlength{\parskip}{6pt}
\setlength{\abovedisplayskip}{8pt}
\setlength{\belowdisplayskip}{8pt}
\setlength{\abovedisplayshortskip}{8pt}
\setlength{\belowdisplayshortskip}{8pt}

\begin{figure*}[t]
\centering

\includegraphics[width=\textwidth]{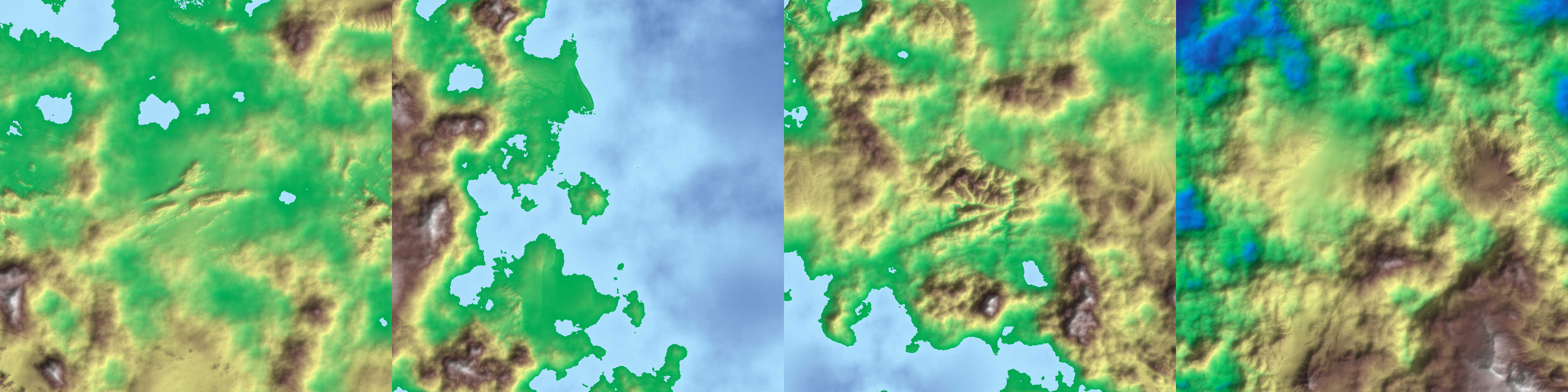}
\caption{Representative samples generated by our implementation of Perlin Blending.}
\label{fig:perlin_blending_samples}

\vspace{1em} 

\captionof{table}{Performance metrics for region generation.}
\label{tab:performance_appendix}
\begin{tabular}{ccccccc}
\toprule
T & Region Size& TTFT (s) $\downarrow$ & TTST (s) $\downarrow$ & Peak VRAM (MB) $\downarrow$& TTFT (p5, p50, p95) $\downarrow$ & TTST (p5, p50, p95) $\downarrow$ \\
\midrule
$T=2$&$1024\times1024$& 2.96s ± 0.35s& 1.79s ± 0.27s& 2214 MB& (2.59, 2.78, 3.74)& (1.40, 1.78, 2.30)\\
$T=2$&$512\times512$& 1.72s ± 0.19s& 0.66s ± 0.18s& 2214 MB& (1.34, 1.76, 2.00)& (0.42, 0.62, 0.96)\\
$T=2$&$256 \times 256$& 1.25s ± 0.19s& 0.26s ± 0.18s& 2214 MB& (1.03, 1.22, 1.54)& (0.00, 0.29, 0.53)\\
$T=2$&$128 \times 128$& 1.06s ± 0.14s& 0.10s ± 0.15s& 2214 MB& (0.72, 1.09, 1.22)&(0.00, 0.00, 0.40)\\
\midrule
$T=1$&$1024\times1024$& 2.58s ± 0.37s& 1.75s ± 0.30s& 2214 MB& (2.18, 2.58, 3.32)& (1.38, 1.71, 2.36)\\
$T=1$&$512\times512$& 1.39s ± 0.20s& 0.63s ± 0.16s& 2214 MB& (0.99, 1.44, 1.72)& (0.39, 0.60, 0.88)\\
$T=1$&$256\times256$& 0.95s ± 0.15s& 0.25s ± 0.17s& 2214 MB& (0.97, 1.10, 1.70)& (0.00, 0.34, 0.56)\\
$T=1$&$128\times128$& 0.79s ± 0.14s& 0.11s ± 0.16s& 2214 MB& (0.45, 0.83, 0.90)& (0.00, 0.00, 0.38)\\
\bottomrule
\end{tabular}
\end{figure*}

\section{Extended Results And Methodology}

\subsection{Additional Qualitative Samples}

In Figure ~\ref{fig:additional_samples}, we showcase twenty additional 1024 by 1024 regions from Terrain Diffusion. Samples are uncurated, except for automatically excluding regions with more than 50\% ocean in the coarse map. We include additional details on overall elevation range, and climate variables, in the top left of each sample.

\subsection{Additional Details on FID Calculations}
\label{appendix:fid-details}

FID calculations use raw elevation values, with images normalized by centering each tile and scaling by the larger of its value range or 255, ensuring that images are not expanded beyond the native precision of the data. In Algorithm ~\ref{alg:normalize_uint8}, we provide pseudocode for the normalization scheme used to convert arbitrary elevation values to the 0-255 range required for FID calculation.

For Perlin blending \cite{10.1145/3571600.3571657}, our implementation differs from the reference method, which targets a zero-centered dataset. To ensure a fair comparison, we adapt the algorithm to our data: rather than applying Perlin noise with a static distribution, we employ an adaptive distribution centered on the mean and scaled by the standard deviation of surrounding elevation tiles (weighted by linear interpolation). Figure ~\ref{fig:perlin_blending_samples} provides representative samples.

\subsection{Decoder FID}
\label{appendix:decoder_details}

We measure the decoder’s standalone rFID at $512 \times 512$ resolution. rFID measures FID between original images and their reconstructions. The one step variant obtains an rFID of $2.83$, while the two step variant reaches $1.07$. Tiling only increases the one-step FID to $2.99$.

\subsection{Additional Performance Metrics}

In addition to the latency metrics provided in the main text, we also report peak VRAM usage and additional latency metrics for varying region sizes. Our results are in Table ~\ref{tab:performance_appendix}. For all performance calculations, we initially compile the models with \texttt{torch.compile()}, and use full-precision inference (\texttt{fp32}), since it provided the best performance. VRAM usage reduces to 1846 MB with half-precision floats. We measure latency end-to-end as observed through the same API exposed to applications, including all system overhead.

\subsection{Extended Visual Comparison}
Figure ~\ref{fig:extended-comparison} visualizes terrain outputs across decreasing noise levels (T=2 to T=1 to T=0). While differences appear subtle in 2D visualization, each stage shift corresponds to elevation changes of $\pm100$m geologically. While all methods produce plausible terrain with no obvious artifacts, which we attribute to our hierarchical architecture, potentially significant differences emerge across timesteps. 

\begin{algorithm}[H]
\caption{Normalizing heightmaps for FID}
\label{alg:normalize_uint8}

\begin{algorithmic}

\State \textbf{Inputs:}

\begin{tabbing}
\hspace{0em} \= \hspace{2.5em} \= \hspace{1em} \= \kill
\> $I$ \> $\vartriangleright$ \> batch of single-channel images ($B \times 1 \times H \times W$) \\
\end{tabbing}

\State $I_{\min} \gets \min(I)$ \Comment{Per-image minimum}
\State $I_{\max} \gets \max(I)$ \Comment{Per-image maximum}
\State $I_{\text{range}} \gets \max(I_{\max} - I_{\min}, 255)$ \Comment{Ensure scaling factor $\ge 255$}
\State $I_{\text{mid}} \gets (I_{\min} + I_{\max}) / 2$

\State $I_{\text{norm}} \gets \text{clamp}\left( \left( \frac{I - I_{\text{mid}}}{I_{\text{range}}} + 0.5 \right) \times 255, 0, 255 \right)$
\State $O \gets \text{repeat}(I_{\text{norm}}, \text{channels}=3)$

\\
\State \textbf{Output:} $O$ cast to uint8
\end{algorithmic}
\end{algorithm}

\section{Formal Properties of InfiniteDiffusion}

\label{appendix:infinitediffusion-properties}

For reference, we include the InfiniteDiffusion algorithm from the main text.

\begin{algorithm}[H]
\caption{Querying $J_t[R]$ ($t$ < $T$) with InfiniteDiffusion.}
\label{alg:infinite_diffusion_step_appdx}

\begin{algorithmic}

\State \textbf{Inputs:}

\begin{tabbing}
\hspace{0em} \= \hspace{2.5em} \= \hspace{1em} \= \kill
\> $\Phi$ \> $\vartriangleright$ \> pretrained diffusion model\\
\> $J_{t+1}$               \> $\vartriangleright$ \> infinite noisy input image \\
\> $A_t$           \> $\vartriangleright$ \> infinite accumulated output image \\
\> $B_t$           \> $\vartriangleright$ \> infinite accumulated weights for $A_t$ \\
\> $R$                 \> $\vartriangleright$ \> region to query \\
\> $P_t$                 \> $\vartriangleright$ \> set of processed windows \\
\end{tabbing}

\For{each window $i$ in $\kappa(R) \setminus P$}
    \State $A_t[R_i] \gets A_t[R_i] + W_i \otimes \Phi(J_{t+1}[R_i] \mid y_i)$
    \State $B_t[R_i] \gets B_t[R_i] + W_i$
\EndFor
\State $P_t \gets P_t \cup \kappa(R)$

\\
\State \textbf{Output:} $J_t[R] = A_t[R] / B_t[R]$

\end{algorithmic}
\end{algorithm}

\subsection{Preliminaries and Notation}

We first make precise the InfiniteDiffusion framework used in the main text.

\paragraph{Infinite image space.}
Let the spatial domain be the integer lattice $\mathbb{Z}^d$, and let
$$\mathcal{J} \coloneqq \mathbb{R}^{\mathbb{Z}^d \times C}$$
denote the space of infinite images with $C$ channels.  We write $J \in \mathcal{J}$ as a function $J : \mathbb{Z}^d \to \mathbb{R}^C$ over pixel coordinates. For the following sections, we write $\Omega$ as shorthand for the index-set $[H_1] \times \cdots \times [H_d]$.

\paragraph{Regions and windows.}
A (rectangular) region is any subset of the form
$$
R = [a_1, b_1) \times \cdots \times [a_d, b_d) \subset \mathbb{Z}^d
$$
with integers $a_k < b_k$.  The set of window indices is a countable set $S$ (e.g.\ $S = \mathbb{Z}^d$).  For each $i \in S$ we are given a window region $R_i \subset \mathbb{Z}^d$, a weight map $W_i \in \mathbb{R}^{\Omega}$, and a conditioning vector $y_i$. We assume the window layout is such that for every finite region $R$, only finitely many windows intersect $R$:

\textbf{Assumption 1 (Finite window overlap).}
For every finite region $R$,
$$
\kappa(R) \coloneqq \{ i \in S \;:\; R_i \cap R \neq \emptyset \}
$$
is finite.

This assumption is true in the sliding window case. We include the algorithm for computing $\kappa$ in the sliding window case in Algorithm ~\ref{alg:pixel_range_to_window_range} below.

\begin{algorithm}[H]
\caption{Finding all windows intersecting a pixel region}
\label{alg:pixel_range_to_window_range}

\begin{algorithmic}

\State \textbf{Inputs:}

\begin{tabbing}
\hspace{0em} \= \hspace{2.5em} \= \hspace{1em} \= \kill
\> $R$ \> $\vartriangleright$ \> pixel region per dimension $d$: $R_d = [R_d^{\text{start}},\, R_d^{\text{stop}})$ \\
\> $s$ \> $\vartriangleright$ \> window size per dimension \\
\> $\sigma$ \> $\vartriangleright$ \> window stride per dimension \\
\> $o$ \> $\vartriangleright$ \> window offset per dimension \\
\end{tabbing}

\State \textit{// Window $i$ covers pixels $[i_d \cdot \sigma_d + o_d,\; i_d \cdot \sigma_d + o_d + s_d)$}
\State \textit{// in each dimension $d$}

\For{each dimension $d$}

    \State $n_d \gets R_d^{\text{start}} - o_d - s_d + 1$
    \State $i_d^{\min} \gets \lceil n_d \mathbin{/} \sigma_d \rceil$

    \State $i_d^{\max} \gets \lfloor (R_d^{\text{stop}} - 1 - o_d) \mathbin{/} \sigma_d \rfloor$

\EndFor

\If{$\exists\, d$ such that $i_d^{\min} > i_d^{\max}$}
    \State \Return $\varnothing$ \Comment{Region is in-between windows}
\EndIf

\State \Return $\bigtimes_d \; [i_d^{\min},\, i_d^{\max}]$

\end{algorithmic}
\end{algorithm}

\paragraph{Embedding operator.}
For any tensor $x \in \mathbb{R}^{\Omega \times C}$, the operator $U_i(x) \in \mathcal{J}$ denotes the infinite image obtained by placing $x$ on $R_i$ and zero elsewhere.

\paragraph{Pretrained diffusion model.}
Let $\Phi$ denote a fixed pretrained denoising network acting on window-sized images.  We formalize it as a deterministic function
$$
\Phi : \mathbb{R}^{\Omega \times C} \times \mathcal{Y} \to \mathbb{R}^{\Omega \times C},
$$
where $\mathcal{Y}$ is the space of conditioning vectors.

\paragraph{InfiniteDiffusion update.}
For a given noisy image $J_t \in \mathcal{J}$, the InfiniteDiffusion update at step $t \to t-1$ is defined, for any finite region $R$, by
\begin{equation}
\Psi(J_t \mid z)[R]
=
\left(
\frac{
\sum_{i \in \kappa(R)}
    U_i\bigl(W_i \otimes \Phi(J_t[R_i] \mid y_i)\bigr)
}{
    \sum_{j \in \kappa(R)} U_j(W_j)
}
\right)[R],
\label{eq:infinitediff-formal}
\end{equation}
where $\otimes$ denotes the Hadamard (elementwise) product and the division is also elementwise. For this update, we adopt the convention that any division by zero is defined to be zero. By Assumption~1, both sums are finite on any finite $R$, so~\eqref{eq:infinitediff-formal} is well-defined. $z$ is a (possibly infinite) vector from which the $y_i$ are computed.

\paragraph{Seeds and initial noise.}
Let $\mathcal{X}$ be a set of seeds.  A seed $s \in \mathcal{X}$ deterministically selects an initial infinite noise image $J_T^{(s)} \in \mathcal{J}$, and a conditioning vector $z^{(s)} \in \mathcal{Z}$.
Formally, there are deterministic functions
$$
G : \mathcal{X} \times \mathbb{Z}^d \to \mathbb{R}^C, \qquad
H : \mathcal{X} \to \mathcal{Z}, \qquad
\Lambda: \mathcal{Z} \times S \to \mathcal{Y}
$$
such that
$$
J_T^{(s)}(p) = G(s,p), \qquad z^{(s)} = H(s), \qquad y_i^{(s)} = \Lambda(z^{(s)}, i).
$$

\paragraph{Definition of $J_t^{(s)}$.}
For a fixed seed $s$, we define the entire trajectory $(J_t^{(s)})_{t=0}^T$ recursively by:
\begin{itemize}
\item $J_T^{(s)}$ is given by $G$.
\item For $t = T, T-1, \dots, 1$, define $J_t^{(s)}$ via  \eqref{eq:infinitediff-formal} with $J_t = J_t^{(s)}$ and  $y_i = y_i^{(s)}$.
\end{itemize}
Because $T$ is finite and each update uses only finite sums on any finite region, the infinite images $J_t^{(s)}$ are well-defined for all $t \in \{0,\dots,T\}$.

\paragraph{Lazy evaluation algorithm.}
Algorithm~\ref{alg:infinite_diffusion_step_appdx} efficiently computes~\eqref{eq:infinitediff-formal}. At level $t < T$ it maintains infinite images $A_t, B_t \in \mathcal{J}$ and a processed set $P_t \subseteq S$. To answer a query $J_t[R]$, it performs:

\begin{quote}
\emph{For each window $i \in \kappa(R) \setminus P_t$:}
$$
\begin{aligned}
A_t[R_i] &\gets A_t[R_i] + W_i \otimes \Phi(J_{t+1}[R_i], y_i),\\
B_t[R_i] &\gets B_t[R_i] + W_i.
\end{aligned}
$$
Add each such $i$ to $P_t$. The result is
$$
J_t[R] = A_t[R] / B_t[R].
$$
\end{quote}

Where the division is performed elementwise. New windows are evaluated recursively by querying $J_{t+1}[\cdot]$ in the same way at level $t+1$, until reaching $J_T^{(s)}$, which is given by the seed. At intialization, $A_t = \textbf{0}$, $B_t = \textbf{0}$, and $P_t = \emptyset$.

\medskip
We now formalize and prove the three properties stated in the main text.

\subsection{Seed Consistency}

Informally, seed consistency says that once the seed is fixed, every
finite region $J_t[R]$ is a deterministic function of the seed and
the region alone, and is independent of the order in which regions are
queried.

\begin{definition}[Seed-consistent generative process]
A family of random infinite images $\{J_t\}_{t=0}^T$ on $\mathbb{Z}^d$ is \emph{seed-consistent} if there exists a set of seeds $\mathcal{X}$ such that for all $s \in \mathcal{X}$, $t \in \{0,\dots,T\}$ and finite $R$,
$$J_t[R]$$ is a function of $s$ and $R$. Consequently, repeated queries for the same $t,s,R$ always return the same value, irrespective of the order in which regions are requested.
\end{definition}

We show that InfiniteDiffusion, as defined in ~\ref{eq:infinitediff-formal}, is seed-consistent.

\begin{lemma}[InfiniteDiffusion is seed-consistent]
\label{lem:deterministic-infinitediffusion}
Fix a seed $s \in \mathcal{X}$.  Then for each $t \in\{0,\dots,T\}$ and finite region $R$, the infinite image $J_t^{(s)}[R]$ defined by the recursive update~\eqref{eq:infinitediff-formal} is uniquely determined by $s$ and $R$.
\end{lemma}

\begin{proof}
We proceed by backward induction on $t$.

\emph{Base case ($t=T$).}  By construction, $J_T^{(s)}(p) = G(s,p)$ for all $p$, so for any finite region $R$, the restriction $J_T^{(s)}[R]$ is uniquely determined by $s$ and $R$.

\emph{Inductive step.}  Assume that for some $t \in \{1,\dots,T\}$, $J_t^{(s)}[R]$ is uniquely determined by $s$ and $R$ for all finite regions $R$.  Consider $J_{t-1}^{(s)}[R]$ for a finite region $R$.

By definition~\eqref{eq:infinitediff-formal},
$$
J_{t-1}^{(s)}[R]
=
\left(
\frac{
\sum_{i \in \kappa(R)}
    U_i\bigl(W_i \otimes \Phi(J_t^{(s)}[R_i] \mid y_i^{(s)})\bigr)
}{
    \sum_{j \in \kappa(R)} U_j(W_j)
}
\right)[R]
$$
For each $i \in \kappa(R)$, the region $R_i$ is finite, so by the inductive hypothesis $J_t^{(s)}[R_i]$ is uniquely determined by $s$ and $R_i$.  The model $\Phi$ and weight maps $W_i$ are deterministic. Therefore each term
$$
U_i\bigl(W_i \otimes \Phi(J_t^{(s)}[R_i] \mid y_i^{(s)})\bigr)
$$
is uniquely determined by $s$, and hence the finite sums in the numerator and denominator are uniquely determined.  Thus $J_{t-1}^{(s)}[R]$ is uniquely determined by $s$ and $R$.

By induction, the claim holds for all $t$.
\end{proof}

Lemma~\ref{lem:deterministic-infinitediffusion} shows that InfiniteDiffusion has a well-defined deterministic output for a fixed seed.  We now show that the lazy query algorithm is consistent with this definition, and is therefore also seed-consistent.

\begin{lemma}[Algorithm Consistency]
\label{lem:algorithm-consistency}
Fix $s \in \mathcal{X}$ and a timestep $t \in \{0,\dots,T-1\}$.  Then immediately before any query $J_t[R]$ with Algorithm~\ref{alg:infinite_diffusion_step_appdx}, the pair $(A_t, B_t)$ satisfy
$$
\begin{aligned}
A_t
&=
\sum_{i \in P_t} U_i\bigl(
  W_i \otimes \Phi(J_{t+1}^{(s)}[R_i] \mid y_i^{(s)})\bigr),\\
B_t
&=
\sum_{i \in P_t} U_i(W_i).
\end{aligned}
$$
Furthermore, after performing one iteration of Algorithm ~\ref{alg:infinite_diffusion_step_appdx}, the updated state satisfies the same form with $P_t$ replaced by $P_t \cup \kappa(R)$.
\end{lemma}

\begin{proof}
We prove this by induction over an arbitrary sequence of queries. At initialization, $A_t = 0$, $B_t = 0$, and $P_t = \emptyset$, so the claim is trivially true.

Now assume that before any query $J_t[R]$ we have
$$
A_t = \sum_{i \in P_t} U_i(V_i), \quad
B_t = \sum_{i \in P_t} U_i(W_i),
$$
where we write $V_i$ as shorthand for $W_i \otimes \Phi(J_{t+1}^{(s)}[R_i] \mid y_i^{(s)})$.

During the query, for each $i \in \kappa(R) \setminus P_t$ the
algorithm performs
$$
A_t[R_i] \gets A_t[R_i] + V_i,\qquad
B_t[R_i] \gets B_t[R_i] + W_i
$$
and does not modify any pixels outside $R_i$.  Equivalently, this is
$$
A_t \gets A_t + U_i(V_i), \qquad
B_t \gets B_t + U_i(W_i).
$$
After processing all such windows we obtain
$$
A_t
=
\sum_{i \in P_t} U_i(V_i)
+
\sum_{i \in \kappa(R) \setminus P_t} U_i(V_i)
=
\sum_{i \in P_t \cup \kappa(R)} U_i(V_i),
$$
and similarly
$$
B_t
=
\sum_{i \in P_t \cup \kappa(R)} U_i(W_i).
$$
Finally, the algorithm updates $P_t \gets P_t \cup \kappa(R)$, so  the updated state satisfies the same form with $P_t$ replaced by $P_t \cup \kappa(R)$.

By induction, the claim holds for all queries $J_t[R]$.
\end{proof}

We now show that Algorithm ~\ref{alg:infinite_diffusion_step_appdx} is consistent with the formal definition in ~\ref{eq:infinitediff-formal}.

\begin{lemma}[Correctness of a single query]
\label{lem:single-query-correctness}
Fix $s$, $t$, and a finite region $R$.  After any query $J_t[R]$ following Algorithm ~\ref{alg:infinite_diffusion_step_appdx}, we have
$$
J_t[R] = A_t[R] / B_t[R]
=
J_t^{(s)}[R].
$$
\end{lemma}

\begin{proof}
By Lemma~\ref{lem:algorithm-consistency}, after the query finishes we have
$$
A_t
=
\sum_{i \in P'_t} U_i(V_i), \qquad
B_t
=
\sum_{i \in P'_t} U_i(W_i),
$$
for some processed set $P'_t \supseteq \kappa(R)$.

Now restrict to the region $R$.  Any window $i \notin \kappa(R)$ has $R_i \cap R = \emptyset$, so
$$
U_i(V_i)[R] = 0, \qquad U_i(W_i)[R] = 0.
$$
Therefore,
$$
A_t[R]
=
\left(
\sum_{i \in P'_t} U_i(V_i)
\right)[R]
=
\left(
\sum_{i \in \kappa(R)} U_i(V_i)
\right)[R],
$$
and similarly
$$
B_t[R]
=
\left(
\sum_{i \in \kappa(R)} U_i(W_i)
\right)[R].
$$
Comparing with the definition~\eqref{eq:infinitediff-formal}, we see
that
$$
\frac{A_t[R]}{B_t[R]}
=
\Psi(J_t^{(s)} \mid z^{(s)})[R]
=
J_t^{(s)}[R],
$$
which is the claim.
\end{proof}

We can now state the main seed consistency result.

\begin{theorem}[Seed consistency of the algorithm]
\label{theorem:seed-consistency}
Under Assumption~1, the InfiniteDiffusion lazy query algorithm is seed-consistent.
\end{theorem}

\begin{proof}
By Lemma~\ref{lem:deterministic-infinitediffusion}, for each $s$, $t$, and finite $R$, $J_t^{(s)}[R]$ is uniquely determined by $(s,R)$, so the process defined by \eqref{eq:infinitediff-formal} is seed-consistent. By Lemma ~\ref{lem:single-query-correctness}, the algorithm is equivalent, and therefore also seed-consistent.
\end{proof}

Theorem~\ref{theorem:seed-consistency} formalizes the informal argument in the main text: once a seed is fixed, the entire trajectory $(J_t^{(s)})_{t=0}^T$ is fully determined, and the lazy querying and caching scheme merely memoizes a deterministic computation without affecting its outcome.  In particular, querying regions in different orders or repeating a query for the same region cannot change the result.

\subsection{Constant-Time Random Access}

We now formalize the claim that accessing the value on any window-sized region has constant computational cost, independent of the absolute location in the infinite domain. We measure cost in terms of the number of evaluations of $\Phi$, which dominate runtime in practice.  Let $C_t(R)$ denote the worst-case number of $\Phi$-calls required by the lazy algorithm to answer a query for $J_t[R]$, starting from an empty cache at all timesteps.

\medskip
\noindent
\textbf{Assumption 2 (Uniform overlap bound).}
There exists a finite constant $M$ such that $|\kappa(R_i)| \leq M$ for every window region $R_i$. In words, each window region overlaps the regions of at most $M$ windows.

We treat the total number of diffusion steps $T$ and the window shape as fixed hyperparameters of the model.

\begin{lemma}[Recursive cost bound]
\label{lem:recursive-cost}
Under Assumption~2, for any timestep $t$ and any window index
$i \in S$, the cost $C_t(R_i)$ satisfies
$$
C_t(R_i) \;\leq\; M(1+\sup_{j \in S} C_{t+1}(R_j))
\quad \text{for } t < T,
$$
with base case $C_T(R_i) = 0$ for all $i$.
\end{lemma}

\begin{proof}
Consider a query for $J_t[R_i]$ at some $t < T$.  By the update rule, computing this query requires evaluating $\Phi(J_{t+1}[R_k], y_k)$ for every $k \in \kappa(R_i)$.

Each such evaluation requires one call to $\Phi$ at level $t$, and in order to provide the input $J_{t+1}[R_k]$, the algorithm may in turn need to perform some number of step-$(t+1)$ queries.  That is, for each $k \in \kappa(R_i)$ we incur cost $1 + C_{t+1}(R_k)$.

The cardinality of $\kappa(R_i)$ is at most $M$ by Assumption~2.  Thus
$$
C_t(R_i) \leq \sum_{k \in \kappa(R_i)} (1+C_{t+1}(R_k))
\leq  M(1 + \sup_{j\in S} C_{t+1}(R_j))
$$
At the top level $t = T$, no further calls to $\Phi$ are required because $J_T$ is given directly by the noise generator, hence $C_T(R_i) = 0$.
\end{proof}

\begin{theorem}[Uniform bound on cost for window regions]
\label{theorem:uniform-cost}
Under Assumption~2, there exists a constant $K$ depending only on $T$
and $M$ such that for all $t$ and all window indices $i$,
$$
C_t(R_i) \leq K.
$$
Equivalently, a
\end{theorem}

\begin{proof}
Let
$$
c_t \coloneqq \sup_{i \in S} C_t(R_i).
$$
By Lemma~\ref{lem:recursive-cost},
$$
c_t \leq M(1+c_{t+1}), \quad t < T,
$$
with $c_T = 0$.

Unwinding this recurrence yields
$$
c_{T-1} \leq M(1+0) = M, \quad
c_{T-2} \leq M(1+c_{T-1}) \leq M + Mc_{T-1},
$$
and in general
$$
c_t
\leq M + M^2 + \dots + M^{T-t}.
$$
For fixed $T$ and $M$, the right-hand side is a finite constant independent of $i$ and the absolute location of $R_i$ in the plane.  Taking
$$
K \coloneqq M + M^2 + \dots + M^{T}
$$
gives the claimed uniform bound.
\end{proof}

\begin{corollary}[Uniform bound on cost for any fixed-size region]
\label{theorem:universal-uniform-cost}
Suppose that for a region $R$ of a given size, $|\kappa(R)|$ is bounded by a constant $M$ dependent on that size. Then accessing $J_t[R]$ for any fixed-size region is $O(1)$.
\end{corollary}

\begin{proof}
Let $R$ be any region of a fixed size. By assumption, $R$ has at most $M$ overlapping windows, so $C_t(R) \leq MK = O(1)$.
\end{proof}

For any window index $i$, Theorem~\ref{theorem:uniform-cost} gives a uniform bound $C_0(R_i) \le K$ on the number of calls to $\Phi$ needed to answer a query for $J_0[R_i]$, starting from an empty cache. This bound depends only on $T$ and $M$, which are constants, and not on $i$ or any regions previously processed. In standard algorithmic notation, this means that the time complexity of a query $J_0[R]$ is $O(1)$, when $R$ is a window region. Queries typically cost far less than $K$, but this is not required for the asymptotic guarantee.

Combined with seed consistency, Theorem ~\ref{theorem:uniform-cost} formally justifies the claim that users can jump to arbitrary locations in the infinite world and query any window region efficiently, without needing to generate intermediate tiles and without affecting the content of the regions.

\subsection{Parallelization}

Finally, we formalize the statement that InfiniteDiffusion admits parallel evaluation of window updates at any fixed timestep.

Recall that, at fixed $t$ and seed $s$, the numerator and denominator images for the update $J_{t+1}^{(s)} \mapsto J_t^{(s)}$ are
\begin{align*}
A_t^{(s)} &=
\sum_{i \in S}
  U_i\bigl( W_i \otimes \Phi(J_{t+1}^{(s)}[R_i] \mid y_i^{(s)})\bigr),\\
B_t^{(s)} &=
\sum_{i \in S}
  U_i(W_i),
\end{align*}
The updated image is then
$$
J_t^{(s)} = A_t^{(s)} / B_t^{(s)}.
$$

We treat calls to $\Phi$ as the only expensive operation, and all other operations (additions, multiplications, divisions and updates to $A_t,B_t$) as free.  In this work, a computation is \emph{parallelizable} if there exists an algorithm in which all calls to $\Phi$ can be partitioned into a constant number of \emph{rounds} so that within each round the calls are independent and can be executed simultaneously.

We first note that, at a fixed timestep, once the inputs $J_{t+1}^{(s)}[R_i]$ are known, all required model evaluations can be done in parallel.

\begin{lemma}[Parallel window updates at a fixed timestep]
\label{lem:single-step-parallel}
Fix a seed $s$, a timestep $t \in \{0,\dots,T-1\}$, and a finite set of window indices $I \subseteq S$.  Suppose that for every $i \in I$ the infinite image $J_{t+1}^{(s)}[R_i]$ is available for free.  Then the evaluations
$$
\Phi(J_{t+1}^{(s)}[R_i] \mid y_i^{(s)}), \qquad i \in I,
$$
can be performed in a single parallel round, and all corresponding contributions to $A_t^{(s)}$ and $B_t^{(s)}$ on $\bigcup_{i \in I} R_i$ can be formed without any further calls to $\Phi$.
\end{lemma}

\begin{proof}
For each $i \in I$, the input to $\Phi$ depends only on $J_{t+1}^{(s)}[R_i]$ and $y_i^{(s)}$, which are both already available.  Thus the evaluations $\Phi(J_{t+1}^{(s)}[R_i] \mid y_i^{(s)})$ are mutually independent and can be carried out simultaneously.

Once these outputs are known, the updates
\begin{align*}
A_t^{(s)} &\gets A_t^{(s)} + U_i\bigl(W_i \otimes \Phi(J_{t+1}^{(s)}[R_i] \mid y_i^{(s)})\bigr) \\
B_t^{(s)} &\gets B_t^{(s)} + U_i(W_i)
\end{align*}
and the division $J_t^{(s)} = A_t^{(s)} / B_t^{(s)}$ involve only element-wise arithmetic and therefore require no additional model evaluations.  Hence all work associated with the windows in $I$ can be completed using a single parallel round of calls to $\Phi$.
\end{proof}

We now show that answering any finite collection of region queries at any timestep admits a parallel schedule of model evaluations.

\begin{theorem}[Parallelization of finite query sets]
\label{theorem:parallel-queries}
Fix a seed $s$, a timestep $t \in \{0,\dots,T\}$, and a finite collection of regions
$$
\mathcal{R} = \{R^{(1)},\dots,R^{(m)}\}.
$$
Consider the computation that evaluates $J_t^{(s)}[R^{(k)}]$ for all $k=1,\dots,m$ using the recursive update~\eqref{eq:infinitediff-formal}, starting from $J_T^{(s)}$ and without caching.  Then all calls to $\Phi$ required by this computation can be arranged into at most $T-t$ parallel rounds.
\end{theorem}

\begin{proof}
By Lemma~\ref{lem:deterministic-infinitediffusion}, once $s$ is fixed the values $J_u^{(s)}[R]$ are uniquely determined for all $u$ and all finite regions $R$.  In particular, the \emph{set} of model evaluations that appear in the recursive computation of $\{J_t^{(s)}[R^{(k)}]\}_{k=1}^m$ is fixed; only their order of execution is not.

We prove the claim by backward induction on $t$.

\emph{Base case ($t = T$).}  By definition, $J_T^{(s)}$ is given directly by the noise generator $G$ and does not require any calls to $\Phi$.  Thus any finite set of queries at $t = T$ is trivially parallelizable with zero rounds.

\emph{Inductive step.}  Fix $t < T$ and assume the statement holds for $t+1$.  Consider a finite collection $\mathcal{R} = \{R^{(1)},\dots,R^{(m)}\}$ of regions at time $t$.

Let
$$
R^\ast \coloneqq \bigcup_{k=1}^m R^{(k)}
$$
be the union of all queried regions at level $t$.  By Assumption~1, only finitely many windows intersect $R^\ast$, so the index set
$$
I_t \coloneqq \kappa(R^\ast) = \{i \in S : R_i \cap R^\ast \neq \emptyset\}
$$
is finite.  By~\eqref{eq:infinitediff-formal}, computing $J_t^{(s)}[R^{(k)}]$ for all $k$ requires knowing the regions $J_{t+1}^{(s)}[R_i]$ for all $i \in I_t$.

Each region $R_i$ is finite, and from the recursion defining $J_t^{(s)}$ we see that $J_t^{(s)}[R_i]$ itself is obtained from regions of the form $J_{t+1}^{(s)}[R_j]$ for windows $j$ that intersect $R_i$.  Let $\mathcal{R}_{t+1}$ denote the (finite) collection of all such window regions $R_j$ at timestep $t+1$ that are needed in this way.  By the induction hypothesis applied at level $t+1$ to the finite set $\mathcal{R}_{t+1}$, all model evaluations needed to compute $\{J_{t+1}^{(s)}[R_j] : R_j \in \mathcal{R}_{t+1}\}$ can be scheduled in at most $T-(t+1)$ parallel rounds.

Once all these regions are available, the corresponding values $J_t^{(s)}[R_i]$ for $i \in I_t$ are determined and can be formed without further calls to $\Phi$.  At this point, Lemma~\ref{lem:single-step-parallel} implies that all remaining evaluations
$$
\Phi(J_{t+1}^{(s)}[R_i] \mid y_i^{(s)}), \qquad i \in I_t,
$$
required to construct $A_t^{(s)}$ and $B_t^{(s)}$ on $R^\ast$ can be carried out in a single additional parallel round.  No more model evaluations are needed to extract $J_t^{(s)}[R^{(k)}]$ from these images.

Thus, the entire computation for the queries $\mathcal{R}$ at timestep $t$ can be performed using at most
$$
1 + (T-(t+1)) = T - t
$$
parallel rounds of calls to $\Phi$.  This completes the induction.
\end{proof}

Theorem~\ref{theorem:parallel-queries} shows that for any fixed seed $s$, timestep $t$ and finite collection of regions $\mathcal{R}$, the recursive computation of $\{J_t^{(s)}[R] : R \in \mathcal{R}\}$ admits a bounded-depth parallel schedule of diffusion model evaluations.  Combined with seed consistency, this formalizes the claim in the main text that all diffusion model evaluations required to serve any finite batch of queries $J_t^{(s)}[R]$ can be performed in parallel, up to the intrinsic sequential dependence across diffusion steps.

\section{The Infinite Tensor Framework}
\label{sec:infinite_tensor_appendix}

To support the computational requirements of InfiniteDiffusion, specifically the need for bounded-memory operations on unbounded domains, we developed \texttt{infinite-tensor}. This open-source Python library\footnote{Source code, documentation, and installation instructions are available at https://github.com/xandergos/infinite-tensor.} abstracts the management of sliding windows, caching, persistence, and dependency chaining, allowing the implementation of diffusion models to remain focused on mathematical operations rather than memory management.

\subsection{Core Abstractions}

The framework treats an infinite tensor $\mathcal{T}$ not as a stored array of data, but as a lazily evaluated, immutable object defined by a deterministic generator function $f$. For context, we typically store the accumulators $A_t$ and $B_t$ from Section 3.3 as one infinite tensor with an extra channel for the weights $B_t$.

\begin{itemize}
    \item \textbf{Infinite Tensors:} A tensor is defined by a shape $(S_1, \dots, S_d)$, where any dimension $S_i$ may be infinite (represented as \texttt{None}). The tensor is backed by a generator function $f$ rather than raw memory, and also declares a \texttt{dtype} and \texttt{device}.
    \item \textbf{Windows:} $f$ operates on fixed-size \textit{windows}. For a $k$ dimensional tensor, a window is defined by its size $(W_1, \dots, W_k)$, stride $s\in \mathbb{Z}^k$, and offset $o \in \mathbb{Z}^k$.
    \item \textbf{TileStore:} Each tensor is backed by a \texttt{TileStore}, which handles caching and storage. Repeated access to the same spatial region reuses cached results rather than re-triggering computation, and dependencies between tensors are tracked automatically. The framework defaults to an in-memory MemoryTileStore.
\end{itemize}

\subsection{Caching and Persistence}
\label{sec:appendix_caching}

The framework exposes two storage backends.

\paragraph{In-memory caching (Default)}
This backend stores computed windows in a shared LRU (Least Recently Used) cache. Conceptually, if $\mathcal{T}[R_i] \gets \mathcal{T}[R_i] + x$ for each processed window region $R_i$, then a requested region $\mathcal{T}[R]$ is reconstructed from the cached window outputs intersecting $R$. Users may bound memory usage directly, for example by specifying a cache-size limit in bytes or in number of windows. This mode is appropriate for transient inference where persistent storage is undesirable. This mode is used in all applications in the main text.

\paragraph{Persistent tiled storage (Optional)}
This backend accumulates overlapping window outputs into fixed-size storage tiles in a user-provided storage device. Compared with storing each window independently, this uses storage more efficiently because overlapping regions share the same underlying tiles. This mode is appropriate for long-running workflows where results should persist across sessions. A default implementation is provided for HDF5.

\subsection{Dependency Chaining and DAGs}

The framework supports the construction of directed acyclic graphs (DAGs) of infinite tensors. A tensor $\mathcal{T}_{child}$ can declare a dependency on $\mathcal{T}_{parent}$. When a region of $\mathcal{T}_{child}$ is requested:
1. The framework calculates the required covering region of $\mathcal{T}_{parent}$.
2. The parent region is fetched (triggering upstream generation if necessary).
3. The parent data is sliced and passed to the child's generator function $f$.

This allows for the construction of complex hierarchical pipelines (e.g., Coarse Model $\rightarrow$ Latent Model $\rightarrow$ Decoder) where each stage is an infinite tensor depending on the previous one.

\section{Additional Panorama Details}
\label{app:panorama_details}

\subsection{Methodology of Table 1}
\label{app:table1methodology}

We provide representative samples of both InfiniteDiffusion and MultiDiffusion on various prompts applied to Stable Diffusion in Table 1 in the main text. FID is calculated by first generating 2000 images with Stable Diffusion for 8 prompts, resulting in 16000 total reference images. We first generate a second version of this dataset in the exact same way, and calculate the FID on both versions. This produces the Stable Diffusion baseline FID. We calculate the MultiDiffusion FID by generating 16000 $512 \times 2048$ panoramas with the same prompts, and choose random $512 \times 512$ crops. For InfiniteDiffusion, we draw the random crops from an infinite panorama, also from random $512 \times 512$ regions, with widths entirely in the range [0, 2048]. For both InfiniteDiffusion and MultiDiffusion, we use a linear kernel, and a stride of 32 (out of 64). MultiDiffusion blends tiles after every iteration, while InfiniteDiffusion with $T = 2$ uses just one intermediate blending step at $t_{inter} = 800$ (from 1-999), which we found worked well, though we did not perform a complete ablation. InfiniteDiffusion with $T = 1$ performs blending in the latent space only after generation. InfiniteDiffusion also tiles the decoder with a stride of 384 (out of 512), while MultiDiffusion simply decodes the entire panorama simultaneously.

\begin{figure}[h]
\centering
\includegraphics[width=\linewidth]{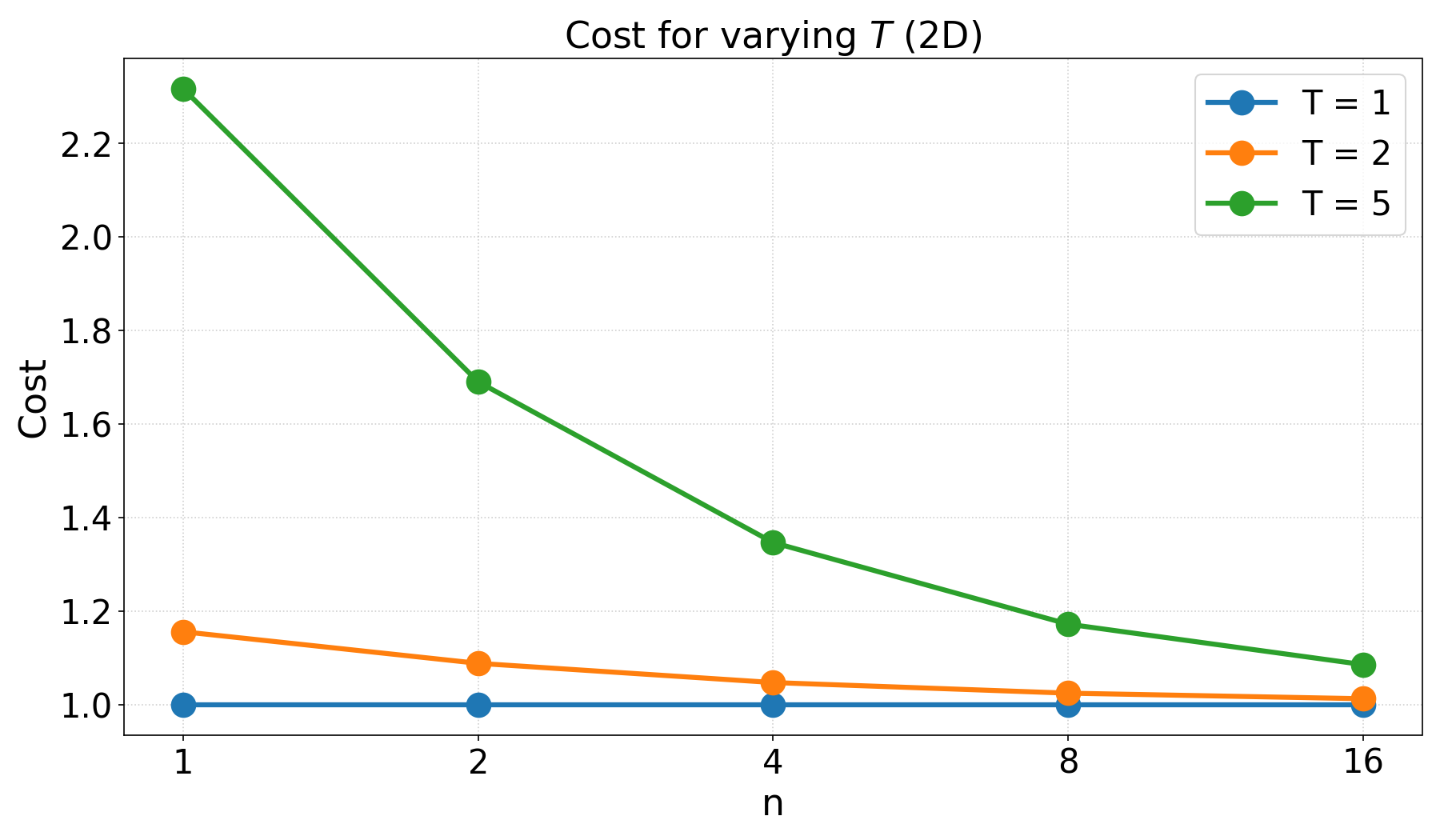}
\caption{Compute overhead of $T=2,5$ versus $T=1$ for a 2D ($n \times n$) window grid.}
\label{fig:t_latency_2d}
\end{figure}

\begin{figure}[h]
\centering
\includegraphics[width=\linewidth]{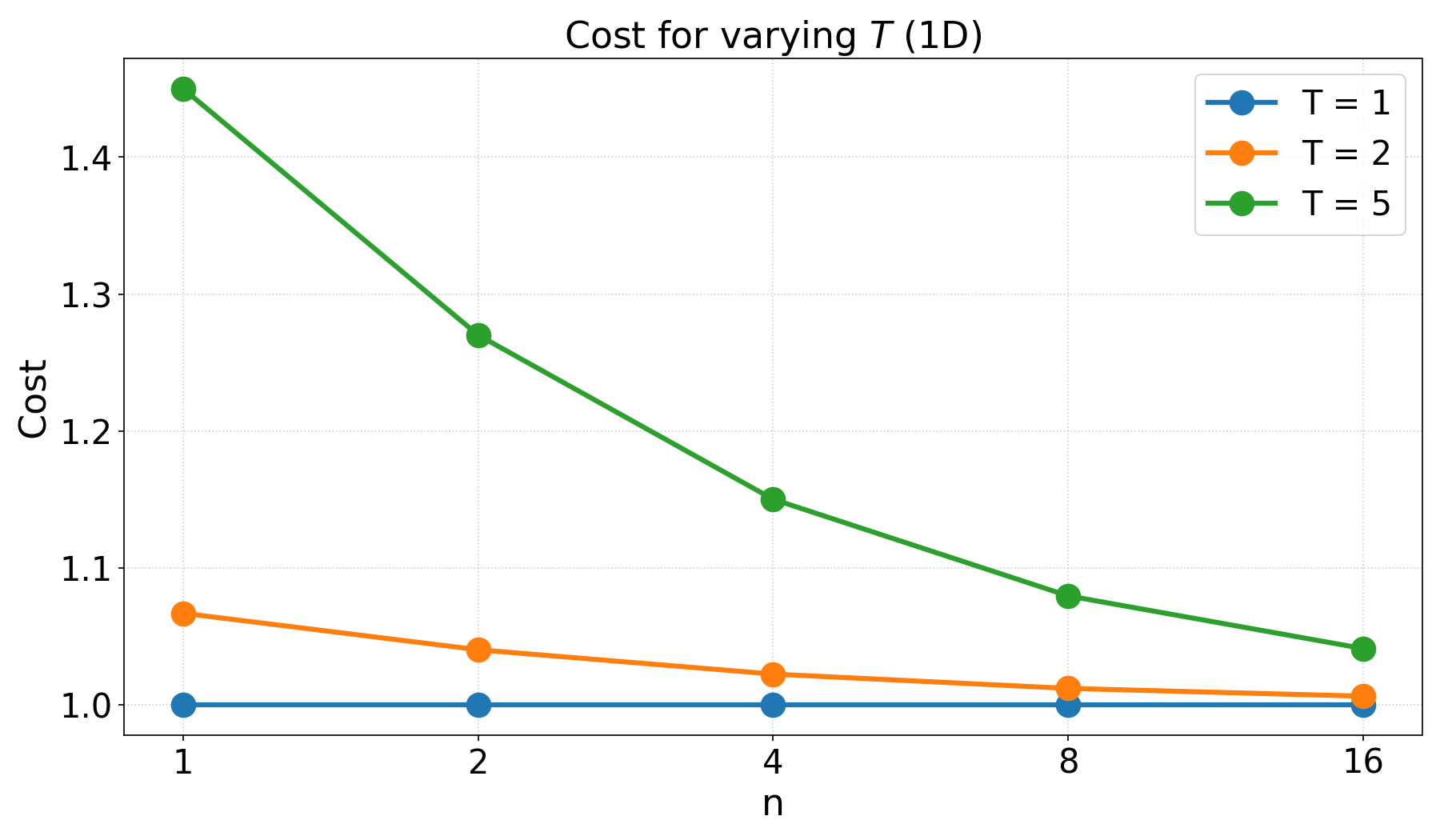}
\caption{Compute overhead of $T=2,5$ versus $T=1$ for a 1D ($n \times 1$) window grid.}
\label{fig:t_latency_1d}
\end{figure}

We use the following 8 prompts:
\begin{enumerate}
    \item A photo of a city skyline at night
    \item A photo of the dolomites
    \item A photo of lush forest with a babbling brook
    \item A photo of a forest with a misty fog
    \item A photo of mountain range at twilight
    \item Cartoon panorama of spring summer beautiful nature
    \item A photo of a snowy mountain peak with skiers
    \item Natural landscape in anime style illustration
\end{enumerate}

\subsection{Additional Samples}
In Figure~\ref{fig:infinite_diffusion_viz_appdx_1} through Figure~\ref{fig:infinite_diffusion_viz_appdx_2} we include uncurated samples of Stable Diffusion panoramas created using MultiDiffusion and InfiniteDiffusion with $T=1$, $T=2$, and $T=5$. We find that $T=1$ or $T=2$ are sufficient for many tasks, while $T=5$ resolves practically all artifacts in every case we tested. Generally, InfiniteDiffusion with minimal $T$ seems to work best for the same types of images where MultiDiffusion itself consistently produces high quality outputs. 

We used the same prompts as the FID computation. For $T=5$ we blend 4 times at ($400$, $600$, $750$, $900$). Both methods use DDIM with 50 sampling steps. We did not perform an exhaustive search for blending timesteps; these values were chosen qualitatively.

\subsection{Performance for T}
We also investigate how compute scales with $T$. For $T=1,2,5$, we measure the number of diffusion steps required to process an $ns \times ns$ region (2D) and an $ns \times 1$ region (1D), where $s$ is the window size and stride is 50\%. Each window's cost is weighted by its fraction of the diffusion schedule. For example, a window starting at $t=1000$ and blended at $t=800$ accounts for 20\% of the schedule. Results are shown in Tables~\ref{fig:t_latency_2d} and~\ref{fig:t_latency_1d}. $T=2$ adds minimal overhead over $T=1$, while $T=5$ is more costly but remains manageable.

Unlike MultiDiffusion, InfiniteDiffusion supports an infinite window layout. Measured on an RTX 3090 Ti, this incurs a fixed initialization overhead of 8s on the first query, regardless of region size, due to windows straddling the query boundary. Subsequent queries are cost-comparable to MultiDiffusion (e.g. 14s for a 512×2048 panorama). For reference, Stable Diffusion takes 2.5s per image.

\section{Additional Implementation Details}
\label{sec:appendix_implementation}

An overview of our architecture with hyperparameters is provided in Table ~\ref{tab:hyperparameters}.

The generative components of Terrain Diffusion are built upon the EDM2 architecture~\cite{Karras2024edm2}. We modify the base architecture, following sCM ~\cite{lu_simplifying_2025}, by implementing pixel-norm on embedding vectors and substituting Fourier embeddings with Positional embeddings. To incorporate auxiliary conditioning, we utilize a magnitude-preserving linear layer to project flattened conditioning inputs

\paragraph{Fidelity Optimization.}
For both the core latent and decoder models, we optimize fidelity using AutoGuidance~\cite{Karras2024autoguidance}. Hyperparameter sweeps are conducted via Optuna, minimizing Kernel Inception Distance (KID)~\cite{binkowski_demystifying_2018} calculated over 1,000 samples. We terminate the sweeps after 300 runs, or once KID improvement saturates.

\paragraph{Consistency Distillation.}
The core latent model and decoder are distilled into continuous-time consistency models (sCM). We perform post-hoc sweeps for the EMA parameter ($\sigma_{\text{rel}}$) and the optimal stopping iteration. This prevents fidelity degradation observed during late-stage training.

\paragraph{Training Speed}
The entire pipeline can be trained from scratch on a top-tier consumer-grade GPU (e.g., RTX 5090) in approximately one week, enabling practical retraining on custom datasets. If the decoder is reused, fine tuning the base model can be completed in about a day.

\begin{table*}[h!]
    \centering
    \caption{Hyperparameters and architectural details for the component models.}
    \label{tab:hyperparameters}
    \small
    \setlength{\tabcolsep}{4pt}
    \begin{tabular}{@{}lcccc@{}}
        \toprule
        \textbf{Parameter} & \textbf{Core Latent} & \textbf{Consistency Decoder} & \textbf{Coarse Model} & \textbf{Autoencoder} \\ 
        \midrule
        Base Architecture & EDM2 + sCM & EDM2 + sCM & EDM2 + sCM & EDM2 + sCM - Skips  \\
        Model Channels & 192 (128 Guide) & 64 (32 Guide) & 128 & 64 \\
        Channel Multipliers & $[1, 2, 3, 4]$ & $[1, 2, 3, 4]$ & $[1]$ & $[1, 2, 4, 4]$ \\
        Layers per Block & 3 & 3 & 2 & 2 \\
        Training Crop Size & $64 \times 64$ (Latent) & $128 \times 128$ & $16 \times 16$ & $128 \times 128$ \\
        Dropout & 10\% & 0\% & 10\% & 0\% \\
        \midrule
        Optimizer & Adam & Adam & Adam & Adam \\
        LR (Training) & $5 \times 10^{-3}$ ($1 \times 10^{-2}$ Guide) & $1 \times 10^{-2}$ & $1 \times 10^{-2}$ & $1 \times 10^{-2}$ \\
        LR (Distillation) & $5 \times 10^{-5}$ & $1 \times 10^{-4}$ & -- & -- \\
        Batch Size & 48 (64 Guide) & 64 & 256 & 64 \\
        Training Steps & 1,024,000 & 716,800 & 153,600 & 153,600 \\
        Guide Training Steps & 102,400 & 102,400 & -- & -- \\
        Distillation Steps & 25,600 & 25,600 & -- & -- \\
        \midrule
        \multicolumn{5}{@{}l}{\textit{Guidance Hyperparameters}} \\
        Guidance & AutoGuidance & AutoGuidance & -- & -- \\
        Main EMA $\sigma_{\text{rel}}$ & 0.12 & 0.093 & -- & -- \\
        Guide EMA $\sigma_{\text{rel}}$ & 0.14 & 0.243 & -- & -- \\
        Guide Stopping Step & 57,344 & 70,656 & -- & -- \\
        Guidance Scale & 2.15 & 1.26 & -- & -- \\
        \midrule
        \multicolumn{5}{@{}l}{\textit{Consistency / Distillation Hyperparameters}} \\
        Intermediate Timestep ($\sigma$) & 0.35 & One-Step or 0.065 & -- & -- \\
        $\sigma_{\text{rel}}$ & 0.08 & 0.21 & -- & -- \\
        EMA Step & 69,632 & 60,416 & -- & -- \\
        \bottomrule
    \end{tabular}
\end{table*}

\section{Dataset Details}
\label{sec:dataset_appendix}

Our dataset combines multiple global sources to provide consistent coverage of both land and ocean. Land elevations are drawn from the 3-arc-second MERIT DEM \cite{yamazaki_highaccuracy_2017}, while ocean bathymetry is taken from the 30-arc-second ETOPO dataset \cite{noaa_national_geophysical_data_center_etopo1_2009}. Since ETOPO's resolution is lower, it is blurred and upsampled to match MERIT’s resolution before merging. To ensure smooth coastal transitions, we measure the distance of each ocean pixel from the nearest coastline and linearly interpolate elevation from 0 m at the shore to the local ocean depth 100 pixels offshore. To support climatic conditioning, we supplement elevation with 30-arc-second WorldClim \cite{fick_worldclim_2017} data for temperature and precipitation.

For efficient processing, data is downloaded and stored in contiguous $2048\times2048$ tiles, each covering equal surface area at approximately 90m resolution. To maintain uniform ground resolution, tiles are stretched in longitude so that each pixel represents roughly the same area regardless of latitude. This equal-area tiling ensures consistent scale across training and allows all models to train on data without distortion. We also exclude tiles beyond 60 (absolute) degrees of latitude to focus on higher quality data at mid latitudes. Finally, 80\% of the tiles are randomly assigned to the train set, and the remainder are withheld for validation.

All models are trained on random crops drawn from random tiles, with sampling biased so that 99\% of tiles contain at least 1\% land, as ocean regions are simpler and lower priority. Each crop is randomly flipped and rotated in 90 degree increments to reflect our goal of generating infinite, directionless terrain. We include scripts for reproducing the dataset in our open-source repository.

\section{Signed Square-Root Transform}
\label{sec:signed-sqrt-appendix}

The signed square-root transform applies the function $$f(x)= \text{sign}(x)\sqrt{|x|}$$ to each pixel of the dataset. After inference, we undo this operation with $$f^{-1}(x) = \text{sign}(x)x^2.$$ The statistical effects of this transform are visualized in Figure ~\ref{fig:signed_sqrt}.

\begin{figure*}[t]
\centering
\includegraphics[width=\textwidth]{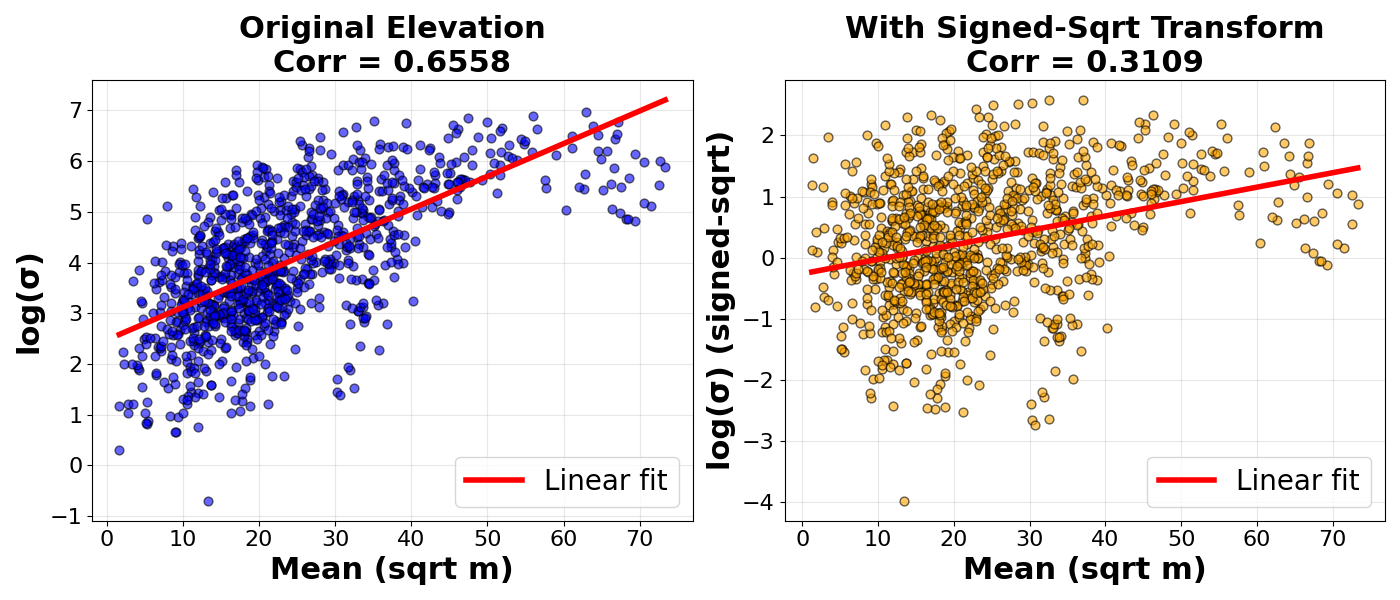}
\caption{Effects of the signed-sqrt transform. Standard deviation become more uniformly distributed with respect to mean elevation, and the range of standard deviations compress.}
\label{fig:signed_sqrt}
\end{figure*}

\newpage
\begin{figure*}[t]
    \centering
    \begin{subfigure}{0.505\textwidth}
        \includegraphics[width=\linewidth]{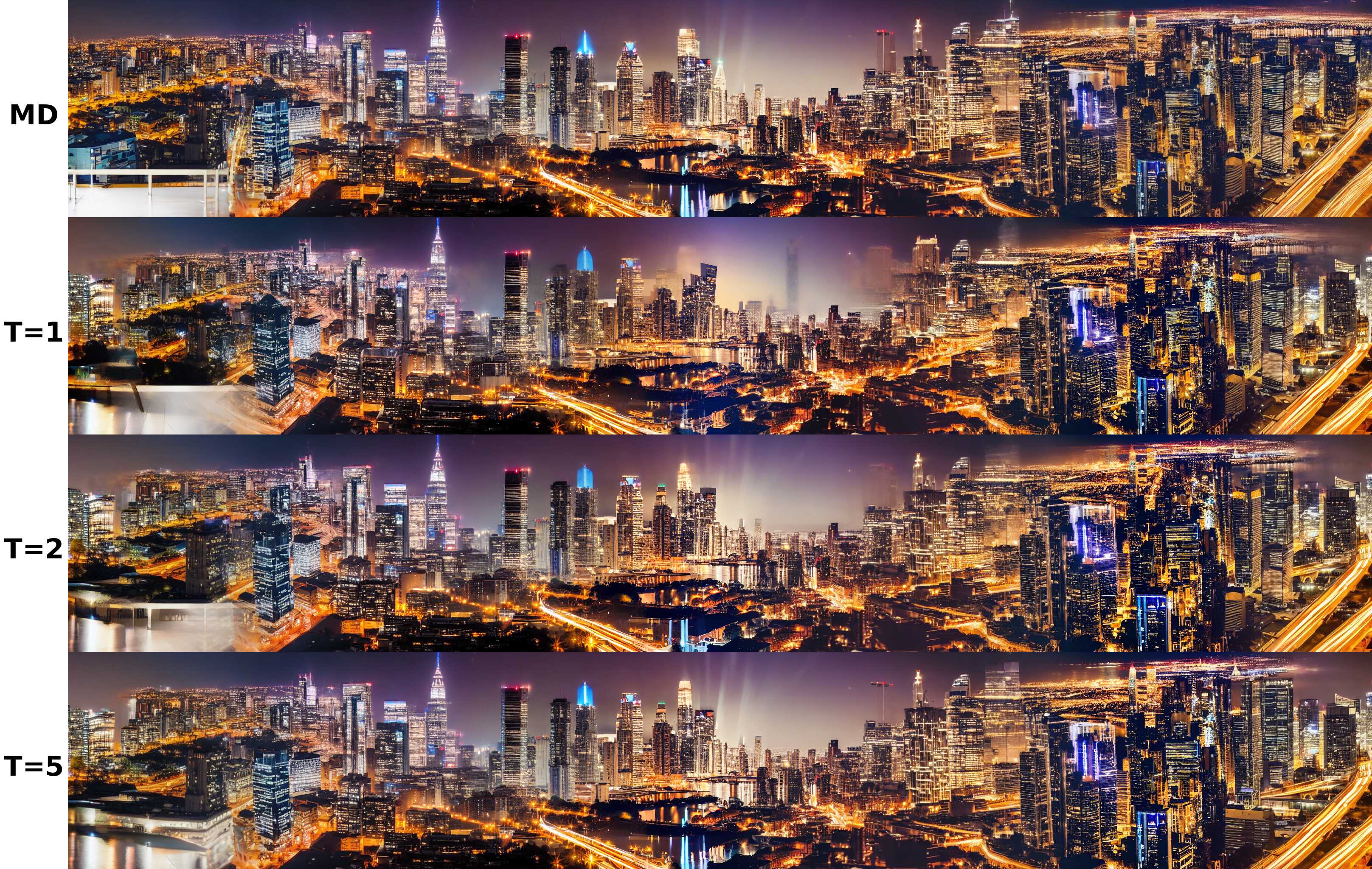}
    \end{subfigure}
    \hfill
    \begin{subfigure}{0.48\textwidth}
        \includegraphics[width=\linewidth]{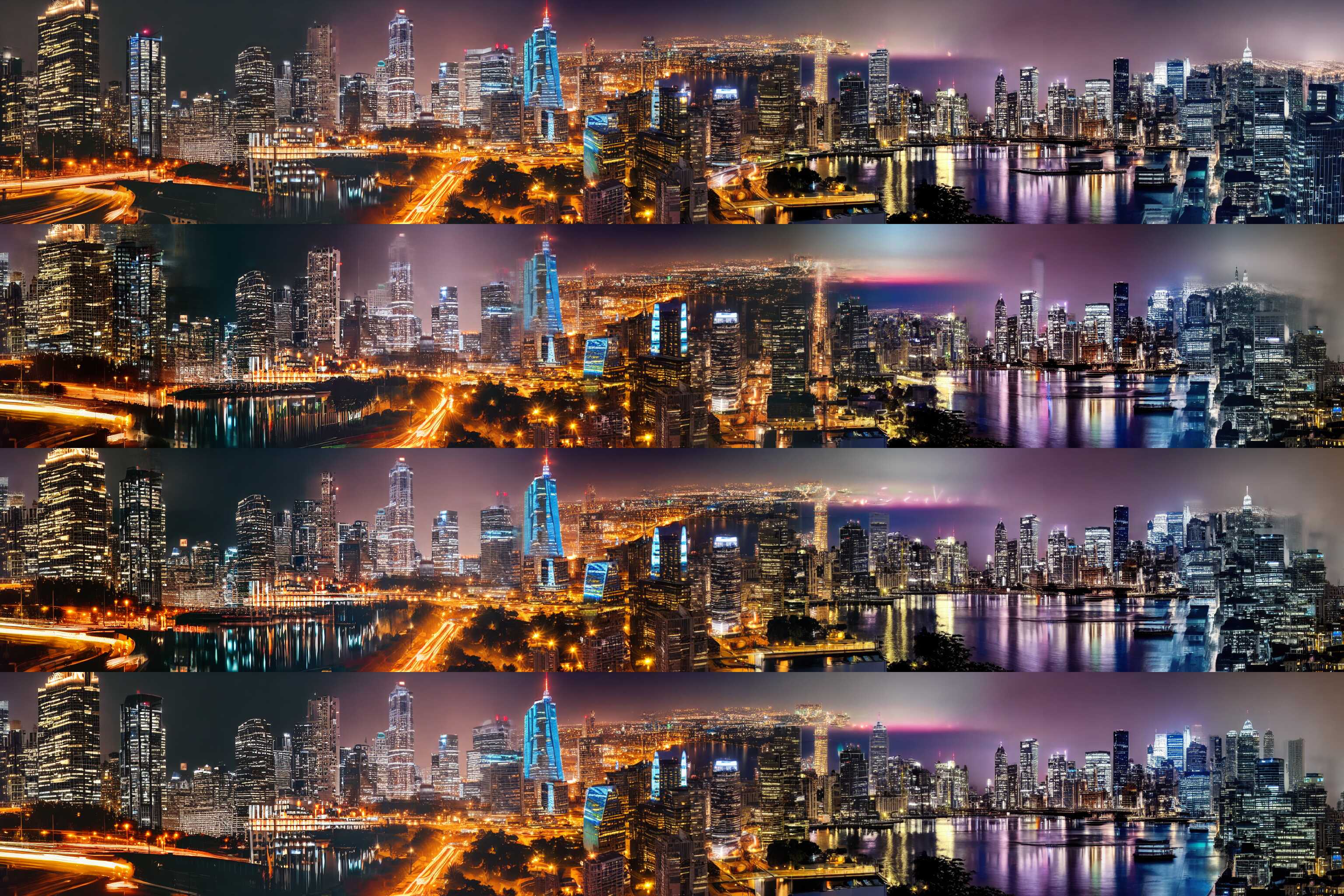}
    \end{subfigure}
    \caption{"A photo of a city skyline at night."}
    \label{fig:infinite_diffusion_viz_appdx_1}
    \vspace{0.5em}  
    \begin{subfigure}{0.505\textwidth}
        \includegraphics[width=\linewidth]{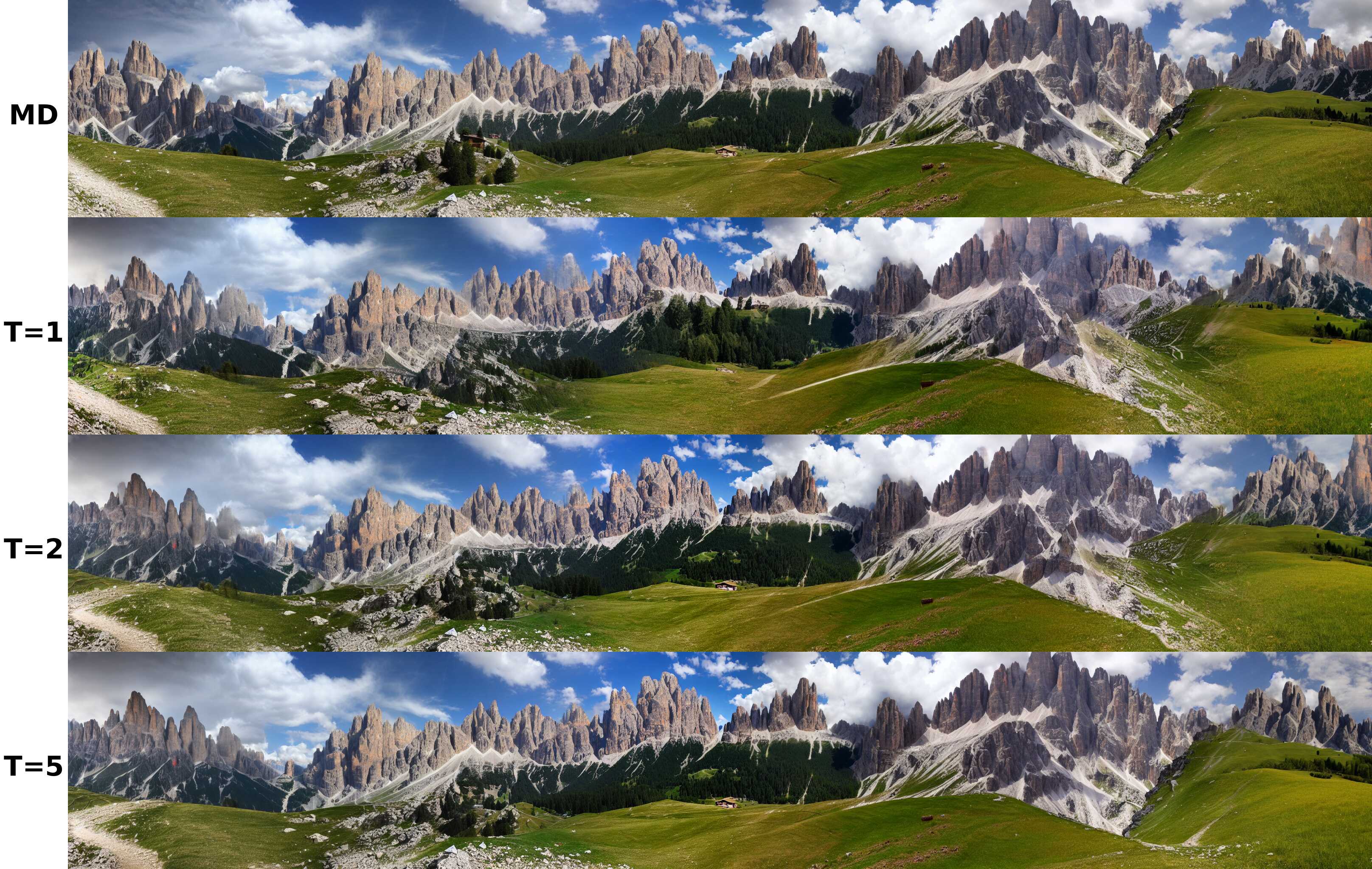}
    \end{subfigure}
    \hfill
    \begin{subfigure}{0.48\textwidth}
        \includegraphics[width=\linewidth]{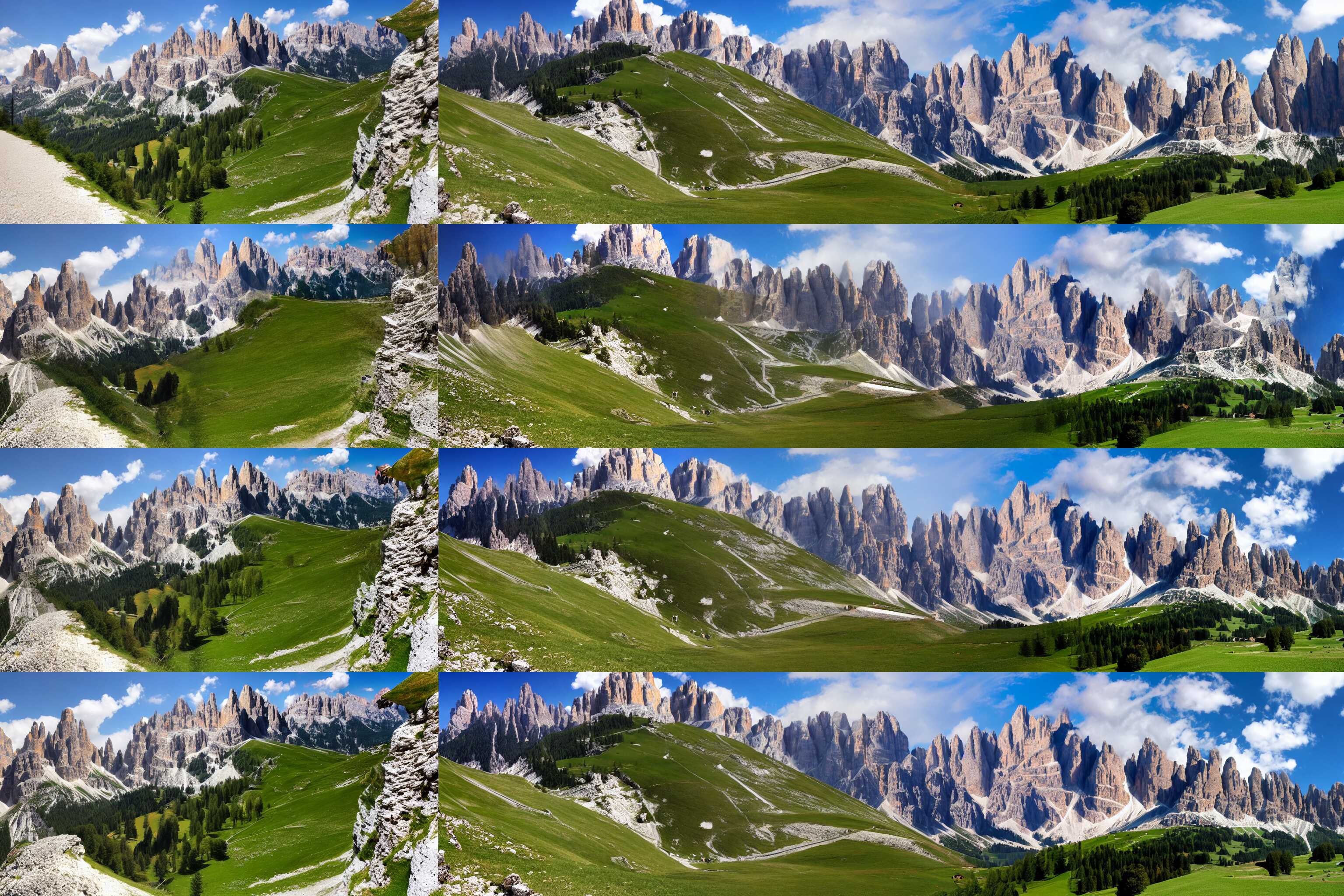}
    \end{subfigure}
    \caption{"A photo of the dolomites"}
    \vspace{0.5em}  
    \begin{subfigure}{0.505\textwidth}
        \includegraphics[width=\linewidth]{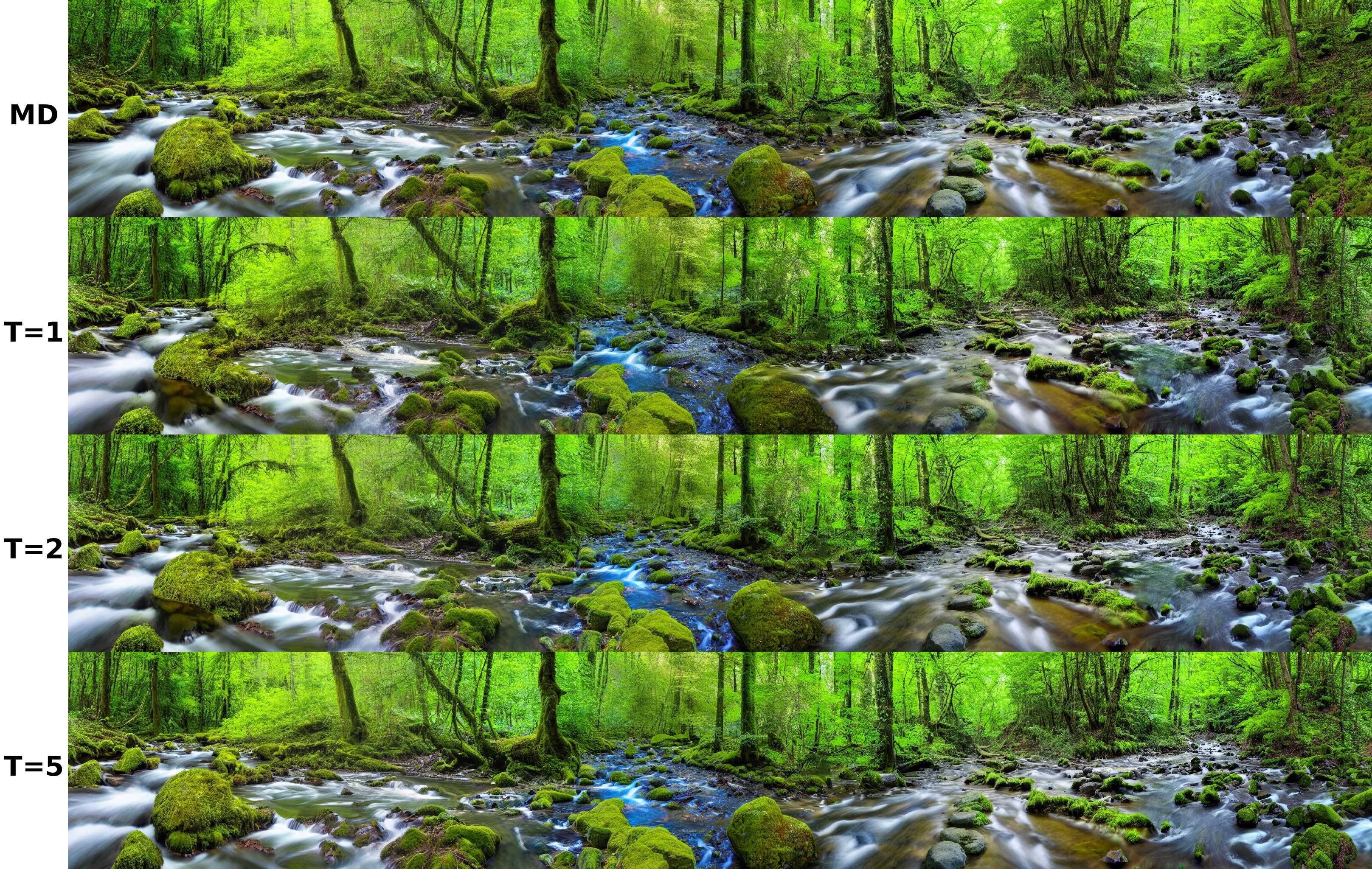}
    \end{subfigure}
    \hfill
    \begin{subfigure}{0.48\textwidth}
        \includegraphics[width=\linewidth]{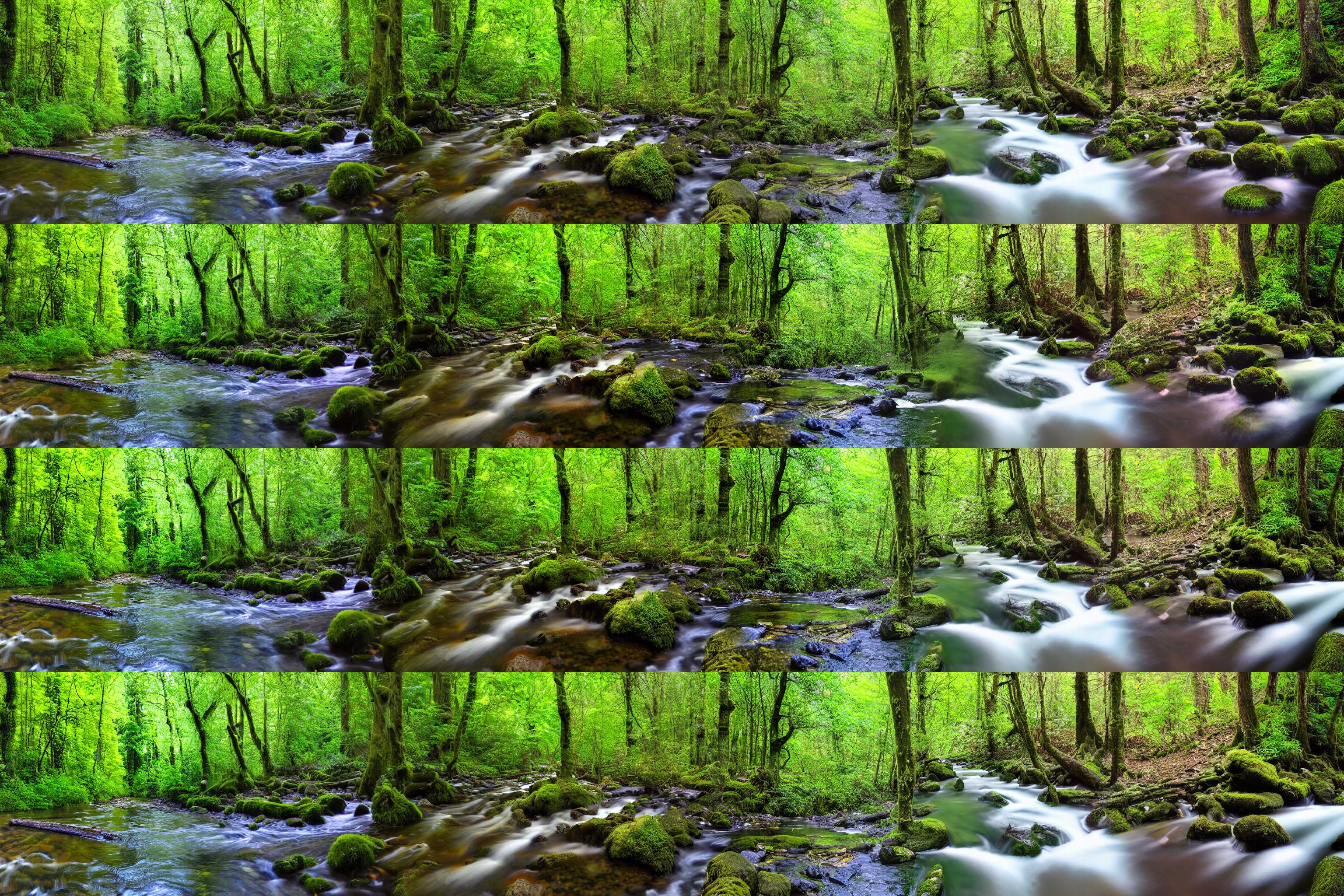}
    \end{subfigure}
    \caption{"A photo of lush forest with a babbling brook"}
\end{figure*}
\clearpage

\newpage
\begin{figure*}[t]
    \centering
    \begin{subfigure}{0.505\textwidth}
        \includegraphics[width=\linewidth]{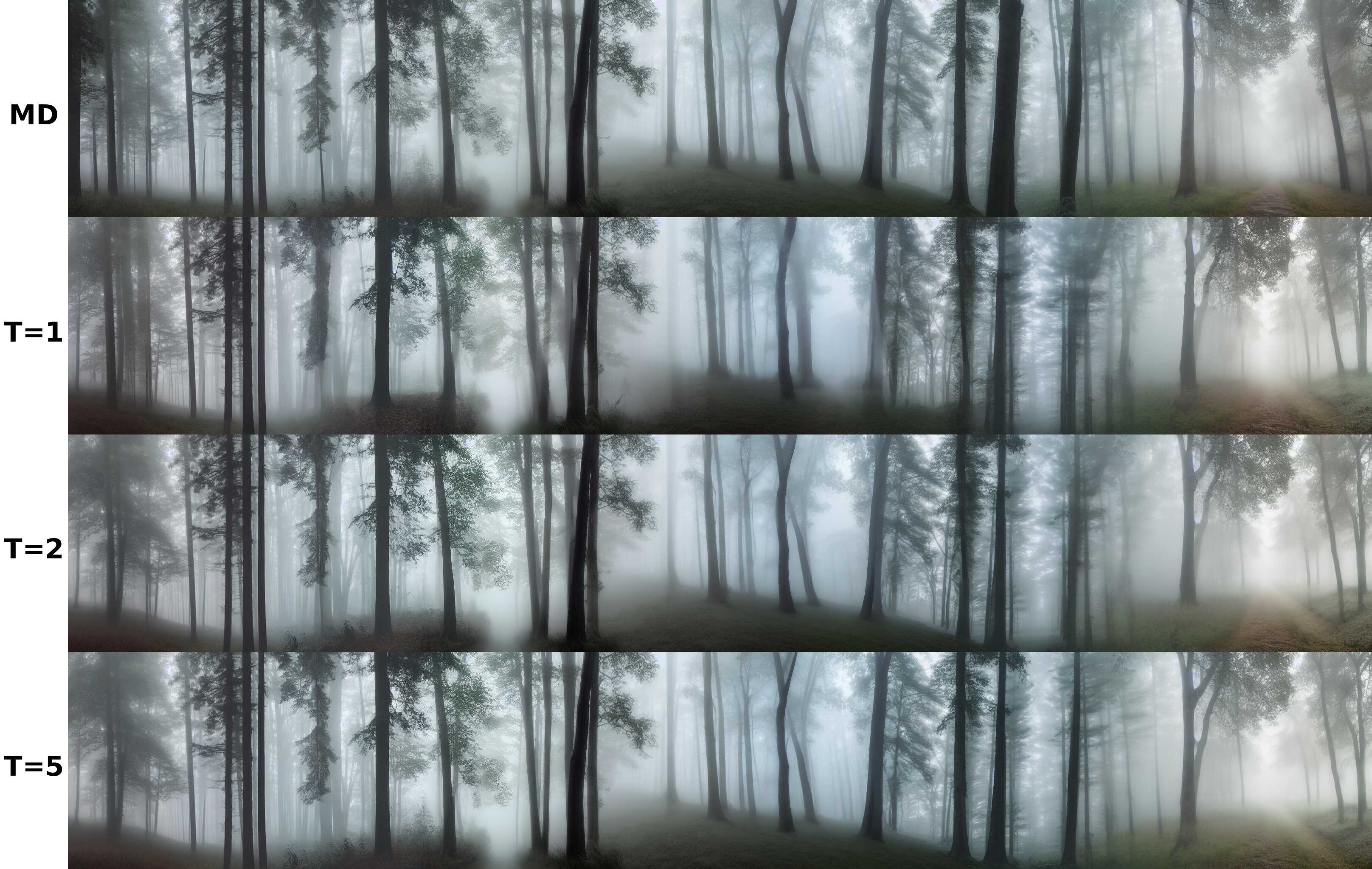}
    \end{subfigure}
    \hfill
    \begin{subfigure}{0.48\textwidth}
        \includegraphics[width=\linewidth]{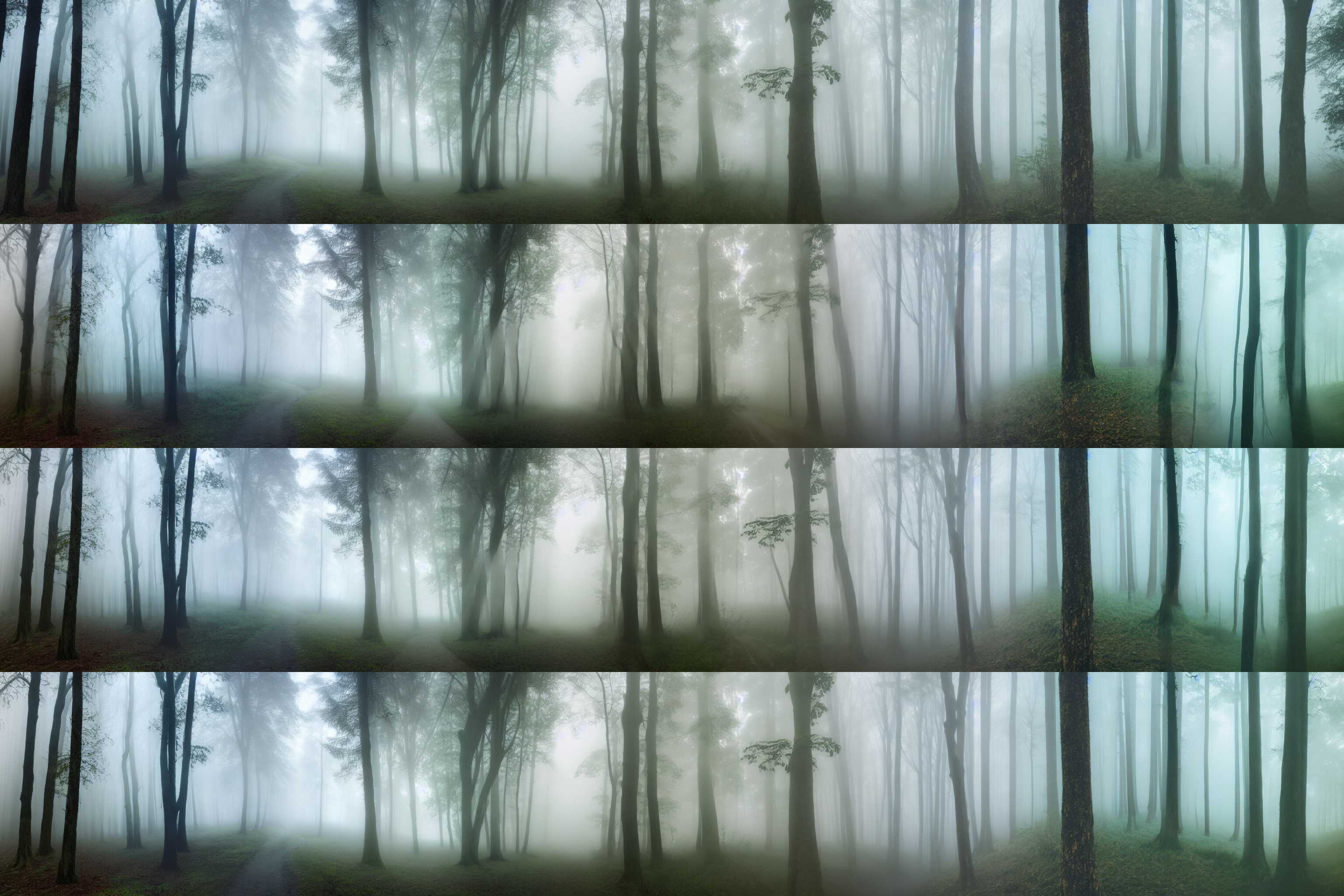}
    \end{subfigure}
    \caption{"A photo of a forest with a misty fog"}
    \vspace{0.5em}  
    \begin{subfigure}{0.505\textwidth}
        \includegraphics[width=\linewidth]{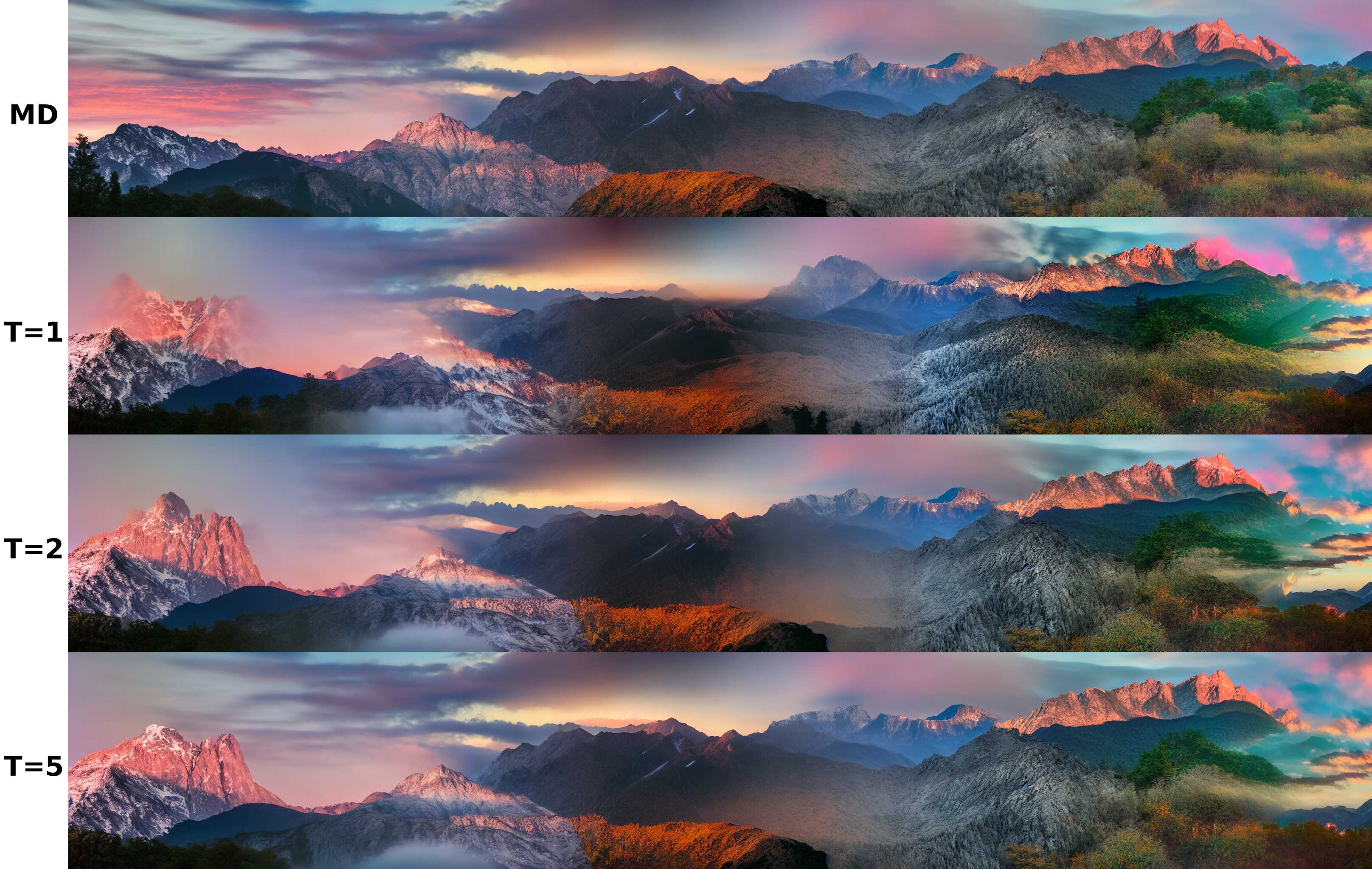}
    \end{subfigure}
    \hfill
    \begin{subfigure}{0.48\textwidth}
        \includegraphics[width=\linewidth]{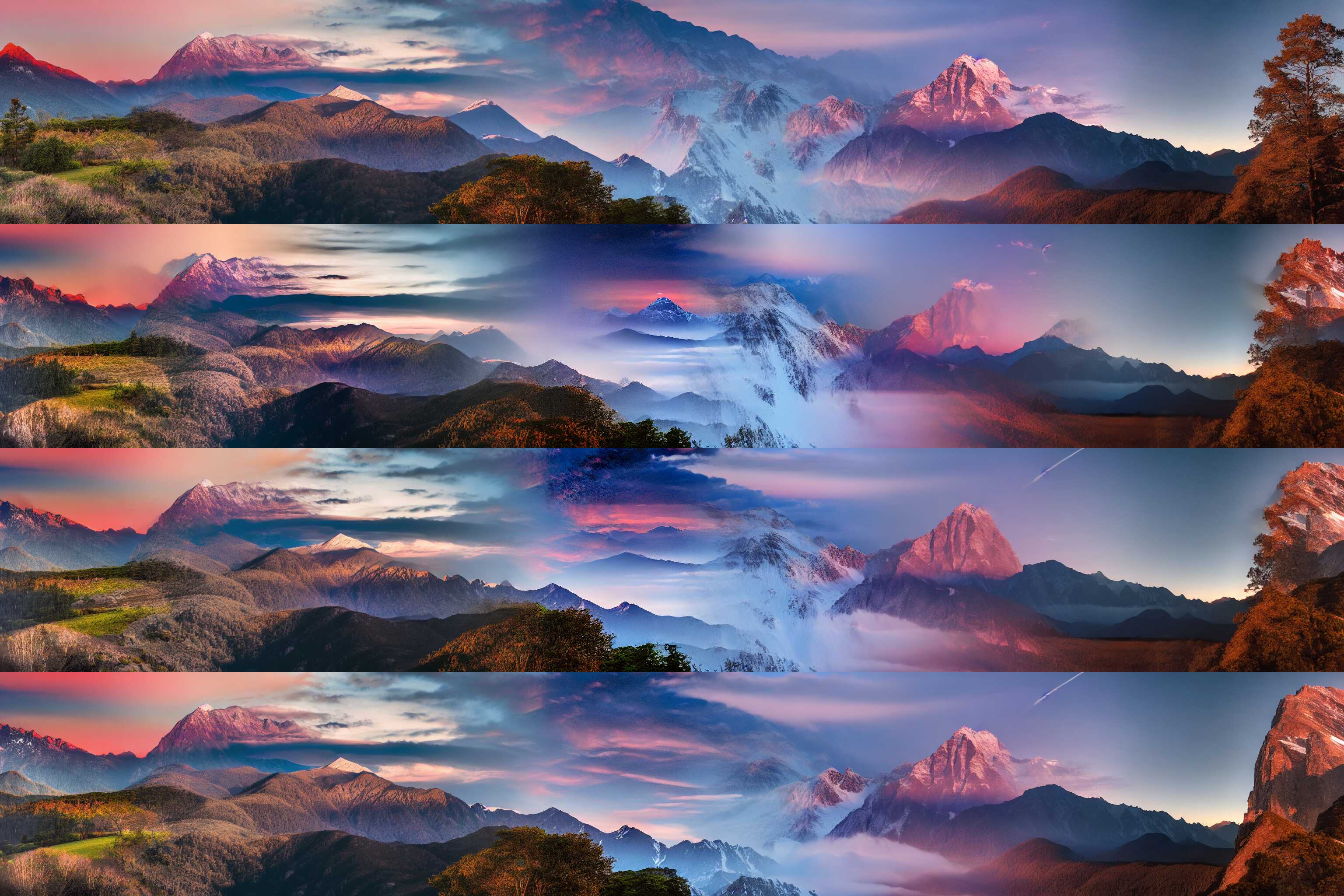}
    \end{subfigure}
    \caption{"A photo of mountain range at twilight"}
    \vspace{0.5em}  
    \begin{subfigure}{0.505\textwidth}
        \includegraphics[width=\linewidth]{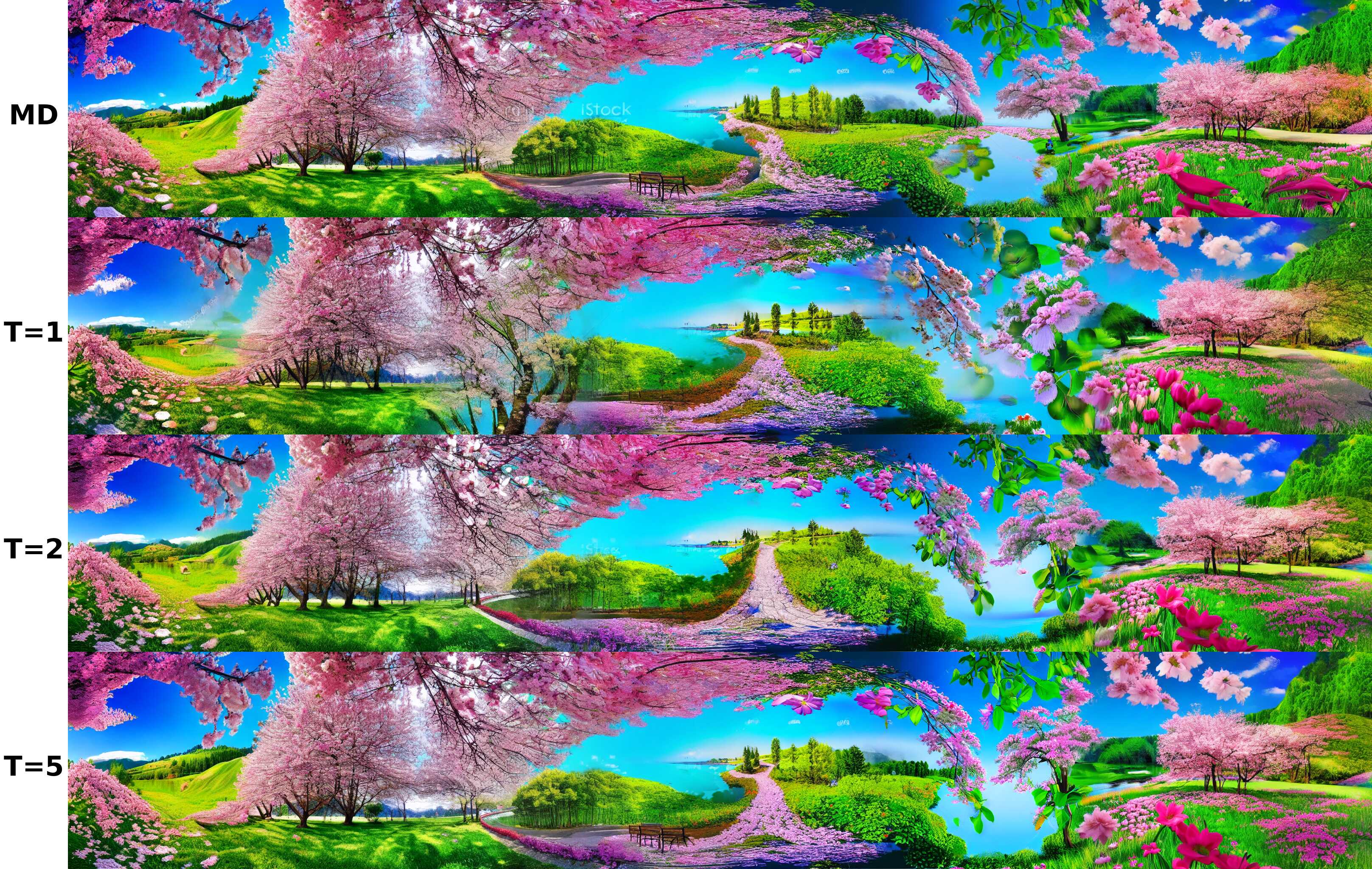}
    \end{subfigure}
    \hfill
    \begin{subfigure}{0.48\textwidth}
        \includegraphics[width=\linewidth]{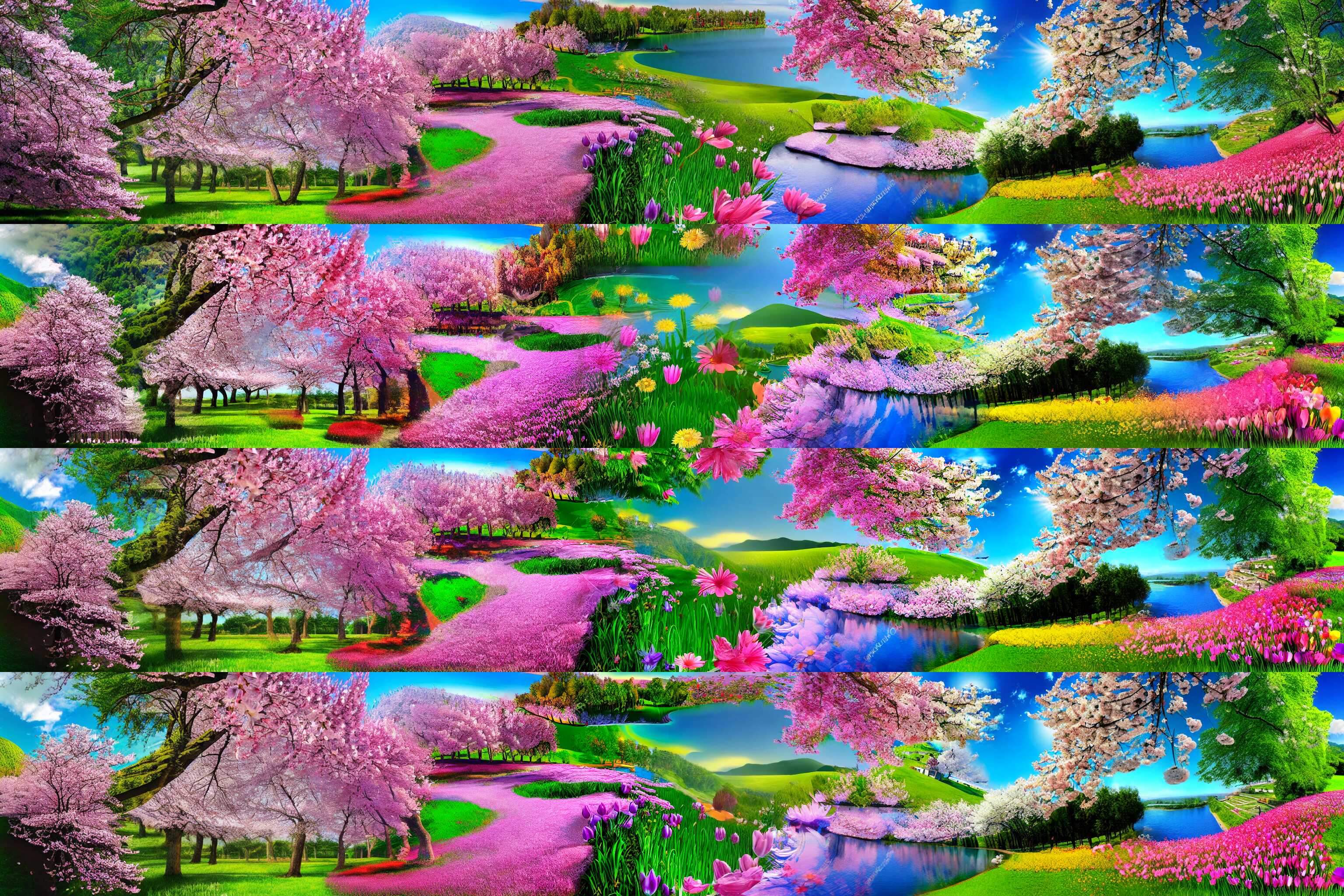}
    \end{subfigure}
    \caption{"Cartoon panorama of spring summer beautiful nature"}
\end{figure*}
\clearpage

\newpage
\begin{figure*}[t]
    \centering
    \begin{subfigure}{0.505\textwidth}
        \includegraphics[width=\linewidth]{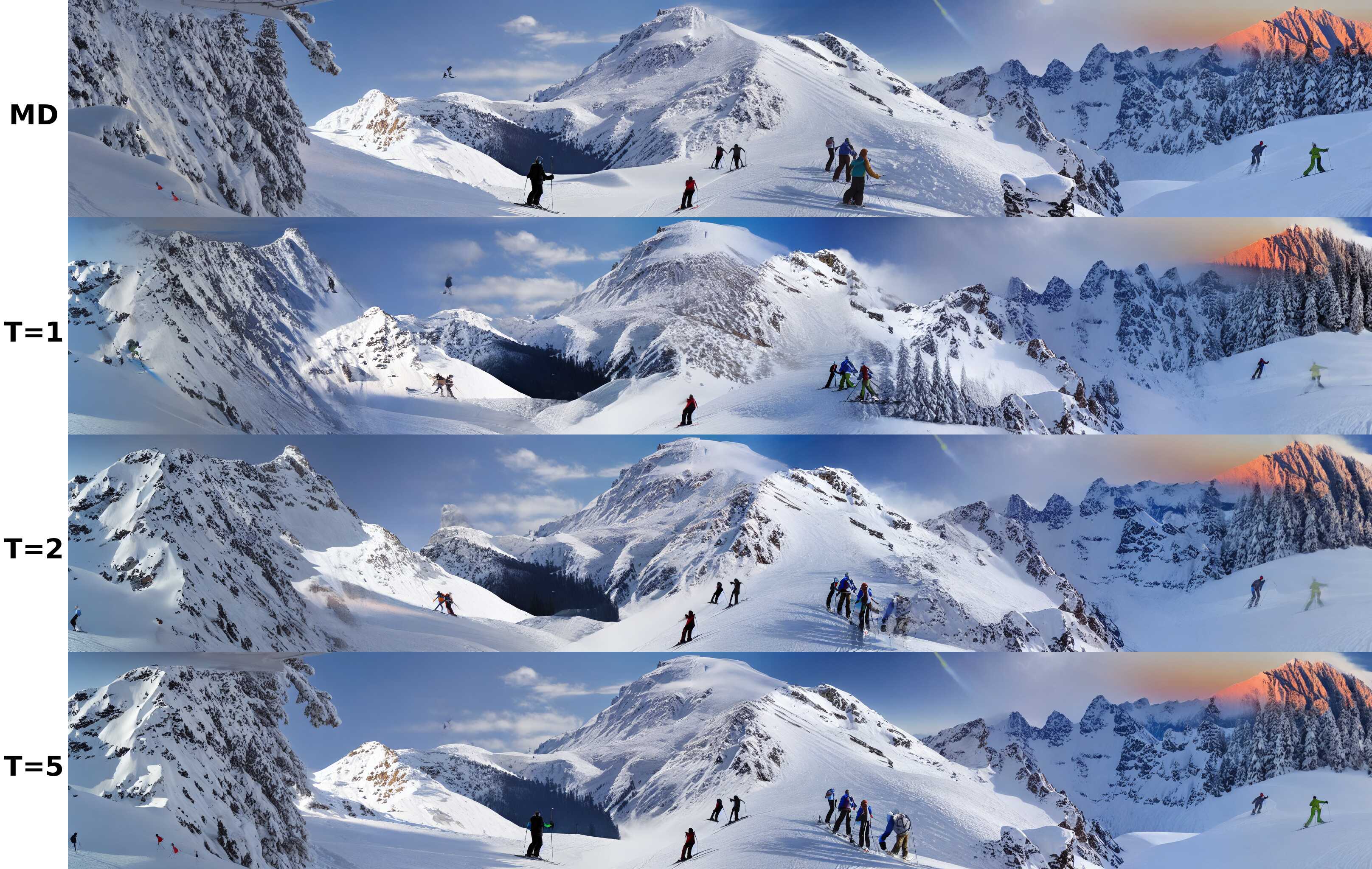}
    \end{subfigure}
    \hfill
    \begin{subfigure}{0.48\textwidth}
        \includegraphics[width=\linewidth]{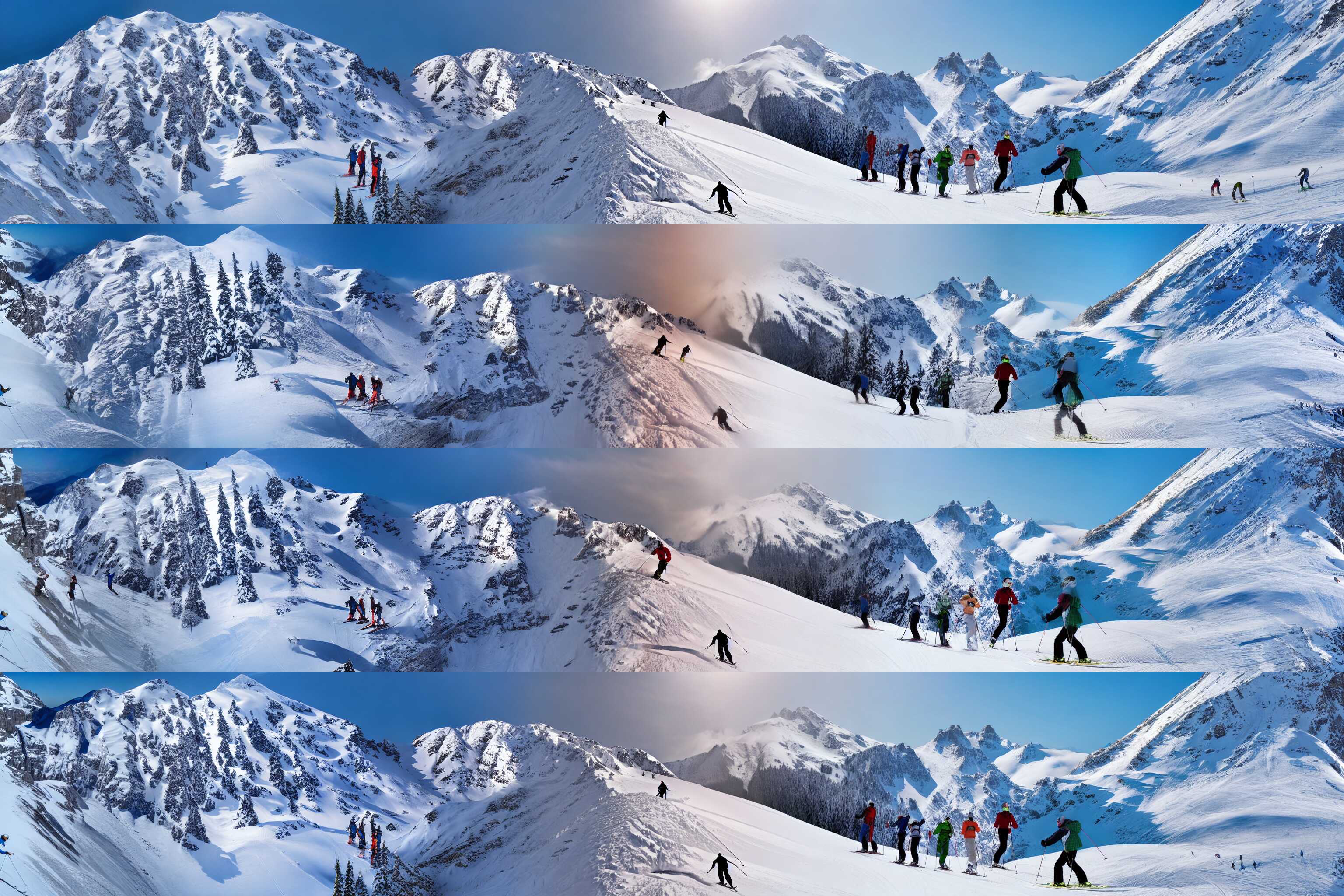}
    \end{subfigure}
    \caption{"A photo of a snowy mountain peak with skiers"}
    \vspace{0.5em}  
    \begin{subfigure}{0.505\textwidth}
        \includegraphics[width=\linewidth]{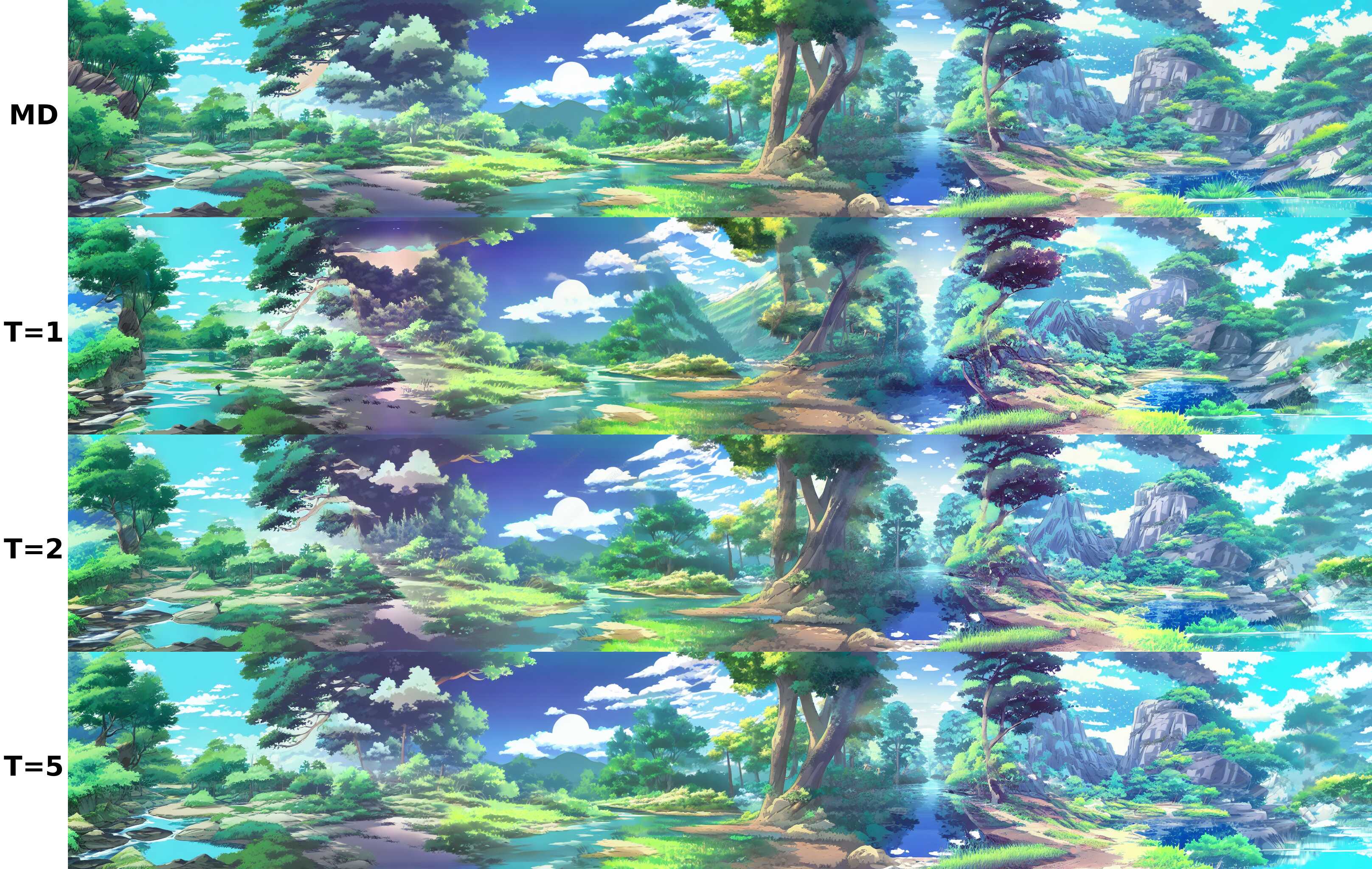}
    \end{subfigure}
    \hfill
    \begin{subfigure}{0.48\textwidth}
        \includegraphics[width=\linewidth]{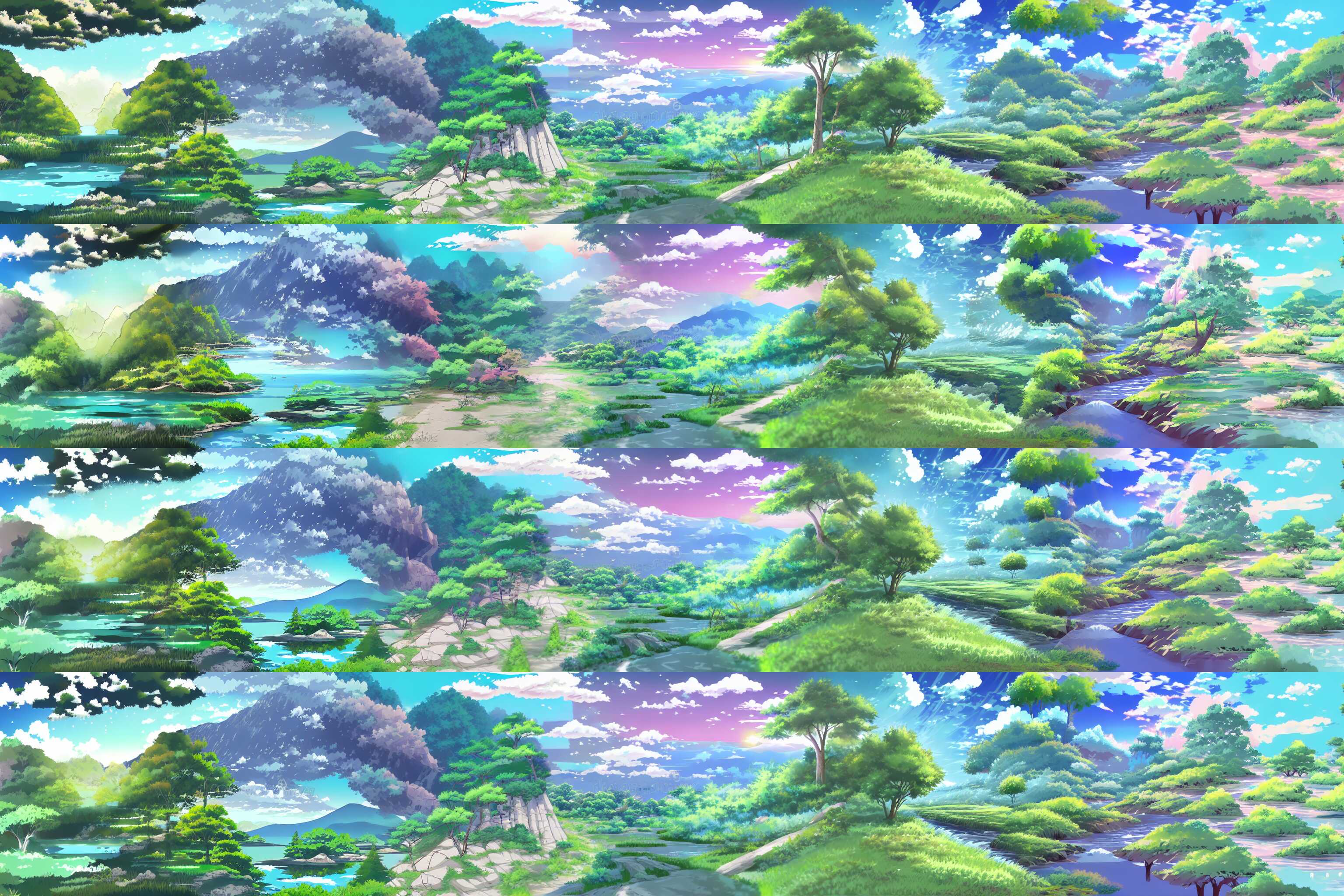}
    \end{subfigure}
    \caption{"Natural landscape in anime style illustration"}
    \label{fig:infinite_diffusion_viz_appdx_2}
\end{figure*}
\clearpage

\end{document}